\documentclass{article} 
\usepackage{iclr2026_conference,times}
\iclrfinalcopy

\usepackage{amsmath,amsfonts,bm}
\usepackage{amsthm}

\def\eqref#1{equation~\ref{#1}}

\def\1{\bm{1}}

\DeclareMathAlphabet{\mathsfit}{\encodingdefault}{\sfdefault}{m}{sl}
\SetMathAlphabet{\mathsfit}{bold}{\encodingdefault}{\sfdefault}{bx}{n}

\usepackage{xcolor}
\usepackage{hyperref}
\usepackage{url}
\usepackage{natbib}
\usepackage{booktabs}
\usepackage{adjustbox}
\usepackage{empheq}
\usepackage{multirow}
\usepackage{tabularx}
\usepackage{subcaption}
\usepackage{wrapfig}
\usepackage{tcolorbox}
\usepackage{thm-restate}
\usepackage{graphicx}
\usepackage{rotating}
\usepackage{cleveref}

\makeatletter
\newenvironment{lemgray}{%
  \par\noindent
  \def\lemgray@pending{}%
  \let\lemgray@orig@footnote\footnote
  \renewcommand{\footnote}[1]{%
    \footnotemark    \g@addto@macro\lemgray@pending{\footnotetext{##1}}%
  }%
  \begin{tcolorbox}[
    colback=gray!10, 
    colframe=gray!10,
    boxrule=0pt,
    left=0pt,right=0pt,
    top=4pt,bottom=
    4pt,
    boxsep=6pt,
    before skip=6pt, after skip=6pt
  ]%
}{%
  \end{tcolorbox}%
  \lemgray@pending
  \let\footnote\lemgray@orig@footnote
  \par
}
\makeatother

\title{HalluGuard: Demystifying Data-Driven and Reasoning-Driven Hallucinations in LLMs}

\author{Xinyue Zeng\thanks{Equal contribution.}\\
CS Department\\
Virginia Tech\\
\And
Junhong Lin\footnotemark[1] \\
EECS Department\\
MIT\\
\And
Yujun Yan \\
CS Department\\
Dartmouth College\\
\And
Feng Guo \\
Statistics Department\\
Virginia Tech \\
\And
Liang Shi \\
Statistics Department\\
Virginia Tech \\
\And
Jun Wu \\
CS Department\\
Michigan State University \\
\And
Dawei Zhou \\
CS Department\\
Virginia Tech
}

\newcommand{\model}{\textsc{HalluGuard}}
\newtheorem{theorem}{Theorem}[section] 
\newtheorem{lemma}[theorem]{Lemma}

\newtheorem{problem}{Problem}

\begin{document}

\maketitle

\begin{abstract}
The reliability of Large Language Models (LLMs) in high-stakes domains such as healthcare, law, and scientific discovery is often compromised by hallucinations. These failures typically stem from two sources: \emph{data-driven hallucinations} and \emph{reasoning-driven hallucinations}. However, existing detection methods usually address only one source and rely on task-specific heuristics, limiting their generalization to complex scenarios.
To overcome these limitations, we introduce the \emph{Hallucination Risk Bound}, a unified theoretical framework that formally decomposes hallucination risk into data-driven and reasoning-driven components, linked respectively to training-time mismatches and inference-time instabilities. This provides a principled foundation for analyzing how hallucinations emerge and evolve.
Building on this foundation, we introduce {\model}, a NTK-based score that leverages the induced geometry and captured representations of the NTK to jointly identify data-driven and reasoning-driven hallucinations.
We evaluate {\model} on 10 diverse benchmarks, 11 competitive baselines, and 9 popular LLM backbones, consistently achieving state-of-the-art performance in detecting diverse forms of LLM hallucinations. We open-source our proposed \model{} model at \href{https://github.com/Susan571/HalluGuard-ICLR2026}{HalluGuard}.
\end{abstract}

\section{Introduction}
Large language models (LLMs) are increasingly deployed in high-stakes domains such as healthcare, law, and scientific discovery \citep{bommasani2021opportunities, thirunavukarasu2023large, ke2025early}. 
However, adoption in these settings remains cautious, as such domains are highly regulated and demand strict compliance, interpretability, and safety guarantees \citep{dennstadt2025implementing,kattnig2024assessing}.
A major barrier is the risk of \emph{hallucinations}, generated content appears unfaithful or nonsensical. Such errors can have severe consequences \citep{dennstadt2025implementing}, as the example in \Cref{fig:diagnostic-reasoning}, a generated incorrect medical diagnosis may delay treatment or lead to harmful interventions. Therefore, detecting hallucinations is not merely a technical challenge but a prerequisite for trustworthy deployment, as undetected errors undermine reliability, accountability, and user safety.

Generally, hallucinations in LLMs arise from two primary sources \citep{ji2023survey, huang2023survey}: \emph{data-driven hallucinations}, which stem from flawed, biased, or incomplete knowledge encoded during pre-training or fine-tuning; and \emph{reasoning-driven hallucinations}, which originate from inference-time failures such as logical inconsistencies or breakdowns in multi-step reasoning \citep{zhang2023language,zhong2024investigating}. 
Detection methods broadly split along these two dimensions. Approaches for data-driven hallucinations often compare outputs against retrieved documents or references \citep{Kurt202retrieval, minfactscore, ji2023survey}, or exploit sampling consistency as in SelfCheckGPT~\citep{manakulselfcheck}.
In contrast, methods for reasoning-driven hallucinations rely on signals of inference-time instability, including probabilistic measures such as perplexity~\citep{ren2022outofdistribution}, length-normalized entropy~\citep{malinin2020uncertainty}, semantic entropy~\citep{kuhn2023semantic}, energy-based scoring~\citep{liu2020energybased}, and RACE~\citep{wang2025joint}. Others probe internal representations, for example, Inside~\citep{cheninside}, which applies eigenvalue-based covariance metrics and feature clipping, ICR Probe~\citep{zhang2025icr}, which tracks residual-stream updates, and Shadows in the Attention~\citep{wei2025shadows}, which analyzes representation drift under contextual perturbations.
While these methods shed light on the mechanisms underlying hallucinations, most remain tailored to a single hallucination type and fail to capture their evolution. Yet growing evidence indicates that data-driven and reasoning-driven hallucinations often evolve during multi-step generation \citep{liu2025morethinking, sun2025mechanistic}. 
As shown in \Cref{fig:diagnostic-reasoning}, it emerges from an initial disease misclassification and evolves into a distorted diagnosis, delaying treatments and risking fatality. 
This gap brings two central questions: \textit{\textbf{(1) How can we develop a unified theoretical understanding of how hallucinations evolve? (2) How can we detect them effectively and efficiently without relying on external references or task-specific heuristics?}}
To address these challenges, we propose a unified theoretical framework--\textit{Hallucination Risk Bound}, which decomposes the overall hallucination risk into two components: a \emph{data-driven term}, capturing semantic deviations rooted in inaccurate, imbalanced, or noisy supervision acquired during model training; and a \emph{reasoning-driven term}, reflecting instability introduced by inference-time dynamics, such as logical missteps or temporal inconsistency. This decomposition not only elucidates the mechanism behind hallucinations but also reveals how they emerge and evolve.
Specifically, our analysis shows that hallucinations originate from semantic approximation gaps, captured by representational limits of the model, and are subsequently amplified by unstable rollout dynamics, evolving across decoding steps.
As such, our framework offers a unified theoretical lens for characterizing the emergence and evolution of these hallucinations.
\begin{figure}[t!]
    \vspace{-1em}
    \centering    \includegraphics[width=\linewidth]{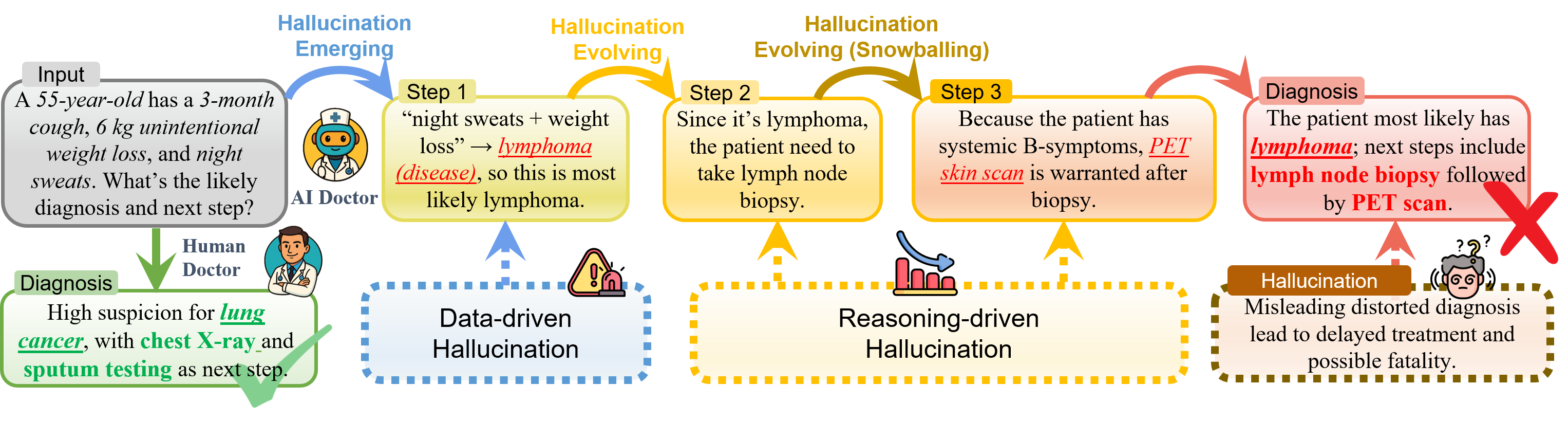}
    \caption{An illustration of hallucination emerging and evolving in the context of disease diagnosis. 
    }
    \label{fig:diagnostic-reasoning}
\end{figure}

Building on the theoretical foundation, we propose \model, a
Neural Tangent Kernel(NTK)-based score that leverages the induced geometry and captured representations of the NTK to jointly identify data-driven and reasoning-driven hallucinations. We evaluate \model{} comprehensively across 10 diverse benchmarks, 11 competitive baselines, and 9 popular LLM backbones. \model{} consistently achieves state-of-the-art hallucination detection performance, demonstrating its efficacy. We open-source our proposed \model{} model at \href{https://github.com/Susan571/HalluGuard-ICLR2026}{HalluGuard}.
\section{Preliminaries}
\label{sec:hallucination-detection}
\paragraph{Hallucination Detection. }
There are two primary sources of hallucinations in LLMs \citep{ji2023survey, huang2023survey}: \emph{data-driven hallucination}, which stems from incomplete or biased knowledge encoded during pre-training or fine-tuning, and \emph{reasoning-driven hallucination}, which arises from unstable or inconsistent inference dynamics at decoding time. This distinction has implicitly guided a broad range of detection strategies, which we examine through these two lenses.

For data-driven causes, a recurring signal is elevated predictive uncertainty. A common formulation adopts the sequence-level negative log-likelihood:
{\small
\begin{equation}
\mathcal{U}(\mathbf{y} \mid \mathbf{x}, \theta) 
= - \frac{1}{T} \sum_{t=1}^{T} \log p_\theta(y_t \mid y_{<t}, \mathbf{x}),
\end{equation}}

which quantifies the average uncertainty of generating a sequence $\mathbf{y} = [y_1, \ldots, y_T]$ from input $\mathbf{x}$ and \(\theta\) denotes model parameters. This directly recovers \emph{Perplexity}~\citep{ren2022outofdistribution}, where low scores imply confident predictions, while high scores indicate implausible generations due to weak priors.  
To capture more nuanced uncertainty, later methods extend this formulation to multi-sample settings. The \emph{Length-Normalized Entropy}~\citep{malinin2020uncertainty} penalizes dispersion across stochastic generations $\mathcal{Y} = \{\mathbf{y}^1, \ldots, \mathbf{y}^K\}$ where $K$ denotes the number of independent stochastic rollouts sampled from the model for a given input, offering a finer-grained view of model indecision. This perspective is further enriched by \emph{Semantic Entropy}~\citep{kuhn2023semantic}, which projects sampled responses into semantic space, and by energy-based scoring \citep{liu2020energybased}, which replaces log-probability with a learned confidence function. Collectively, these methods reflect a progression from token-level likelihoods to semantically grounded multi-sample uncertainty estimators.

In contrast, reasoning-driven hallucinations arise from brittle inference trajectories, where identical contexts may yield inconsistent or incoherent outputs. A commonly used measure of such instability is the cross-sample consistency score:
\vspace{-0.5em}
{\small
\begin{equation}
\mathcal{C}(\mathcal{Y} \mid \mathbf{x}, \theta) 
= \frac{1}{C} \sum_{i=1}^{K}\sum_{j=i+1}^{K} 
\text{sim}(\mathbf{y}^i, \mathbf{y}^j),
\vspace{-0.5em}
\end{equation}}

where $C = K \cdot \frac{(K-1)}{2}$, and $\text{sim}(\cdot,\cdot)$ is a similarity function such as ROUGE-L~\citep{lin2004rouge}, cosine similarity, or BLEU~\citep{chen2024hall}. Low scores reflect diverging generations and unstable reasoning.  
Several reasoning-driven detection methods can be interpreted through this lens. Early approaches used surface-level lexical overlap metrics \citep{lin-etal-2022-towards}, while \emph{SelfCheckGPT}~\citep{manakulselfcheck} advanced this by evaluating factual entailment across responses, and \emph{FActScore}~\citep{minfactscore} extended this further by comparing outputs to retrieved reference documents.  More recent efforts probe internal signals directly: \emph{Inside}~\citep{cheninside}
analyzes the covariance spectrum of embedding representations, and \emph{RACE}~\citep{wang2025joint} diagnoses instability in multi-step reasoning.
\paragraph{NTK in LLMs.}
NTK provides a principled framework for analyzing the training dynamics in the overparameterized regime characteristic of modern LLMs \citep{jacot2018ntk}.
Formally, for a network output $f(x, \theta)$ with input $x$ and parameters $\theta$, the NTK is defined as:
{\small
\begin{equation}
\Theta(x, x', \theta) = \nabla_{\theta} f(x, \theta) \cdot \nabla_{\theta} f(x', \theta).
\end{equation}}

This kernel $\Theta(x, x', \theta)$ quantifies the similarity of training dynamics between inputs $x$ and $x'$. In the infinite-width limit, it converges to a deterministic value at initialization and remains nearly constant throughout training \citep{lee2020wide}. This stability reduces the highly nonlinear optimization of deep networks to a tractable kernel regression problem. By examining the eigenspectrum of the NTK, one can probe how internal representations are shaped during training: which features are prioritized (e.g., syntax versus semantics), how quickly different tasks converge, and why overparameterized networks generalize effectively to unseen data \citep{ju2022gen}. In this way, the NTK transforms the apparent complexity of LLM optimization into a clear lens on how these models capture, process, and generalize information \citep{zeng2025lensllm}.

\section{Methodology}
\label{sec:methodology}



\subsection{Problem Setting}
Our analysis reveals that hallucination is not a unified failure mode but rather shifts with the task structure.  
On the instruction-following \texttt{Natural} benchmark~\citep{wang2022supernatural}, 88.9\% of the overall 3499 errors are from logical missteps (\emph{reasoning-driven}) while 11.1\% are factual inaccuracies (\emph{data-driven}).  
By contrast, on the math-focused \texttt{MATH-500}~\citep{hendrycks2021measuring}, the 1985 wrong generations are dominated by 1946 reasoning errors (98.1\%), with only 19 factual flaws (1.9\%).  
This contrast highlights that, in practice, hallucinations are rarely pure but often mixtures of data-driven bias and reasoning-driven instability-motivating our formal decomposition of hallucination sources.

\textbf{Problem Definition. }

Let $\mathcal{Y}$ denote the discrete space of all possible finite-length textual token sequences. 
We define a continuous semantic embedding space $U_h \subseteq\ \mathbb{R}^{d_h}$ equipped with a norm $\|\cdot\|$. Each vector $u \in U_h$ represents the semantic representation of a reasoning chain composed of step-wise logical statements. We define a task-specific encoder 
$ \Phi: \mathcal{Y} \rightarrow U_h $
that maps a discrete textual sequence into this continuous hypothesis space. 
In this framework, for an input $\mathbf{x}$ with a ground-truth output sequence $y^* \in \mathcal{Y}$, we define the target semantic representation as $ u^* := \Phi(y^*) \in U_h. $
An LLM with parameters $\theta$ emits a random sequence $Y = (Y_1, \dots, Y_T) \in \mathcal{Y}$ 
via the autoregressive decoding distribution  $p_\theta(y_t \mid y_{<t}, \mathbf{x})$, 
yielding a predicted semantic representation  $ u_h := \Phi(Y) \in U_h. $
Thus, the model's expected sematic output is defined as $\mathbb{E}[u_h] := \mathbb{E}_{Y \sim p_\theta(\cdot \mid \mathbf{x})}[\Phi(Y)].$

To analyze inference dynamics, we consider perturbations in a local neighborhood of the decoding process. Let $\mathbb{R}^r$ denote the $r$-dimensional continuous space of the model's internal states (e.g., prefix embeddings or hidden activations). We parameterize a small perturbation by $\delta \in \mathbb{R}^r,$ restricted to a local $\ell_2$-ball $\mathcal{B}_\rho := \{\delta \in \mathbb{R}^r : \|\delta\|_2 \le \rho\}.$
Let $P_\theta(\cdot \mid \mathbf{x}, \delta)$ denote the perturbed decoding distribution induced by $\delta$. We define the mean semantic response map $G_Y : \mathbb{R}^r \rightarrow U_h, G_Y(\delta) := \mathbb{E}_{Y \sim P_\theta(\cdot \mid \mathbf{x}, \delta)}[\Phi(Y)]$ with its corresponding inference Jacobian $J := DG_Y(0) \in \mathbb{R}^{d_h \times r}$. Thus, we formally define the problem as follows:

\begin{problem}[Hallucination Dynamics Characterization]
\mbox{}\\
\noindent \textbf{Given:}
(1) The target semantic representation $u^* := \Phi(y^*) \in U_h$ for a ground-truth output $y^* \in \mathcal{Y}$; (2) the random sequence $Y \in \mathcal{Y}$ emitted via the autoregressive decoding distribution $p_\theta(y_t \mid y_{<t}, \mathbf{x})$, yielding a predicted representation $u_h := \Phi(Y)$ with expected value $\mathbb{E}[u_h]$ ; and (3) the inference constraints defined by a local perturbation $\delta \in \mathbb{R}^r$ restricted to the $\ell_2$-ball $\mathcal{B}_\rho$

\textbf{Find:}
A formal geometric mechanism to characterize how hallucinations emerge and evolve by analyzing the Mean Semantic Response Map $G_Y(\delta)$ and the Inference Jacobian $J$, which captures the sensitivity of the model's reasoning trajectory to internal instabilities.
\end{problem}

\subsection{Hallucination Risk Bound}
\label{sec:bound}
To bridge the formal setup with the phenomenon of hallucination, we first disentangle the sources of hallucinations.  Intuitively, hallucinations may arise either from systematic biases in the knowledge encoded by the model (data-driven) or from instabilities during autoregressive decoding (reasoning-driven). The following proposition formalizes this idea by decomposing the total hallucination risk into two components.

We first impose the following assumptions:
\begin{itemize}
\item[\textbf{A1.}] $(U_h,\|\cdot\|)$ is a finite-dimensional Hilbert space.
The encoder $\Phi:\mathcal{Y}\to U_h$ is measurable, and the random variable $\Phi(Y)$ has finite second moment under the model's unperturbed decoding distribution:
$
\mathbb{E}_{Y \sim p_\theta(\cdot \mid \mathbf{x})}
\big[
\|\Phi(Y)\|^2
\big] 
< \infty.
$
This ensures that the mean semantic representation 
$\mathbb{E}[\Phi(Y)]$ is well-defined in $U_h$.
\item[\textbf{A2.}] Let $(\mathcal{Y}, d_{\mathcal{Y}})$ be the discrete metric space equipped with edit distance.
The encoder 
$
\Phi : (\mathcal{Y}, d_{\mathcal{Y}}) \rightarrow (U_h, \|\cdot\|)
$
is $L_\Phi$-Lipschitz continuous:
$
\|\Phi(y) - \Phi(y')\| 
\le 
L_\Phi \, d_{\mathcal{Y}}(y,y')
\quad
\forall y,y' \in \mathcal{Y}.
$
\item[\textbf{A3.}] For any perturbation $\delta$ in the closed ball 
$\mathcal{B}_\rho := \{\delta \in \mathbb{R}^r : \|\delta\|_2 \le \rho\}$,
the mean semantic response map
$
G_Y(\delta)
=
\mathbb{E}_{Y \sim P_\theta(\cdot \mid \mathbf{x}, \delta)}[\Phi(Y)]
$
is twice Fréchet differentiable in a neighborhood of $\delta = 0$ and admits the expansion
$
G_Y(\delta)
=
G_Y(0)
+
J\delta
+
R(\delta),
$
where $J = D G_Y(0) \in \mathbb{R}^{d_h \times r}$ and the remainder satisfies
$
\|R(\delta)\|
\le
\frac{1}{2} H_\star \|\delta\|_2^2, 
\forall \delta \in \mathcal{B}_\rho,
$
for some constant $H_\star > 0$.
\end{itemize}

\vspace{-0.5em}
\begin{lemgray}
\textbf{Proposition 3.1 (Hallucination Risk Decomposition).}
\label{prop:prelim-bound}
Under A1-A3, applying the triangle inequality yields a natural split of the risk:
{\small
\begin{equation*}
\|u^* - u_h\| \;\le\; 
\underbrace{\|u^* - \mathbb{E}[u_h]\|}_{\text{data-driven term}}
\;+\;
\underbrace{\|u_h - \mathbb{E}[u_h]\|}_{\text{reasoning-driven term}}
\end{equation*}
}
This decomposition distinguishes errors caused by systematic bias in the learned representation from those introduced during stochastic rollout.
\end{lemgray}
\vspace{-0.5em}
\paragraph{Characterizing Data-Driven Hallucination.}
To quantify the data-driven term, we take inspiration from the NTK, 
which has proven effective in analyzing training dynamics of overparameterized models.
Here, NTK geometry provides a way to measure how well the model's representation space 
aligns with task generation under small perturbations.  

Let $U_h$ denote the hypothesis subspace accessible to the model under perturbations.
By Céa's lemma\citep{cea1964approximation} 
with curvature penalty, the data-driven term
can be bounded as
{\small
\begin{equation}
\|u^*-\mathbb{E}[u_h]\|
\;\le\;
\frac{\Lambda}{\gamma}
\inf_{u\in U_h}\|u^*-u\|,
\end{equation}
}

where $\gamma=\lambda_{\min}(\mathcal{K}_\Phi)$ is the smallest eigenvalue of the NTK Gram matrix on embedded perturbations $\mathcal{K}_\Phi$, and $\Lambda\le\|\mathcal{T}\|$, where $\mathcal{T}: U_h \to U_h$ denotes the operator mapping.
Intuitively, the ratio $\tfrac{\Lambda}{\gamma}$ measures the conditioning of the feature map:  
well-conditioned NTK spectra allow a closer approximation to the true generation.  

Thus, the ratio can be further controlled in terms of pretraining-finetuning mismatch:
{\small
\begin{equation}
\label{eq:LambdaGamma}
\frac{\Lambda}{\gamma}
\;\le\;
1 \;+\; k_{\mathrm{pt}}\log\!\mathcal{O}(P,L)
\;+\; k\cdot \frac{\epsilon_{\mathrm{mismatch}}}{\mathrm{Signal}_k},
\end{equation}
}

where $\log\!\mathcal{O}(P,L)$ is a complexity term from parameter count $P$ and prompt length $L$, $\epsilon_{\mathrm{mismatch}}$ denotes the Wasserstein distance between prompt and query distributions, $\mathrm{Signal}_k$ measures task-aligned energy in the top-$k$ eigenspace. $k_{\mathrm{pt}}$ and $k$ are task and model-dependent constants. Thus, data-driven hallucinations grow when the mismatch is large or when the task signal is weak.  

\paragraph{Characterizing Reasoning-Driven Hallucination.}
The reasoning-driven term 
captures \emph{reasoning-driven} instability that accumulates during autoregressive decoding.  
Here, we model generation as a martingale process, where deviation from the expectation
is controlled by concentration inequalities.  
Specifically, Freedman's inequality~\citep{geman1992neural} gives
{\small
\begin{equation}
\|u_h-\mathbb{E}[u_h]\|
\;\le\;
K\cdot \exp\!\Big(-\tfrac{K\epsilon^2}{C}\Big)\cdot \alpha(e^{\beta T}-1),
\end{equation}
}

where $K$ is the number of rollouts averaged, $\beta$ summarizes per-step growth in local Jacobians, $\alpha$ scales the cumulative effect and $C$ is a task and model-dependent constant.
This bound shows that reasoning-driven hallucinations grow exponentially with sequence length $T$.  

We now synthesize the two components into a unified result that characterizes the overall risk of hallucination.  
By combining the NTK-conditioned approximation bound for data-driven deviation with the Freedman-style concentration bound for reasoning-driven instability, we obtain the following unified bound of data-driven and reasoning-driven hallucinations (detailed proof is provided in \Cref{appendix:proof}):
\begin{lemgray}
\textbf{Theorem 3.2 (Hallucination Risk Bound).}
\label{rem:hallucination-bound}
Let $u^* := \Phi(y^*)$ denote the semantic embedding of the ground-truth output
and $u_h := \Phi(Y)$ that of the model-generated output.  
Under Assumptions~A1-A3, suppose there exists $\beta \ge 0$ such that
$
\Big\|\prod_{t=1}^T J_t\Big\|_2 \;\le\; e^{\beta T}.
$
Then the total hallucination risk satisfies
{\small
\begin{equation*}
\|u^* - u_h\|
\;\le\;
\underbrace{\Big(1 + k_{\mathrm{pt}}\log \mathcal{O}(P,L) + k\cdot \tfrac{\epsilon_{\mathrm{mismatch}}}{\mathrm{Signal}_k}\Big)
\inf_{u\in U_h}\|u^*-u\|}_{\text{data-driven term}}
\;+\;
\underbrace{|\mathcal{L}|\cdot \exp\!\Big(-\tfrac{K\epsilon^2}{C}\Big)\cdot \alpha\big(e^{\beta T}-1\big)}_{\text{reasoning-driven term}}
\end{equation*}
}
Here, $|\mathcal{L}|$ denotes the total sampled trajectories.
\end{lemgray}
\vspace{-0.5em}

\subsection{Hallucination Quantification via {\model}}

While Theorem \ref{rem:hallucination-bound} makes explicit how data-driven and reasoning-driven hallucinations
emerge and evolve, applying it directly at inference is impractical since direct step-wise Jacobians for billion-parameter LLMs are intractable, so we seek a \emph{proxy score} that is computable, stable, and faithful to our decomposition.

Let $\mathcal{K}$ denote the NTK Gram matrix with eigenvalues $\lambda_1 \ge \cdots \ge \lambda_r > 0$ 
and condition number $\kappa(\mathcal{K}) = \lambda_{\max} / \lambda_{\min}$. 
Let $J_t$ be the step-$t$ input-output Jacobian of the decoder, and define $\sigma_{\max} := \sup_t \|J_t\|_2$ as the uniform spectral bound(note that $\sigma_{\max}$ is independent of the spectrum of $\mathcal{K}$).

Under Assumptions A1-A3, a standard NTK approximation argument yields
$
\inf_{u \in U_h} \|u^* - u\| \;\le\; C_d \, \det(\mathcal{K})^{-c_d} \, \|u^*\|,
$
so that $\det(\mathcal{K})$ capture the representations in systematic bias.  

For autoregressive rollout, based on the property of  Jacobian, we have 
$
\Big\|\prod_{t=1}^T J_t\Big\|_2 \;\le\; \prod_{t=1}^T \|J_t\|_2
\;=\; \exp\!\Big(\sum_{t=1}^T \log \|J_t\|_2\Big)
$
, so that we have
$
\Big\|\prod_{t=1}^T J_t\Big\|_2 \;\le\; e^{\beta T}.
$
Since $\beta \le \log \sigma_{\max}$ with
$\sigma_{\max} := \sup_t \|J_t\|_2$
thus we have the upper bound as  $\|\prod_{t=1}^T J_t\|_2 \le \sigma_{\max}^T = e^{(\log\sigma_{\max})T}$.
Thus, $\log\sigma_{\max}$ serves as a stable and tractable proxy for the per-step amplification rate.


Perturbation analysis of $\mathcal{K}$, together with classical eigenvalue
sensitivity results \citep{trefethen2022numerical}, yields
$
\mathrm{Var}[u_h] \;\le\; c_v \,\kappa(\mathcal{K})^2 \, \|\delta\|^2,
$
showing that instability grows quadratically with the condition number
$\kappa(\mathcal{K})$. To temper this effect and ensure additivity, we penalize ill-conditioned representations via $-\log \kappa^2$, where log compression brings a well-behaved dynamic range.

\begin{wraptable}{r}{0.46\textwidth} 
\centering
\vspace{-0.75em}
\setlength{\tabcolsep}{1pt}
\caption{Correlation between NTK proxies and task families.}
\vspace{-0.75em}
\small
\begin{tabular}{lccc}
\toprule
 & SQuAD & Math-500 & TruthfulQA \\
\midrule
$\det(\mathcal{K})$     & 0.84 & 0.42 & 0.61 \\
$\log \sigma_{\max}-\log \kappa^2$ & 0.39 & 0.88 & 0.67 \\
\bottomrule
\end{tabular}
\vspace{-0.75em}
\label{tab:correlation}
\end{wraptable}
In summary, $\det(\mathcal{K})$ quantifies representational adequacy, $\log\sigma_{\max}$ captures rollout amplification, and $-\log\kappa^2$ penalizes spectral instability, together forming a compact and tractable proxy consistent with the Hallucination Risk Bound. The lightweight projection layers are self-supervised spectral calibration modules, optimized offline (via AdamW) to align NTK spectral properties across heterogeneous backbones into a stable, comparable geometric space-without hallucination labels or task-specific supervision, with the backbone fully frozen and zero runtime overhead during inference. Detailed proofs are provided in \Cref{appendix:connection}.

\paragraph{Empirical validation.}
We empirically validate how those proxies correlate with different task families.
In \Cref{tab:correlation}, $\det(\mathcal{K})$ correlates most strongly 
with the data-centric task \texttt{SQuAD} ($0.84$), indicating its role in capturing 
factual fidelity. In contrast, for the reasoning-oriented \texttt{MATH-500}, the 
highest correlation is observed with $\log\sigma_{\max} - \log\kappa^2$ ($0.88$), 
reflecting the importance of amplification and stability in multi-step reasoning. 

Motivated by the above, we formally define \model{} as follows, which provides a principled and unified lens for hallucination detection:
\begin{equation}
\label{eq:model-score}
\boxed{
\text{\textbf{\model}}(u_h) \;=\;
\det(\mathcal{K})
\;+\;
\log \sigma_{\max}
\;-\;
\log \kappa^2.}
\end{equation}

\section{Experiments}
\label{sec:experiments}
We comprehensively evaluate \model{} across 10
diverse benchmarks, 11 competitive baselines, and 9 popular LLM backbones. We aim to evaluate its efficacy from the following five questions:
\textbf{\textit{Q1:}} \textit{How does \model{} perform across different task families?}  
\textbf{\textit{Q2:}} \textit{How does \model{} perform across LLMs of different scales?}  
\textbf{\textit{Q3:}} \textit{How does each term capture trends across task families?}
\textbf{\textit{Q4:}} \textit{Can \model{} guide test-time inference to improve downstream reasoning?} \textbf{\textit{Q5:}} \textit{How well does \model{} generalize to detecting fine-grained hallucinations beyond benchmarks?} 

\Cref{sec:setup} details the setup; \Cref{sec:mainresults} evaluates \model{} as a detection method(Q1--Q3), \Cref{sec:testtime} applies \model{} in score-guided inference(Q4) and \Cref{sec:case} analyzes \model{} on fine-grained hallucination via a case study on semantic data(Q5).
\vspace{-0.5em}
\subsection{Evaluation Setup}
\label{sec:setup}
\textbf{Benchmarks.}
We evaluate across 10 widely used benchmarks spanning three distinct categories. For data-grounded QA, we include \texttt{RAGTruth}~\citep{niu2024ragtruth}, \texttt{NQ-Open}~\citep{kwiat2019natural}, \texttt{HotpotQA}~\citep{yang2018hotpot} and \texttt{SQuAD}~\citep{rajpurkar2016squad}, which emphasize factual correctness through external evidence. For reasoning-oriented tasks, we use \texttt{GSM8K}~\citep{cobbe2021training}, \texttt{MATH-500}~\citep{hendrycks2021measuring}, and \texttt{BBH}~\citep{suzgun2022challenging}, which require multi-step derivations prone to compounding errors. Finally, for instruction-following settings, we consider \texttt{TruthfulQA}~\citep{lin2022truthfulqa}, \texttt{HaluEval}~\citep{li2023halu} and \texttt{Natural}~\citep{wang2022supernatural}, which probe hallucinations under open-ended or adversarial prompts.

\textbf{Baselines.} We compare \model{} with 11 competitive detectors spanning diverse strategies. Uncertainty-based methods include Perplexity~\citep{ren2022outofdistribution}, Length-Normalized Predictive Entropy(LN-Entropy)~\citep{malinin2020uncertainty}, Semantic Entropy~\citep{kuhn2023semantic}, Energy Score~\citep{liu2020energybased} and P(true)~\citep{kadavath2022language}. Consistency-based approaches cover SelfCheckGPT~\citep{manakulselfcheck}, Lexical Similarity~\citep{lin-etal-2022-towards}, FActScore~\citep{minfactscore} and RACE~\citep{wang2025joint}. Internal-state methods are represented by Inside~\citep{cheninside} and MIND~\citep{suunsuper}. 

\textbf{LLM Backbone Models.} We evaluate 9 publicly available LLMs spanning different scales and architectures. These include five models from the Llama family (Llama2-7B, Llama2-13B, Llama2-70B, Llama3-8B, and Llama3.2-3B)~\citep{touvron2023llama2,grattafiori2024llama3}, along with OPT-6.7B~\citep{zhang2022opt}, Mistral-7B-Instruct~\citep{jiang2023mistral7b}, QwQ-32B~\citep{yang2024qwen}, and GPT-2 (117M)~\citep{radford2019language}. All models are used in their off-the-shelf form with pre-trained weights and tokenizers provided by Hugging Face, without further fine-tuning.

\textbf{Evaluation Metrics.} 
We evaluate hallucination detection ability under two regimes following~\citet{janiak2025illusion}: ROUGE-based reference evaluation (\(*_r\)) and \textsc{LLM-as-a-Judge} (\(*_{\mathrm{llm}}\)). For performance measures, we report the area under the receiver operating characteristic curve (AUROC) and the area under the precision-recall curve (AUPRC). AUROC is widely used to assess the quality of binary classifiers and uncertainty estimators, while AUPRC highlights performance under class imbalance. In both cases, higher values indicate better detection.

\subsection{Main Results}
\label{sec:mainresults}
\textbf{Q1: How does \model{} perform across different task families?}
To evaluate how \model{} performs across different task types, we conduct experiments on all benchmarks. For clarity, \Cref{tab:Q1} presents representative results from three task families: data-centric (\texttt{RAGTruth}), reasoning-oriented (\texttt{Math-500}), and instruction-following (\texttt{TruthfulQA}). 
As shown, \model{} consistently outperforms all baselines across backbones. On \texttt{Math-500}, it reaches 81.76\% AUROC and 79.76\% AUPRC, improving over the second-best method by up to 8.3\%. On \texttt{RAGTruth}, it attains 84.59\% AUROC and 81.15\% AUPRC, with gains of up to 7.7\%. On \texttt{TruthfulQA}, it achieves 77.05\% AUROC and 73.79\% AUPRC, exceeding the next strongest baseline by as much as 6.2\%. Overall, \model{} establishes new state-of-the-art results across diverse task families, with particularly pronounced improvements on reasoning-oriented benchmarks.
\begin{table*}[ht!]
\centering
\scriptsize
\caption{Performance comparison on representative benchmarks: data-centric (\texttt{RAGTruth}), reasoning-oriented (\texttt{Math-500}), and instruction-following (\texttt{TruthfulQA}). We highlight the \textbf{first} and \underline{second} best results.}
\resizebox{\textwidth}{!}{
\begin{tabularx}{0.96\textwidth}{@{}l@{\hspace{0.5em}}l@{\hspace{0em}}c@{\hspace{0.2em}}c@{\hspace{0.2em}}c@{\hspace{0em}}c@{\hspace{0.6em}}c@{\hspace{0.2em}}c@{\hspace{0.2em}}c@{\hspace{0em}}c@{\hspace{0.6em}}c@{\hspace{0.2em}}c@{\hspace{0.2em}}c@{\hspace{0em}}c@{\hspace{0.6em}}c@{\hspace{0.2em}}c@{\hspace{0em}}c@{\hspace{0.2em}}c@{}}%
\toprule
 & & \multicolumn{4}{c}{\textbf{GPT2}} & \multicolumn{4}{c}{\textbf{OPT-6.7B}} & \multicolumn{4}{c}{\textbf{Mistral-7B}} & \multicolumn{4}{c}{\textbf{QwQ-32B}} \\
\cmidrule(lr){3-6} \cmidrule(lr){7-10} \cmidrule(lr){11-14} \cmidrule(lr){15-18}
& & \rotatebox{55}{\tiny{AUROC$_r$}} & \rotatebox{55}{\tiny{AUPRC$_r$}} & \rotatebox{60}{\tiny{AUROC$_\text{llm}$}} & \rotatebox{55}{\tiny{AUPRC$_\text{llm}$}}
  & \rotatebox{55}{\tiny{AUROC$_r$}} & \rotatebox{55}{\tiny{AUPRC$_r$}} & \rotatebox{60}{\tiny{AUROC$_\text{llm}$}} & \rotatebox{55}{\tiny{AUPRC$_\text{llm}$}}
  & \rotatebox{55}{\tiny{AUROC$_r$}} & \rotatebox{55}{\tiny{AUPRC$_r$}} & \rotatebox{60}{\tiny{AUROC$_\text{llm}$}} & \rotatebox{55}{\tiny{AUPRC$_\text{llm}$}}
  & \rotatebox{55}{\tiny{AUROC$_r$}} & \rotatebox{55}{\tiny{AUPRC$_r$}} & \rotatebox{60}{\tiny{AUROC$_\text{llm}$}} & \rotatebox{55}{\tiny{AUPRC$_\text{llm}$}} \\
\midrule
\multirow{12}{*}{\rotatebox{90}{\textbf{RAGTruth}}}
 & \model        & \textbf{75.51} & \textbf{73.40} & \textbf{62.40} & \textbf{56.60} & \textbf{80.13} & \textbf{76.77} & \textbf{71.01} & \textbf{63.58} & \textbf{82.31} & \textbf{80.79} & \textbf{64.89} & \textbf{67.25} & \textbf{84.59}	&\textbf{81.15}	&\textbf{71.82}	&\textbf{66.68}\\
 & Inside       & \underline{73.42} & \underline{73.08} & \underline{61.99} & \underline{56.39} & \underline{79.49} & \underline{71.82} & \underline{66.1} & \underline{62.46} & \underline{75.32}	&\underline{73.19}	&\underline{64.58}&	\underline{61.05} & \underline{77.72}	&\underline{73.47}	&\underline{66.05}&	\underline{64.73}\\
 & MIND          & 58.54 & 54.79 & 43.47 & 41.85 & 63.82 & 62.58 & 51.03 & 44.78 & 73.13 & 71.53 & 58.25 & 58.6 & 64.23&63.06	&47.37&	51.47\\
 & Perplexity    & 58.07 & 56.68 & 43.84 & 41.53 & 64.47 & 61.57 & 47.12 & 52.98 &65.42	&63.63	&53.28	&51.36 & 73.91	&72.92	&60.81	&59.77\\
 & LN-Entropy    & 64.42 & 60.79 & 49.41 & 45.04 & 60.81 & 57.91 & 48.76 & 42.27 & 64.22 & 60.92 & 52.24 & 48.41 & 63.81	&62.26	&47.52&	52.17\\
 & Energy        & 65.53 & 62.42 & 51.8 & 47.22 & 66.54 & 63.28 & 54.21 & 49.19 & 64.36 & 62.26 & 48.64 & 53.93 & 73.26	&71.21&	65.43&	62.32\\
 & Semantic Ent. & 60.72 & 59.41 & 50.55 & 45.86 & 70.2 & 68.34 & 54.54 & 56.74 & 66.01 & 64.49 & 53.01 & 55.5 & 66.48&	64.41	&51.54&	50.11\\
 & Lexical Sim.  & 64.72 & 63.1 & 55.04 & 48.04 & 67.28 & 64.62 & 52.55 & 54.86 & 64.96 & 61.17 & 52.34 & 45.11 & 70.87	&67.41&	61.25&	51.01\\
 & \tiny SelfCheckGPT  & 65.4 & 62.79 & 52.85 & 52.43 & 66.64 & 64.89 & 52.69 & 51.17 & 71.19 & 68.45 & 63.13 & 60.23 &65.79&62.45	&54.76&	51.29\\
 & RACE          & 64.83 & 62.84 & 51.8 & 48.44 & 64.26 & 61.03 & 52.74 & 46.22 & 66.34 & 64.54 & 51.88 & 53.86 & 71.13&69.96	&57.58&55.54 \\
 & P(true)       & 66.19 & 64.04 & 48.2 & 56.27 & 68.44 & 65.48 & 57.53 & 53.08 & 72.54 & 71.8 & 57.25 & 59.42 & 65.32&63.01	&53.01&	52.32\\
 & FActScore     & 65.72 & 64.39 & 51.94 & 47.51 & 61.53 & 58.2 & 51.86 & 45.57 & 63.98 & 60.71 & 53.54 & 49.34 & 66.72&	64.03&	58.21&	49.17\\
\midrule
\multirow{12}{*}{\rotatebox{90}{\textbf{BBH}}}
 & \model        & \textbf{71.06} & \textbf{67.94} & \textbf{62.05} & \textbf{59.05} & \textbf{73.1} & \textbf{70.88} & \textbf{63.67} & \textbf{61.88} & \textbf{79.85} & \textbf{76.5} & \textbf{67.13}& \textbf{60.57} &  \textbf{81.76}	&\textbf{79.76}	&\textbf{68.77}	&\textbf{65.46}\\
 & Inside        & \underline{66.18} & \underline{66.81} & \underline{56.15} & \underline{58.62} & \underline{70.64} & \underline{65.22} & \underline{63.28} & \underline{59.28} & 67.2 & 65.49 & 51.3 & 53.46 & \underline{80.8}	&\underline{71.49}	&\underline{64.05}&	\underline{63.42}\\
 & MIND          & 55.41 & 51.77 & 39.01 & 41.59 & 55.48 & 53.46 & 38.59 & 40.88 & 65.71 & 63.7 & 49.61 & 52.54 & 61.75	&60.18	&53.46&	50.04\\
 & Perplexity    & 53.28 & 50.22 & 43.86 & 38.98 & 64.89 & 62.12 & 48.65 & 51.99 & 61.97 & 60.05 & 51.15 & 42.87 & 60.28	&57.75	&51.62	&43.38\\
 & LN-Entropy    & 60.84 & 58.76 & 42.76 & 47.48 & 58.71 & 55.01 & 43.55 & 42.02 & \underline{68.96} & \underline{69.44} & \underline{58.79} & \underline{57.49} & 63.96	&62.18	&46.01&	49.5\\
 & Energy        & 55.09 & 51.99 & 46.2 & 39.5 & 53.96 & 50.98 & 42.56 & 34.12 & 66.27 & 62.72 & 49.48 & 50.06 &  69.61	&68.66&	54.35&	57.36\\
 & Semantic Ent. & 58.16 & 54.81 & 49.61 & 40.39 & 62.63 & 59.52 & 50.14 & 45.02 & 64.99 & 61.33 & 50.11 & 45.53 & 62.76&	60.95	&45.77&	45.75\\
 & Lexical Sim.  & 51.37 & 47.18 & 38.37 & 39.06 & 61.27 & 58.06 & 44.13 & 42.96 & 58.25 & 55.92 & 46.31 & 46.01 & 69.46	&67.59&	55.93&	52.6\\
 & \tiny SelfCheckGPT  & 54.51 & 51.86 & 44.62 & 44.01 & 57.36 & 53.21 & 42.55 & 38.27 & 63.68 & 62.5 & 51.7 & 53.03 & 64.56	&62.49	&55.85&	45.8\\
 & RACE          & 55.99 & 54.66 & 41.39 & 38.32 & 64.23 & 62.03 & 56.03 & 53.44 & 66.88 & 64.33 & 49.57 & 48.5 & 59.5	&55.83	&46.13	&41.07\\
 & P(true)       & 54.57 & 52.88 & 45.45 & 44.74 & 57.02 & 55.49 & 48.81 & 37.84 & 57.11 & 55.21 & 43.93 & 47.05 & 61.49	&59.03	&44.37&	44.69\\
 & FActScore     & 56.76 & 53.85 & 40.25 & 40.01 & 54.51 & 53.2 & 38.45 & 36.49 & 62.11 & 58.64 & 53.52 & 47.27 & 58.82&	57.47&	49.48&	42.74\\
\midrule
\multirow{12}{*}{\rotatebox{90}{\textbf{TruthfulQA}}}
 & \model        & \textbf{72.1} & \textbf{68.76} & \textbf{60.09} & \textbf{52.01} & \textbf{69.59} & \textbf{68.36} & \textbf{58.52} & \textbf{52.65} & \textbf{77.05} & \textbf{73.79} & \textbf{63.62} & \textbf{62.26} & \textbf{74.26}	&\textbf{72.76}&	\textbf{57.39}	&\textbf{64.07}\\
 & Inside        & \underline{70.42} & \underline{68.76} & \underline{60.09} & \underline{52.01} & 62.1 & 59.78 & 51.07 & 51.38 & 62.53 & 60.99 & 52.3 & 49.35 & \underline{70.89}	&\underline{64.44}	&\underline{56.61}	&\underline{56.01}\\
 & MIND          & 59.45	&56.79	&45.22	&43.71 & 60.56 & 58.55 & 47.49 & 49.63 & 59.2 & 57.98 & 47.23 & 41.79 & 62.81&	61.5&	52.56&	46.37\\
 & Perplexity    & 50.57 & 47.87 & 40.64 & 35.63 & 55.07 & 52.26 & 44.43 & 42.79 & 60.8 & 59.69 & 47.33 & 41.62 & 55.29	&52.46	&43.95	&43.92\\
 & LN-Entropy    & 58.04	&56.99	&41.94	&47.21 & 56.12 & 54.01 & 47.06 & 38.4 & 59.67 & 56.25 & 41.99 & 41.25 & 60.76	&58.21&	46.24	&42.64\\
 & Energy        & 55.02&	53.31	&38.78&	45.16 & 54.42 & 51.85 & 36.21 & 42.57 & 58.93 & 55.25 & 50.76 & 41.72 & 64.15	&61.32	&51.78	&50.02 \\
 & Semantic Ent. & 61.01	&57.08	&43.35&	45.2 & 51.48 & 47.81 & 34.15 & 38.16 & 54.44 & 53.33 & 36.62 & 40.35 & 66.75	&63.85	&51.11&	46.71 \\
 & Lexical Sim.  & 52.54	&50.56&	39.94	&33.42 & 59.74 & 55.72 & 49.89 & 46.81 & 66.16 & 64.05 & 54.08 & 51.65 & 55.24	&51.36	&46.39	&39.57\\
 & \tiny SelfCheckGPT  & 56.04 & 54.48 & 43.78 & 44.38 & 58.93 & 56.47 & 47.65 & 39.02 & 61.14 & 58.91 & 42.97 & 47.01 & 55.86	&54.95	&41.08	&37.35\\
 & RACE          & 53.02 & 50.33 & 41.7 & 33.81 & \underline{62.95} & \underline{67.89} & \underline{54.61} & \underline{51.93} & \underline{71.06} & \underline{68.49} & \underline{60.4} & \underline{57.44} & 55.75	&52.62	&46.5	&43.19 \\
 & P(true)       & 55.52 & 53.41 & 38.33 & 38.38 & 54.88 & 53.1 & 38.22 & 40.96 & 55.8 & 52.01 & 40.88 & 38.72 & 57.18	&55.16	&46.19	&38.21 \\
 & FActScore     & 53.82 & 51.42 & 41.33 & 35.2 & 54.57	&51.26	&42.51&	35.52 & 53.97 & 50.2 & 42.97 & 36.16 & 62.31	&60.23	&45.06	&49.9 \\
\bottomrule
\vspace{-3em}
\end{tabularx}
}
\label{tab:Q1}
\end{table*}

\textbf{Q2: How does \model{} perform across LLMs of different scales?}
We further investigate whether the effectiveness of \model{} depends on model scale, 
as smaller backbones are typically more prone to hallucination. 
\Cref{tab:Q2} reports representative results on small(Llama2-7B, Llama3-8B), mid-sized(Llama2-13B), and large-scale(Llama2-70B) models using \texttt{SQuAD}, \texttt{GSM8K}, and \texttt{HaluEval}.
Across all settings, \model{} consistently surpasses baselines, 
with the largest margins on smaller models-for instance, 
\begin{wrapfigure}{h!}{0.45\columnwidth}
    \centering
    \vspace{-1em}
    \includegraphics[width=0.45\columnwidth]{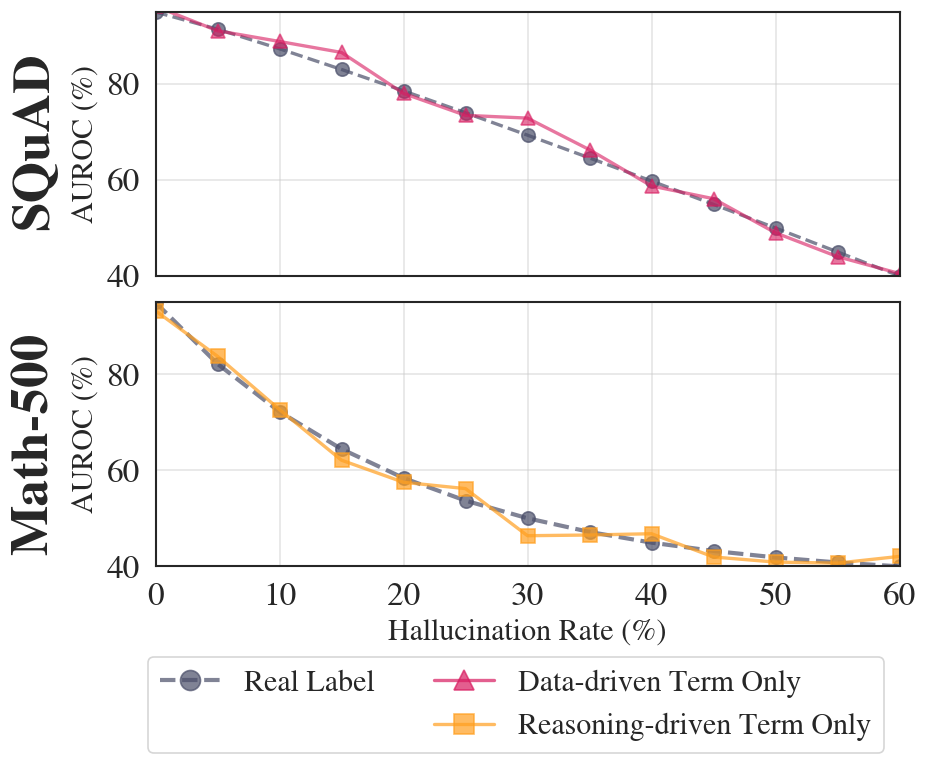}
    \vspace{-1.2em}
    \caption{Ablation results comparing individual terms with ground-truth trends on \texttt{SQuAD} (top) and \texttt{Math-500} (bottom).}
    \vspace{-1.5em}
    \label{fig:ablation}
\end{wrapfigure}
72.89\% AUPRC$_r$ on \texttt{HaluEval} with Llama2-7B, more than 10\% above the second best. 
Mid-sized models also exhibit clear gains (e.g., 79.01\% AUROC$_r$ on \texttt{GSM8K}), while even large-scale models like Llama2-70B see steady improvements (e.g., 83.8\% AUROC$_r$ on \texttt{SQuAD}). 
Overall, \model{} benefits most on small backbones while maintaining consistent advantages across scales.

\textbf{Q3: How does each term capture trends across task families?}
As shown in \Cref{fig:ablation}, each term faithfully tracks the ground-truth trend within its respective task family. On data-centric \texttt{SQuAD}, the \emph{data-driven term} closely follows the dashed gold curve across the variant hallucination rate, capturing the smooth AUROC decline. On reasoning-oriented \texttt{MATH-500}, the \emph{reasoning-driven term} mirrors the monotonic AUROC drop as reasoning drift increases. These results show that each term is well matched to its task family and faithfully tracks performance trends as hallucination rates rise.

\begin{table*}[h!]
\centering
\scriptsize
\caption{Performance comparison across backbone scales (small, mid-sized, and large) on three benchmarks: \texttt{SQuAD}, \texttt{GSM8K}, \texttt{HaluEval}. We highlight the \textbf{first} and \underline{second} best results.}
\resizebox{\textwidth}{!}{
\begin{tabularx}{0.96\textwidth}{@{}l@{\hspace{0.5em}}l@{\hspace{0em}}c@{\hspace{0.2em}}c@{\hspace{0.2em}}c@{\hspace{0em}}c@{\hspace{0.6em}}c@{\hspace{0.2em}}c@{\hspace{0.2em}}c@{\hspace{0em}}c@{\hspace{0.6em}}c@{\hspace{0.2em}}c@{\hspace{0.2em}}c@{\hspace{0em}}c@{\hspace{0.6em}}c@{\hspace{0.2em}}c@{\hspace{0em}}c@{\hspace{0.2em}}c@{}}%
\toprule
 & & \multicolumn{4}{c}{\textbf{Llama2-7B}} & \multicolumn{4}{c}{\textbf{Llama-3-8B}} & \multicolumn{4}{c}{\textbf{Llama2-13B}} & \multicolumn{4}{c}{\textbf{Llama2-70B}} \\
\cmidrule(lr){3-6} \cmidrule(lr){7-10} \cmidrule(lr){11-14} \cmidrule(lr){15-18}
& & \rotatebox{55}{\tiny{AUROC$_r$}} & \rotatebox{55}{\tiny{AUPRC$_r$}} & \rotatebox{60}{\tiny{AUROC$_\text{llm}$}} & \rotatebox{55}{\tiny{AUPRC$_\text{llm}$}}
  & \rotatebox{55}{\tiny{AUROC$_r$}} & \rotatebox{55}{\tiny{AUPRC$_r$}} & \rotatebox{60}{\tiny{AUROC$_\text{llm}$}} & \rotatebox{55}{\tiny{AUPRC$_\text{llm}$}}
  & \rotatebox{55}{\tiny{AUROC$_r$}} & \rotatebox{55}{\tiny{AUPRC$_r$}} & \rotatebox{60}{\tiny{AUROC$_\text{llm}$}} & \rotatebox{55}{\tiny{AUPRC$_\text{llm}$}}
  & \rotatebox{55}{\tiny{AUROC$_r$}} & \rotatebox{55}{\tiny{AUPRC$_r$}} & \rotatebox{60}{\tiny{AUROC$_\text{llm}$}} & \rotatebox{55}{\tiny{AUPRC$_\text{llm}$}} \\
\midrule
\multirow{12}{*}{\rotatebox{90}{\textbf{SQuAD}}}
 & \model        & \textbf{81.05} & \textbf{77.16} & \textbf{71.18} & \textbf{64.38} & \textbf{79.56} & \textbf{78.29} & \textbf{67.97} & \textbf{63.27} & \textbf{81.45} & \textbf{78.39} & \textbf{64.39} & \textbf{65.07} & \textbf{83.8}	&\textbf{81.77}	&\textbf{70.46}	&\textbf{73.24}\\
 & Inside       & \underline{73.63} & \underline{75.74} & \underline{65.22} & \underline{59.11} & \underline{76.13} & \underline{72.44} & \underline{65.62} & \underline{62.94} & \underline{74.68} & \underline{74.81} & \underline{61.01} & \underline{59.51} & \underline{81.24}	&\underline{75.09}	&\underline{69.48}&	\underline{62.4}\\
 & MIND          & 64.57 & 61.11 & 52.39 & 53.13 & 62.29 & 59.58 & 44.49 & 48.61 & 68.64 & 66.95 & 54.92 & 52.49 & 73.46	&71.71	&57.76&	56.77\\
 & Perplexity    & 63.93 & 61.77 & 46.97 & 48.2 & 70.51 & 67.51 & 55.71 & 52,68 & 70.19 & 69.22 & 60.33 & 54.82 & 74.23	&70.88	&62.24	&58.05\\
 & LN-Entropy    & 65.96 & 64.22 & 53.43 & 52.84 & 63.7 & 60.4 & 46.19 & 42.85 & 61.66 & 59.16 & 49.05 & 46.27 & 72.44&68.91	&56.77&	52.63\\
 & Energy        & 59.83 & 56.11 & 46.19 & 43.18 & 64.41 & 61.02 & 56.17 & 46.21& 61.02 & 59.73 & 48.26 & 42.08 & 69.01	&66.19&	58.44&	49.82\\
 & Semantic Ent. & 60.29 & 57.73 & 43.63 & 48.83 & 66.52 & 62.62 & 52.37 & 52.7 & 70.58 & 67.22 & 53.31 & 52.94 & 72.01&	68.51&56.49&50.9\\
 & Lexical Sim.  & 70.31 & 69.08 & 53.97 & 53.31 & 66.43 & 63.56 & 53.19 & 50.96 & 68.53 & 67.42 & 50.73 & 54.12 & 68.95&67.91&60.52&56.56\\
 & \tiny SelfCheckGPT  & 68.26 & 67.09 & 60.06 & 57.31 & 73.99 & 72.15 & 65.26 & 54.02 & 65.47 & 61.65 & 53.12 & 49.89 & 73.07&70.49&56.59&54.65\\
 & RACE          & 71.35 & 69.23 & 59.18 & 54.73 & 68.17 & 66.02 & 54.65 & 53.06 & 64.19 & 60.45 & 47.53 & 45.66 & 64.05&62.39&54.38&50.07\\
 & P(true)       & 62.55 & 61.09 & 46.84 & 52.32 & 67.42 & 63.94 & 55.35 & 47.52 & 71.56 & 68.4 & 57.51 & 45.66 & 66.81&62.71&57.43&46.85\\
 & FActScore     & 70.32 & 68.63 & 58.13 & 53.01 & 71.2 & 69.45 & 61.92 & 54.91 & 66.65 & 63.2 & 56.41 & 53.42 &68.33&65.26&56.93&48.46\\
\midrule
\multirow{12}{*}{\rotatebox{90}{\textbf{GSM8K}}}
 & \model        & \textbf{75.89} & \textbf{72.83} & \textbf{62.29} & \textbf{63.46} & \textbf{75.2} & \textbf{72.9} & \textbf{63.62} & \textbf{61.79} & \textbf{79.01} & \textbf{76.73} & \textbf{64.38}& \textbf{64.97} &  \textbf{77.33}	&\textbf{73.97}&	\textbf{60.48}&	\textbf{61.26}\\
 & Inside        & \underline{74.61} & \underline{68.35} & \underline{58.57} & \underline{62.58} & \underline{73.73} & \underline{67.51} & \underline{56.02} & \underline{57.28} & \underline{75.79} & \underline{76.26} & \underline{60.91} & \underline{59.77} & \underline{72.3}	&\underline{72.26}	&\underline{54.49}&	\underline{58.39}\\
 & MIND          & 65.88 & 63.4 & 48.28 & 48.17 & 66.57 & 65.55 & 48.84 & 53.4 & 61.49 & 59.55 & 51.63 & 51.45 & 66.41	&63.44&	52.05&	53.57\\
 & Perplexity    & 66.23 & 64.1 & 53.52 & 52.31 & 57.61 & 53.63 & 41.37 & 41.59 & 60.96 & 58.67 & 46.27 & 47.44 & 64.32&62.81&51.15&51.3\\
 & LN-Entropy    & 59.45& 55.95 & 43.04& 44.08 & 68.22 & 66.05 & 53.03 & 53.21 & 61.31 & 58.90 & 45.83 & 40.86 & 61.81&60.46&44.5&44.76\\
 & Energy        & 58.15& 54.71 & 43.65& 36.71 & 59.79 & 56.52 & 50.31 & 42.23 & 57.58 & 56.07 & 43.39 & 38.94 & 65.27&62.94&52.8&46.6\\
 & Semantic Ent. &57.95 & 54.68 & 42.78 & 41.95 & 66.9 & 64.81 & 50.47 & 55.36 & 62.72 & 59.09 & 49.33 & 44.35 & 60.63&57.01&46.22&40.24\\
 & Lexical Sim.  & 65.8 & 63.7 & 52.12& 54.07 & 63.29 & 59.87 & 53.17& 50.02 & 63.83 & 60.20 & 54.43 & 44.82 & 63.27&59.41&47.42&47.38\\
 & \tiny SelfCheckGPT  & 60.99 & 57.54 & 49.28 & 44.43 & 65.72 & 62.01& 54.49 & 50.34& 57.98 & 54.58 & 46.72 & 39.86 & 68.06&65.09&52.99&50.89\\
 & RACE  & 63.37 & 62.33 & 53.53 & 49.94 & 64.49 & 61.47& 53.28 & 47.55& 64.20 & 61.96 & 50.15 & 45.35 & 68.35&66.66&50.41&51.16\\
 & P(true)  & 65.95& 63.63 & 54.95 & 48.25 &62.59 & 58.88 & 47.21 & 42.2 & 67.08 & 65.60 & 53.66 & 55.12 & 60.16&58.14&47.73&49.49\\
 & FActScore     & 56.69 & 53.71 & 45.78& 39.52 & 65.69 & 61.95& 53.69 & 46.06 & 55.76 & 54.17 & 44.91 & 43.18 & 59.84&55.85&44.05&39.49\\
\midrule
\multirow{12}{*}{\rotatebox{90}{\textbf{HaluEval}}}
 & \model        & \textbf{75.72} & \textbf{72.89} & \textbf{66.65} & \textbf{63.15} & \textbf{73.43} & \textbf{71.19} & \textbf{64.95} & \textbf{54.8} & \textbf{78.15} & \textbf{74.15} & \textbf{65.39} & \textbf{61.14} & \textbf{80.79}	&\textbf{79.54}&	\textbf{67.68}	&\textbf{68.51}\\
 & Inside        & \underline{71.33} & \underline{67.63} & \underline{59.73} & \underline{53.15} & \underline{67.95} & \underline{64.93} & \underline{60.31} & \underline{52.21} & \underline{72.01} & \underline{71.97} & \underline{56.51} & \underline{60.64} & \underline{74.62}	&\underline{68.33}	&\underline{62.22}	&\underline{64.4}\\
 & MIND          & 54.8	&51.43	&44.15	&43.34 & 64.54 & 60.89 & 49.09 & 45.13 & 55.05 & 53.28 & 39.16 & 45.17 & 57.98&	56.01&	45.82&41.69\\
 & Perplexity    & 54.02 & 52.53& 38.76 & 40.51 & 61.31 & 59.36 & 50.62 & 46.01 & 54.99 & 51.39 & 42.64 & 35.64 & 62.85&60.59&48.29&43.85\\
 & LN-Entropy    & 59.47&58.33&50.2	&46.91& 64.89 & 60.72 & 51.78 & 46.39 & 65.18 & 63.53 & 49.70 & 48.09 & 60.16&58.89&50.29&48.42\\
 & Energy        &62.29&	59.6&50.68&	42.24 &62.74 & 61.61& 50.17 & 52.01 & 60.54 & 59.04 & 43.53 & 50.37 & 60.13&58.44&48.79&48.01\\
 & Semantic Ent. &59.39&55.94&48.53&46.35 & 55.25 & 53.05 & 44.5 & 44.35 & 59.44 & 57.72 & 45.38 & 40.77 & 61.57&57.99&49.07&45.39\\
 & Lexical Sim.  & 63.61	&61.16&55.01&44.75& 56.59 & 55.39 & 44.45 & 45.57 & 53.46 & 52.06 & 41.34 & 40.57 & 64.37&60.92&54.29&50.86\\
 & \tiny SelfCheckGPT  & 64.29& 61.83&48.4 & 45.49& 65.44 & 63.13 & 57.02 & 48.23 & 65.24 & 63.52 & 53.71 & 54.33 & 57.12&55.26&40.5&43.06\\
 & RACE          &59.78 & 59.14& 48.1 &40.47 & 61.98 & 60.32 & 48.08 & 46.29 & 60.65 & 59.11 & 49.92 & 44.51 & 62.11&58.24	&40.5&43.06\\
 & P(true)       & 57.46& 54.8 & 41.84 & 40.47 & 56.32 & 54.04 & 42.55 & 43.75 & 65.77 & 63.01 & 49.98 & 45.47 & 55.75&54.94&44.14&43.97\\
 & FActScore     & 63.93& 61.33& 46.9& 51.87 & 61.73&57.85	&49.92&42.15 & 65.15 & 63.71 & 55.98 & 54.61 & 62.66&60.3&53.13&46.42\\
\bottomrule
\end{tabularx}
}
\vspace{-1.5em}
\label{tab:Q2}
\end{table*}

\subsection{Test-Time Inference} 
\label{sec:testtime}
Test-time reasoning remains challenging, as models need to generate coherent multi-step solutions without drifting into errors. To assess whether hallucination detection can mitigate this difficulty, we integrate detectors into beam search and evaluate Qwen2.5-Math-7B on \texttt{MATH-500} and Llama3.1-8B on \texttt{Natural}.\ As shown in \Cref{tab:test_time_results}, \model{} achieves the strongest gains: on \texttt{MATH-500}, it reaches 81.00\% accuracy, around 10\% higher than IO Prompt; on \texttt{Natural}, it attains 70.96\%, exceeding IO Prompt by 15.72\%. These results demonstrate that \model{} not only detects hallucinations but also strengthens test-time reasoning by guiding models toward more reliable solutions.

\begin{table*}[h!]
\centering
\scriptsize
\renewcommand{\tabcolsep}{2pt} 
\caption{Performance of hallucination score-guided test-time inference across reasoning tasks. We highlight the \textbf{first} and \underline{second} best results.}
\resizebox{\textwidth}{!}{%
\begin{tabularx}{\textwidth}{>{\centering\arraybackslash}p{0.1\textwidth} *{12}{>{\centering\arraybackslash}p{0.0635\textwidth}}}
\toprule
\textbf{Dataset} 
& \tiny{IO Prompt} & \textbf{\tiny{Ours}} & \tiny{Inside} & \tiny{MIND} & \tiny{Perplexity} & \tiny{LN-Entropy} & \tiny{Energy} & \tiny{Semantic Ent.} & \tiny{SelfCheck- GPT} & \tiny{RACE} & \tiny{P(true)} & \tiny{FActScore} \\
\midrule
MATH-500
& 72.70 
& \textbf{81.00} 
& 74.90
& 77.10 
& 77.10 
& 76.20 
& \underline{78.00} 
& 72.50 
& 74.00
& 75.10 
& 67.10 
& 71.60 \\
Natural
& 55.24 
& \textbf{70.96}
& 67.42 
& 68.32 
& 67.51
& 68.04 
& \underline{68.59}
& 68.10 
& 65.68
& 66.90
& 68.16 
& 67.74 \\
\bottomrule
\end{tabularx}%
}
\vspace{-1.5em}
\label{tab:test_time_results}
\end{table*}
\subsection{Case Study}
\label{sec:case}
Fine-grained hallucinations-lexically similar yet semantically incorrect outputs-pose a particular challenge for detection. To evaluate whether \model{} can comprehensively capture such subtle errors, we use the PAWS dataset~\citep{zhang2019paws}, which contrasts paraphrases with high surface overlap but divergent meanings. Following~\citet{li2025principled}, we adopt ROUGE-based reference signals for evaluation (\Cref{tab:semantic_results}). Across model scales, \model{} consistently surpasses baselines: it achieves 90.18\% AUROC and 87.64\% AUPRC on Llama2-70B, and 91.24\% AUROC and 88.53\% AUPRC on QwQ-32B-exceeding the next-best method by nearly five points. Even on GPT-2, it leads with 83.27\% AUROC and 80.46\% AUPRC. These results confirm \model{}'s effectiveness in capturing fine-grained semantic inconsistencies beyond benchmark settings.
\vspace{-0.5em}
\begin{table*}[h!]
\centering
\scriptsize
\caption{Results on PAWS measuring semantic hallucination detection with Llama-3.2-3B, Llama2-70B, and QwQ-32B. We highlight the \textbf{first} and \underline{second} best results.}
\label{tab:semantic_results_rotated}
\renewcommand{\tabcolsep}{1.5pt} 
\resizebox{\textwidth}{!}{%
\begin{tabularx}{\textwidth}{ll *{12}{>{\centering\arraybackslash}p{0.062\textwidth}}}
\toprule
& Method & \textbf{Ours} & \tiny{Inside} & \tiny{MIND} & \tiny{Perplexity} & \tiny{LN-Entropy} & \tiny{Energy} 
& \tiny{Semantic Ent.} & \tiny{Lexical Sim.} & \tiny{SelfCheck- GPT} & \tiny{RACE} & \tiny{P(true)} & \tiny{FActScore} \\
\midrule
\multirow{2}{*}{{Llama3.2}} 
& AUROC & \textbf{85.63} & \underline{80.46} & 78.93 & 71.27 & 72.19 & 73.05 &75.11 & 64.58 & 77.82 & 79.47 & 73.56 & 68.44\\
& AUPRC & \textbf{82.14} & \underline{77.28} &75.41&67.55 &68.34 & 70.22 & 72.41& 59.67 &73.41 & 76.28 & 70.43 & 63.58\\
\midrule
\multirow{2}{*}{{Llama2}} 
& AUROC & \textbf{90.18} & \underline{85.47} & 83.92 & 75.68 & 76.23 & 77.14 & 79.06 & 68.35 & 82.71 & 84.26 & 77.39 & 72.62 \\
& AUPRC & \textbf{87.64} & \underline{82.38} & 81.06 & 71.42 & 72.59 & 74.28 & 76.32 & 63.44 & 78.89 & 81.73 & 74.18 & 67.58 \\
\midrule
\multirow{2}{*}{{QwQ}} 
& AUROC & \textbf{91.24} & 85.41 & 84.56 & 76.72 & 77.43 & 78.29 & 80.42 & 69.54 & 83.59 & \underline{86.38} & 78.53 & 73.46 \\
& AUPRC & \textbf{88.53} & 82.27 & 81.37 & 72.63 & 73.29 & 75.44 & 77.18 & 64.27 & 79.42 & \underline{83.41} & 75.21 & 68.32 \\
\bottomrule
\label{tab:semantic_results}
\end{tabularx}}
\vspace{-3em}
\end{table*}

\section{Related Work}
\label{sec:related}
In this section, we review prior hallucination-detection methods by their detection target--\emph{Data-driven hallucinations} and \emph{reasoning-driven hallucinations}.

\textbf{Detecting Data-Driven Hallucinations.}
Recent work has shown that internal activations encode rich indicators of such flaws. \citet{cheninside} proposed \textsc{Eigenscore}, which computes statistics of hidden representations from the eigen matrix to estimate hallucination risk.  
\citet{suunsuper} introduced \textsc{MIND}, an unsupervised detector that models temporal dynamics of hidden states without requiring labels, along with \textsc{HELM} benchmark to enable standardized evaluation.  
\citet{azaria2023internal} demonstrated using linear probes on intermediate states to predict truthfulness.

\textbf{Detecting Reasoning-Driven Hallucinations.}
There are other works targeting inference-time inconsistencies during generation-such as logical errors, instability across decoding steps, or temporal drift in extended outputs.  
\citet{manakulselfcheck} proposed \textsc{SelfCheckGPT}, which assesses self-consistency by sampling multiple candidate generations and measuring their alignment using entailment and lexical overlap.  
\citet{kalai2024calib} introduced a suite of calibration-based uncertainty scores designed to capture hallucination risk directly from output distributions.  
\citet{ding2025d2hscor} proposed \textsc{ReActScore}, which integrates entropy with intermediate reasoning traces to detect failures in multi-step decision-making.  
\textsc{FActScore}~\citep{minfactscore} decomposes outputs into atomic factual units and verifies each against retrieved passages using entailment-based scoring.

\section{Conclusion}
\label{sec:conclusion}

The reliability of LLMs is often undermined by hallucinations, which arise from two main sources: \emph{data-driven}, caused by flawed knowledge acquired during training, and \emph{reasoning-driven}, stemming from inference-time instabilities in multi-step generation. Although these hallucinations frequently evolve in practice, existing detectors usually target only one source and lack a solid theoretical foundation. 
To address this gap, we propose a unified theoretical framework--a \textit{Hallucination Risk Bound}, which formally decomposes hallucination risk into data-driven and reasoning-driven components, offering a principled view of how hallucinations emerge and evolve during generation. Building on this foundation, we introduce \textbf{\model}, a NTK-based score that measures sensitivity to semantic perturbations and captures internal instabilities, thereby enabling holistic detection of both data-driven and reasoning-driven hallucinations. We evaluate \model{} across 10 diverse benchmarks, 11 competitive baselines, and 9 popular LLM backbones, where it consistently achieves state-of-the-art performance, demonstrating robustness and practical efficacy. Looking forward, leveraging HalluGuard's sensitivity to error propagation offers a promising pathway for developing prognostic indicators in interactive multi-turn dialogues, enabling systems to predict and preempt hallucinations before they fully manifest.

\section*{Reproducibility Statement} 
We have taken several measures to ensure the reproducibility of our work. A complete description of the theoretical framework, including the formal assumptions and proofs of the Hallucination Risk Bound, is provided in \Cref{sec:methodology} and \Cref{appendix:proof}. Detailed experimental settings and evaluation protocols are documented in \Cref{sec:experiments} and \Cref{appendix:exp}, covering all 10 benchmarks, 11 baselines, and 9 LLM backbones. Together, these resources ensure that both our theoretical claims and empirical results can be independently validated and extended by the community.

\section*{Ethics Statement}
This study is based exclusively on publicly available datasets and open-source large language models, and does not involve human subjects or the use of private data. All scientific concepts, methodological designs, experimental implementations, and resulting conclusions remain entirely the responsibility of the authors.

\section*{Acknowledgements}
We thank the anonymous reviewers for their constructive comments. This work is supported by the National Science Foundation under Award No. IIS-2339989 and No. 2406439, DARPA under contract No. HR00112490370 and No. HR001124S0013, U.S. Department of Homeland Security under Grant Award No. 17STCIN00001-08-00, Amazon-Virginia Tech Initiative for Efficient and Robust Machine Learning, Amazon AWS, Google, Cisco, 4-VA, Commonwealth Cyber Initiative, National Surface Transportation Safety Center for Excellence, and Virginia Tech. The views and conclusions are those of the authors and should not be interpreted as representing the official policies of the funding agencies or the government.
\bibliographystyle{iclr2026_conference} \bibliography{iclr2026_conference} 

@article{cheninside,
  title={INSIDE: LLMs' Internal States Retain the Power of Hallucination Detection},
  author={Chao Chen and Kai Liu and Ze Chen and Yi Gu and Yue Wu and Mingyuan Tao and Zhihang Fu and Jieping Ye},
  journal={International Conference on Learning Representations(ICLR)},
  year={2024},
  eprint={2402.03744}
}

@article{suunsuper,
      title={Unsupervised Real-Time Hallucination Detection based on the Internal States of Large Language Models}, 
      author={Weihang Su and Changyue Wang and Qingyao Ai and Yiran HU and Zhijing Wu and Yujia Zhou and Yiqun Liu},
      year={2024},
      eprint={2403.06448},
      journal = {Association for Computational Linguistics(ACL)} 
}

@inproceedings{azaria2023internal,
  title={The Internal State of an LLM Knows When It’s Lying},
  author={Azaria, Amos and Mitchell, Tom},
  booktitle={Findings of the Association for Computational Linguistics: Conference on Empirical Methods in Natural Language Processing(EMNLP)},
  year={2023}
}

@article{manakulselfcheck,
      title={SelfCheckGPT: Zero-Resource Black-Box Hallucination Detection for Generative Large Language Models}, 
      author={Potsawee Manakul and Adian Liusie and Mark J. F. Gales},
      year={2023},
      eprint={2303.08896},
      journal = {Conference on Empirical Methods in Natural Language Processing(EMNLP)}
}

@article{minfactscore,
      title={FActScore: Fine-grained Atomic Evaluation of Factual Precision in Long Form Text Generation}, 
      author={Sewon Min and Kalpesh Krishna and Xinxi Lyu and Mike Lewis and Wen-tau Yih and Pang Wei Koh and Mohit Iyyer and Luke Zettlemoyer and Hannaneh Hajishirzi},
      year={2023},
      eprint={2305.14251},
      journal = {Conference on Empirical Methods in Natural Language Processing(EMNLP)} 
}

@article{niu2024ragtruth,
      title={RAGTruth: A Hallucination Corpus for Developing Trustworthy Retrieval-Augmented Language Models}, 
      author={Cheng Niu and Yuanhao Wu and Juno Zhu and Siliang Xu and Kashun Shum and Randy Zhong and Juntong Song and Tong Zhang},
      year={2024},
      eprint={2401.00396},
      journal = {Association for Computational Linguistics(ACL)} 
}

@inproceedings{Kurt202retrieval,
  title={Retrieval Augmentation Reduces Hallucination in Conversation},
  author={Kurt Shuster and Spencer Poff and Moya Chen and Douwe Kiela and Jason Weston},
  booktitle={Conference on Empirical Methods in Natural Language Processing(EMNLP)},
  year={2021}
}

@article{zeng2025lensllm,
  title={LENSLLM: Unveiling Fine‑Tuning Dynamics for LLM Selection},
  author={Zeng, Xinyue and Wang, Haohui and Lin, Junhong and Wu, Jun and Cody, Tyler and Zhou, Dawei},
  journal={International Conference on Machine Learning(ICML)},
  year={2025},
  eprint={2505.03793}
}

@misc{ding2025d2hscor,
      title={D$^2$HScore: Reasoning-Aware Hallucination Detection via Semantic Breadth and Depth Analysis in LLMs}, 
      author={Yue Ding and Xiaofang Zhu and Tianze Xia and Junfei Wu and Xinlong Chen and Qiang Liu and Liang Wang},
      year={2025},
      eprint={2509.11569},
      archivePrefix={arXiv},
      primaryClass={cs.CL},
      url={https://arxiv.org/abs/2509.11569}, 
}

@article{kalai2024calib,
      title={Calibrated Language Models Must Hallucinate}, 
      author={Adam Tauman Kalai and Santosh S. Vempala},
      year={2024},
      journal = {ACM Symposium on Theory of Computing (STOC)},
      eprint={2311.14648},
}

@article{ji2023survey,
   title={Survey of Hallucination in Natural Language Generation},
   volume={55},
   ISSN={1557-7341},
   url={http://dx.doi.org/10.1145/3571730},
   DOI={10.1145/3571730},
   number={12},
   journal={ACM Computing Surveys},
   publisher={Association for Computing Machinery (ACM)},
   author={Ji, Ziwei and Lee, Nayeon and Frieske, Rita and Yu, Tiezheng and Su, Dan and Xu, Yan and Ishii, Etsuko and Bang, Ye Jin and Madotto, Andrea and Fung, Pascale},
   year={2023},
   month=mar, pages={1–38} }

@article{bommasani2021opportunities,
      title={On the Opportunities and Risks of Foundation Models}, 
      author={Rishi Bommasani and Drew A. Hudson and Ehsan Adeli and Russ Altman and Simran Arora and Sydney von Arx and Michael S. Bernstein and Jeannette Bohg and Antoine Bosselut and Emma Brunskill and Erik Brynjolfsson and Shyamal Buch and Dallas Card and Rodrigo Castellon and Niladri Chatterji and Annie Chen and Kathleen Creel and Jared Quincy Davis and Dora Demszky and Chris Donahue and Moussa Doumbouya and Esin Durmus and Stefano Ermon and John Etchemendy and Kawin Ethayarajh and Li Fei-Fei and Chelsea Finn and Trevor Gale and Lauren Gillespie and Karan Goel and Noah Goodman and Shelby Grossman and Neel Guha and Tatsunori Hashimoto and Peter Henderson and John Hewitt and Daniel E. Ho and Jenny Hong and Kyle Hsu and Jing Huang and Thomas Icard and Saahil Jain and Dan Jurafsky and Pratyusha Kalluri and Siddharth Karamcheti and Geoff Keeling and Fereshte Khani and Omar Khattab and Pang Wei Koh and Mark Krass and Ranjay Krishna and Rohith Kuditipudi and Ananya Kumar and Faisal Ladhak and Mina Lee and Tony Lee and Jure Leskovec and Isabelle Levent and Xiang Lisa Li and Xuechen Li and Ece Kamar and Michal Kosinski and Ryan Chi-Ying Hsieh and Drew A. Linsley and Long O. Mai and Nikolay Manchev and Christopher D. Manning and Yian Yin and Christopher J. N. de M. L. Matthews and Lucia Mondragon and Ognjen Oreskovic and Mark Sabini and Yusuf Sahin and Clark Barrett and Christopher Potts and James Y. Zou and Jiajun Wu and Percy Liang},
      journal={ArXiv},
      year={2021},
      url={https://crfm.stanford.edu/assets/report.pdf}
}

@article{thirunavukarasu2023large,
  title={Large language models in medicine},
  author={Thirunavukarasu, Arun James and Ting, Darren Shu Jeng and Elangovan, Kavya and Gutierrez, Lio and Tan, Teng Fong and Ting, Daniel Shu Wei},
  journal={Nature Medicine},
  volume={29},
  number={8},
  pages={1930--1940},
  year={2023},
  publisher={Nature Publishing Group UK}
}

@article{huang2023survey,
      title={A Survey on Hallucination in Large Language Models: Principles, Taxonomy, Challenges, and Open Questions}, 
      author={Lei Huang and Weijiang Yu and Weitao Wang and Yujia Wang and Shi-Qi Chen and Ju-Hua Wang},
      year={2025},
      journal = {ACM Transactions on Information Systems},
      eprint={2311.05232},
}

@article{zhang2023language,
    title={How Language Model Hallucinations Can Snowball},
    author={Muru Zhang and Ofir Press and William Merrill and Alisa Liu and Noah A. Smith},
    year={2023},
    eprint={2305.13534},
    journal = {International Conference on Machine Learning(ICML)}
}

@article{zhong2024investigating,
    title={Investigating and Mitigating the Multimodal Hallucination Snowballing in Large Vision-Language Models},
    author={Weihong Zhong and Xiaocheng Feng and Liang Zhao and Qiming Li and Lei Huang and Yuxuan Gu and Weitao Ma and Yuan Xu and Bing Qin},
    year={2024},
    eprint={2407.00569},
    journal = {Association for Computational Linguistics(ACL)}
}

@article{ren2022outofdistribution,
    title={Out-of-Distribution Detection and Selective Generation for Conditional Language Models},
    author={Jie Ren and Jiaming Luo and Yao Zhao and Kundan Krishna and Mohammad Saleh and Balaji Lakshminarayanan and Peter J. Liu},
    year={2023},
    eprint={2209.15558},
    journal = {International Conference on Learning Representations(ICLR)}
}

@article{malinin2020uncertainty,
    title={Uncertainty Estimation in Autoregressive Structured Prediction},
    author={Andrey Malinin and Mark Gales},
    year={2021},
    eprint={2002.07650},
    journal = {International Conference on Learning Representations(ICLR)}
}

@inproceedings{lin-etal-2022-towards,
    title = "Towards Collaborative Neural-Symbolic Graph Semantic Parsing via Uncertainty",
    author = "Lin, Zi  and
      Liu, Jeremiah Zhe  and
      Shang, Jingbo",
    editor = "Muresan, Smaranda  and
      Nakov, Preslav  and
      Villavicencio, Aline",
    booktitle = "Findings of the Association for Computational Linguistics(ACL)",
    month = may,
    year = "2022",
    address = "Dublin, Ireland",
    publisher = "Association for Computational Linguistics(ACL)",
    url = "https://aclanthology.org/2022.findings-acl.328/",
    doi = "10.18653/v1/2022.findings-acl.328",
    pages = "4160--4173"
}

@article{liu2020energybased,
    title={Energy-based Out-of-distribution Detection},
    author={Weitang Liu and Xiaoyun Wang and John D. Owens and Yixuan Li},
    year={2020},
    eprint={2010.03759},
    journal = {Conference on Neural Information Processing Systems(NeurIPS)}
}

@article{kuhn2023semantic,
    title={Semantic Uncertainty: Linguistic Invariances for Uncertainty Estimation in Natural Language Generation},
    author={Lorenz Kuhn and Yarin Gal and Sebastian Farquhar},
    year={2023},
    journal = {International Conference on Learning Representations(ICLR)},
    eprint={2302.09664}
}

@misc{kadavath2022language,
    title={Language Models (Mostly) Know What They Know},
    author={Saurav Kadavath and Tom Conerly and Amanda Askell and Tom Henighan and Dawn Drain and Ethan Perez and Nicholas Schiefer and Zac Hatfield-Dodds and Nova DasSarma and Eli Tran-Johnson and Scott Johnston and Sheer El-Showk and Andy Jones and Nelson Elhage and Tristan Hume and Anna Chen and Yuntao Bai and Sam Bowman and Stanislav Fort and Deep Ganguli and Danny Hernandez and Josh Jacobson and Jackson Kernion and Shauna Kravec and Liane Lovitt and Kamal Ndousse and Catherine Olsson and Sam Ringer and Dario Amodei and Tom Brown and Jack Clark and Nicholas Joseph and Ben Mann and Sam McCandlish and Chris Olah and Jared Kaplan},
    year={2022},
    eprint={2207.05221},
    archivePrefix={arXiv},
    primaryClass={cs.CL}
}

@misc{wang2025joint,
    title={Joint Evaluation of Answer and Reasoning Consistency for Hallucination Detection in Large Reasoning Models},
    author={Changyue Wang and Weihang Su and Qingyao Ai and Yiqun Liu},
    year={2025},
    eprint={2506.04832},
    archivePrefix={arXiv},
    primaryClass={cs.CL}
}

@article{dennstadt2025implementing,
  title={Implementing large language models in healthcare while balancing control, collaboration, costs and security},
  author={Dennst{\"a}dt, Fabio and Hastings, Janna and Putora, Paul Martin and Schmerder, Max and Cihoric, Nikola},
  journal={NPJ digital medicine},
  volume={8},
  number={1},
  pages={143},
  year={2025},
  publisher={Nature Publishing Group UK London}
}

@article{kattnig2024assessing,
  title={Assessing trustworthy AI: Technical and legal perspectives of fairness in AI},
  author={Kattnig, Markus and Angerschmid, Alessa and Reichel, Thomas and Kern, Roman},
  journal={Computer Law \& Security Review},
  volume={55},
  pages={106053},
  year={2024},
  publisher={Elsevier}
}

@inproceedings{lin2004rouge,
  title={Rouge: A package for automatic evaluation of summaries},
  author={Lin, Chin-Yew},
  booktitle={Text summarization branches out},
  pages={74--81},
  year={2004}
}

@article{janiak2025illusion,
    title={The Illusion of Progress: Re-evaluating Hallucination Detection in LLMs},
    author={Denis Janiak and Jakub Binkowski and Albert Sawczyn and Bogdan Gabrys and Ravid Shwartz-Ziv and Tomasz Kajdanowicz},
    year={2025},
    journal = {Conference on Empirical Methods in Natural Language Processing(EMNLP)},
    eprint={2508.08285}
}

@article{geman1992neural,
  title={Neural networks and the bias/variance dilemma},
  author={Geman, Stuart and Bienenstock, Elie and Doursat, Ren{\'e}},
  journal={Neural computation},
  volume={4},
  number={1},
  pages={1--58},
  year={1992},
  publisher={MIT Press}
}

@inproceedings{cea1964approximation,
  title={Approximation variationnelle des probl{\`e}mes aux limites},
  author={C{\'e}a, Jean},
  booktitle={Annales de l'institut Fourier},
  volume={14},
  pages={345--444},
  year={1964}
}

@article{zhang2025icr,
      title={ICR Probe: Tracking Hidden State Dynamics for Reliable Hallucination Detection in LLMs}, 
      author={Zhenliang Zhang and Xinyu Hu and Huixuan Zhang and Junzhe Zhang and Xiaojun Wan},
      year={2025},
      eprint={2507.16488},
      journal = {Association for Computational Linguistics(ACL)} 
}

@misc{wei2025shadows,
      title={Shadows in the Attention: Contextual Perturbation and Representation Drift in the Dynamics of Hallucination in LLMs}, 
      author={Zeyu Wei and Shuo Wang and Xiaohui Rong and Xuemin Liu and He Li},
      year={2025},
      eprint={2505.16894},
      archivePrefix={arXiv},
      primaryClass={cs.CL},
      url={https://arxiv.org/abs/2505.16894}, 
}

@article{liu2025morethinking,
      title={More Thinking, Less Seeing? Assessing Amplified Hallucination in Multimodal Reasoning Models}, 
      author={Chengzhi Liu and Zhongxing Xu and Qingyue Wei and Juncheng Wu and James Zou and Xin Eric Wang and Yuyin Zhou and Sheng Liu},
      year={2025},
      eprint={2505.21523},
      journal = {Conference on Neural Information Processing Systems(NeurIPS)} 
}

@article{sun2025mechanistic,
      title={Detection and Mitigation of Hallucination in Large Reasoning Models: A Mechanistic Perspective}, 
      author={Zhongxiang Sun and Qipeng Wang and Haoyu Wang and Xiao Zhang and Jun Xu},
      year={2025},
      eprint={2505.12886},
      journal = {Conference on Neural Information Processing Systems(NeurIPS) Workshop}
}

@article{jacot2018ntk,
      title={Neural Tangent Kernel: Convergence and Generalization in Neural Networks}, 
      author={Arthur Jacot and Franck Gabriel and Clément Hongler},
      year={2018},
      journal = {Conference on Neural Information Processing Systems(NeurIPS)},
      eprint={1806.07572}
}

@article{lee2020wide,
   title={Wide neural networks of any depth evolve as linear models under gradient descent
                  *},
   volume={2020},
   ISSN={1742-5468},
   url={http://dx.doi.org/10.1088/1742-5468/abc62b},
   DOI={10.1088/1742-5468/abc62b},
   number={12},
   journal={Journal of Statistical Mechanics: Theory and Experiment},
   publisher={IOP Publishing},
   author={Lee, Jaehoon and Xiao, Lechao and Schoenholz, Samuel S and Bahri, Yasaman and Novak, Roman and Sohl-Dickstein, Jascha and Pennington, Jeffrey},
   year={2020},
   month=dec, pages={124002} }

@article{ju2022gen,
      title={On the Generalization Power of the Overfitted Three-Layer Neural Tangent Kernel Model}, 
      author={Peizhong Ju and Xiaojun Lin and Ness B. Shroff},
      year={2022},
      journal = {Conference on Neural Information Processing Systems(NeurIPS)},
      eprint={2206.02047}
}

@book{trefethen2022numerical,
  title={Numerical linear algebra},
  author={Trefethen, Lloyd N and Bau, David},
  year={2022},
  publisher={Society for Industrial and Applied Mathematics(SIAM)}
}

@article{kwiat2019natural,
    title = "Natural Questions: A Benchmark for Question Answering Research",
    author = "Kwiatkowski, Tom  and
      Palomaki, Jennimaria  and
      Redfield, Olivia  and
      Collins, Michael  and
      Parikh, Ankur  and
      Alberti, Chris  and
      Epstein, Danielle  and
      Polosukhin, Illia  and
      Devlin, Jacob  and
      Lee, Kenton  and
      Toutanova, Kristina  and
      Jones, Llion  and
      Kelcey, Matthew  and
      Chang, Ming-Wei  and
      Dai, Andrew M.  and
      Uszkoreit, Jakob  and
      Le, Quoc  and
      Petrov, Slav",
    editor = "Lee, Lillian  and
      Johnson, Mark  and
      Roark, Brian  and
      Nenkova, Ani",
    journal = "Transactions of the Association for Computational Linguistics(TACL)",
    volume = "7",
    year = "2019",
    address = "Cambridge, MA",
    publisher = "MIT Press",
    url = "https://aclanthology.org/Q19-1026/",
    doi = "10.1162/tacl_a_00276",
    pages = "452--466"
}

@article{yang2018hotpot,
      title={HotpotQA: A Dataset for Diverse, Explainable Multi-hop Question Answering}, 
      author={Zhilin Yang and Peng Qi and Saizheng Zhang and Yoshua Bengio and William W. Cohen and Ruslan Salakhutdinov and Christopher D. Manning},
      year={2018},
      eprint={1809.09600},
      journal = {Conference on Empirical Methods in Natural Language Processing(EMNLP)}
}

@article{rajpurkar2016squad,
      title={SQuAD: 100,000+ Questions for Machine Comprehension of Text}, 
      author={Pranav Rajpurkar and Jian Zhang and Konstantin Lopyrev and Percy Liang},
      year={2016},
      eprint={1606.05250},
      journal = {Conference on Empirical Methods in Natural Language Processing(EMNLP)}
}

@misc{cobbe2021training,
      title={Training Verifiers to Solve Math Word Problems}, 
      author={Karl Cobbe and Vineet Kosaraju and Mohammad Bavarian and Mark Chen and Heewoo Jun and Lukasz Kaiser and Matthias Plappert and Jerry Tworek and Jacob Hilton and Reiichiro Nakano and Christopher Hesse and John Schulman},
      year={2021},
      eprint={2110.14168},
      archivePrefix={arXiv},
      primaryClass={cs.LG},
      url={https://arxiv.org/abs/2110.14168}, 
}

@article{hendrycks2021measuring,
      title={Measuring Mathematical Problem Solving With the MATH Dataset}, 
      author={Dan Hendrycks and Collin Burns and Saurav Kadavath and Akul Arora and Steven Basart and Eric Tang and Dawn Song and Jacob Steinhardt},
      journal = {Conference on Neural Information Processing Systems(NeurIPS)} , 
      year={2021},
      eprint={2103.03874},
}

@article{suzgun2022challenging,
      title={Challenging BIG-Bench Tasks and Whether Chain-of-Thought Can Solve Them}, 
      author={Mirac Suzgun and Nathan Scales and Nathanael Schärli and Sebastian Gehrmann and Yi Tay and Hyung Won Chung and Aakanksha Chowdhery and Quoc V. Le and Ed H. Chi and Denny Zhou and Jason Wei},
      year={2023},
      eprint={2210.09261},
      journal = {Findings of the Association for Computational Linguistics(ACL)} 
}

@article{lin2022truthfulqa,
      title={TruthfulQA: Measuring How Models Mimic Human Falsehoods}, 
      author={Stephanie Lin and Jacob Hilton and Owain Evans},
      year={2022},
      eprint={2109.07958},
      journal = {Association for Computational Linguistics(ACL)} 
}

@article{li2023halu,
      title={HaluEval: A Large-Scale Hallucination Evaluation Benchmark for Large Language Models}, 
      author={Junyi Li and Xiaoxue Cheng and Wayne Xin Zhao and Jian-Yun Nie and Ji-Rong Wen},
      year={2023},
      eprint={2305.11747},
      journal = {Conference on Empirical Methods in Natural Language Processing(EMNLP)}
}

@misc{touvron2023llama2,
      title={Llama 2: Open Foundation and Fine-Tuned Chat Models}, 
      author={Hugo Touvron and Louis Martin and Kevin Stone and Peter Albert and Amjad Almahairi and Yasmine Babaei and Nikolay Bashlykov and Soumya Batra and Prajjwal Bhargava and Shruti Bhosale and Dan Bikel and Lukas Blecher and Cristian Canton Ferrer and Moya Chen and Guillem Cucurull and David Esiobu and Jude Fernandes and Jeremy Fu and Wenyin Fu and Brian Fuller and Cynthia Gao and Vedanuj Goswami and Naman Goyal and Anthony Hartshorn and Saghar Hosseini and Rui Hou and Hakan Inan and Marcin Kardas and Viktor Kerkez and Madian Khabsa and Isabel Kloumann and Artem Korenev and Punit Singh Koura and Marie-Anne Lachaux and Thibaut Lavril and Jenya Lee and Diana Liskovich and Yinghai Lu and Yuning Mao and Xavier Martinet and Todor Mihaylov and Pushkar Mishra and Igor Molybog and Yixin Nie and Andrew Poulton and Jeremy Reizenstein and Rashi Rungta and Kalyan Saladi and Alan Schelten and Ruan Silva and Eric Michael Smith and Ranjan Subramanian and Xiaoqing Ellen Tan and Binh Tang and Ross Taylor and Adina Williams and Jian Xiang Kuan and Puxin Xu and Zheng Yan and Iliyan Zarov and Yuchen Zhang and Angela Fan and Melanie Kambadur and Sharan Narang and Aurelien Rodriguez and Robert Stojnic and Sergey Edunov and Thomas Scialom},
      year={2023},
      eprint={2307.09288},
      archivePrefix={arXiv},
      primaryClass={cs.CL},
      url={https://arxiv.org/abs/2307.09288}, 
}

@misc{grattafiori2024llama3,
      title={The Llama 3 Herd of Models}, 
      author={Aaron Grattafiori and Abhimanyu Dubey and Abhinav Jauhri and Abhinav Pandey and Abhishek Kadian and Ahmad Al-Dahle and Aiesha Letman and Akhil Mathur and Alan Schelten and Alex Vaughan and Amy Yang and Angela Fan and Anirudh Goyal and Anthony Hartshorn and Aobo Yang and Archi Mitra and Archie Sravankumar and Artem Korenev and Arthur Hinsvark and Arun Rao and Aston Zhang and Aurelien Rodriguez and Austen Gregerson and Ava Spataru and Baptiste Roziere and Bethany Biron and Binh Tang and Bobbie Chern and Charlotte Caucheteux and Chaya Nayak and Chloe Bi and Chris Marra and Chris McConnell and Christian Keller and Christophe Touret and Chunyang Wu and Corinne Wong and Cristian Canton Ferrer and Cyrus Nikolaidis and Damien Allonsius and Daniel Song and Danielle Pintz and Danny Livshits and Danny Wyatt and David Esiobu and Dhruv Choudhary and Dhruv Mahajan and Diego Garcia-Olano and Diego Perino and Dieuwke Hupkes and Egor Lakomkin and Ehab AlBadawy and Elina Lobanova and Emily Dinan and Eric Michael Smith and Filip Radenovic and Francisco Guzmán and Frank Zhang and Gabriel Synnaeve and Gabrielle Lee and Georgia Lewis Anderson and Govind Thattai and Graeme Nail and Gregoire Mialon and Guan Pang and Guillem Cucurell and Hailey Nguyen and Hannah Korevaar and Hu Xu and Hugo Touvron and Iliyan Zarov and Imanol Arrieta Ibarra and Isabel Kloumann and Ishan Misra and Ivan Evtimov and Jack Zhang and Jade Copet and Jaewon Lee and Jan Geffert and Jana Vranes and Jason Park and Jay Mahadeokar and Jeet Shah and Jelmer van der Linde and Jennifer Billock and Jenny Hong and Jenya Lee and Jeremy Fu and Jianfeng Chi and Jianyu Huang and Jiawen Liu and Jie Wang and Jiecao Yu and Joanna Bitton and Joe Spisak and Jongsoo Park and Joseph Rocca and Joshua Johnstun and Joshua Saxe and Junteng Jia and Kalyan Vasuden Alwala and Karthik Prasad and Kartikeya Upasani and Kate Plawiak and Ke Li and Kenneth Heafield and Kevin Stone and Khalid El-Arini and Krithika Iyer and Kshitiz Malik and Kuenley Chiu and Kunal Bhalla and Kushal Lakhotia and Lauren Rantala-Yeary and Laurens van der Maaten and Lawrence Chen and Liang Tan and Liz Jenkins and Louis Martin and Lovish Madaan and Lubo Malo and Lukas Blecher and Lukas Landzaat and Luke de Oliveira and Madeline Muzzi and Mahesh Pasupuleti and Mannat Singh and Manohar Paluri and Marcin Kardas and Maria Tsimpoukelli and Mathew Oldham and Mathieu Rita and Maya Pavlova and Melanie Kambadur and Mike Lewis and Min Si and Mitesh Kumar Singh and Mona Hassan and Naman Goyal and Narjes Torabi and Nikolay Bashlykov and Nikolay Bogoychev and Niladri Chatterji and Ning Zhang and Olivier Duchenne and Onur Çelebi and Patrick Alrassy and Pengchuan Zhang and Pengwei Li and Petar Vasic and Peter Weng and Prajjwal Bhargava and Pratik Dubal and Praveen Krishnan and Punit Singh Koura and Puxin Xu and Qing He and Qingxiao Dong and Ragavan Srinivasan and Raj Ganapathy and Ramon Calderer and Ricardo Silveira Cabral and Robert Stojnic and Roberta Raileanu and Rohan Maheswari and Rohit Girdhar and Rohit Patel and Romain Sauvestre and Ronnie Polidoro and Roshan Sumbaly and Ross Taylor and Ruan Silva and Rui Hou and Rui Wang and Saghar Hosseini and Sahana Chennabasappa and Sanjay Singh and Sean Bell and Seohyun Sonia Kim and Sergey Edunov and Shaoliang Nie and Sharan Narang and Sharath Raparthy and Sheng Shen and Shengye Wan and Shruti Bhosale and Shun Zhang and Simon Vandenhende and Soumya Batra and Spencer Whitman and Sten Sootla and Stephane Collot and Suchin Gururangan and Sydney Borodinsky and Tamar Herman and Tara Fowler and Tarek Sheasha and Thomas Georgiou and Thomas Scialom and Tobias Speckbacher and Todor Mihaylov and Tong Xiao and Ujjwal Karn and Vedanuj Goswami and Vibhor Gupta and Vignesh Ramanathan and Viktor Kerkez and Vincent Gonguet and Virginie Do and Vish Vogeti and Vítor Albiero and Vladan Petrovic and Weiwei Chu and Wenhan Xiong and Wenyin Fu and Whitney Meers and Xavier Martinet and Xiaodong Wang and Xiaofang Wang and Xiaoqing Ellen Tan and Xide Xia and Xinfeng Xie and Xuchao Jia and Xuewei Wang and Yaelle Goldschlag and Yashesh Gaur and Yasmine Babaei and Yi Wen and Yiwen Song and Yuchen Zhang and Yue Li and Yuning Mao and Zacharie Delpierre Coudert and Zheng Yan and Zhengxing Chen and Zoe Papakipos and Aaditya Singh and Aayushi Srivastava and Abha Jain and Adam Kelsey and Adam Shajnfeld and Adithya Gangidi and Adolfo Victoria and Ahuva Goldstand and Ajay Menon and Ajay Sharma and Alex Boesenberg and Alexei Baevski and Allie Feinstein and Amanda Kallet and Amit Sangani and Amos Teo and Anam Yunus and Andrei Lupu and Andres Alvarado and Andrew Caples and Andrew Gu and Andrew Ho and Andrew Poulton and Andrew Ryan and Ankit Ramchandani and Annie Dong and Annie Franco and Anuj Goyal and Aparajita Saraf and Arkabandhu Chowdhury and Ashley Gabriel and Ashwin Bharambe and Assaf Eisenman and Azadeh Yazdan and Beau James and Ben Maurer and Benjamin Leonhardi and Bernie Huang and Beth Loyd and Beto De Paola and Bhargavi Paranjape and Bing Liu and Bo Wu and Boyu Ni and Braden Hancock and Bram Wasti and Brandon Spence and Brani Stojkovic and Brian Gamido and Britt Montalvo and Carl Parker and Carly Burton and Catalina Mejia and Ce Liu and Changhan Wang and Changkyu Kim and Chao Zhou and Chester Hu and Ching-Hsiang Chu and Chris Cai and Chris Tindal and Christoph Feichtenhofer and Cynthia Gao and Damon Civin and Dana Beaty and Daniel Kreymer and Daniel Li and David Adkins and David Xu and Davide Testuggine and Delia David and Devi Parikh and Diana Liskovich and Didem Foss and Dingkang Wang and Duc Le and Dustin Holland and Edward Dowling and Eissa Jamil and Elaine Montgomery and Eleonora Presani and Emily Hahn and Emily Wood and Eric-Tuan Le and Erik Brinkman and Esteban Arcaute and Evan Dunbar and Evan Smothers and Fei Sun and Felix Kreuk and Feng Tian and Filippos Kokkinos and Firat Ozgenel and Francesco Caggioni and Frank Kanayet and Frank Seide and Gabriela Medina Florez and Gabriella Schwarz and Gada Badeer and Georgia Swee and Gil Halpern and Grant Herman and Grigory Sizov and Guangyi and Zhang and Guna Lakshminarayanan and Hakan Inan and Hamid Shojanazeri and Han Zou and Hannah Wang and Hanwen Zha and Haroun Habeeb and Harrison Rudolph and Helen Suk and Henry Aspegren and Hunter Goldman and Hongyuan Zhan and Ibrahim Damlaj and Igor Molybog and Igor Tufanov and Ilias Leontiadis and Irina-Elena Veliche and Itai Gat and Jake Weissman and James Geboski and James Kohli and Janice Lam and Japhet Asher and Jean-Baptiste Gaya and Jeff Marcus and Jeff Tang and Jennifer Chan and Jenny Zhen and Jeremy Reizenstein and Jeremy Teboul and Jessica Zhong and Jian Jin and Jingyi Yang and Joe Cummings and Jon Carvill and Jon Shepard and Jonathan McPhie and Jonathan Torres and Josh Ginsburg and Junjie Wang and Kai Wu and Kam Hou U and Karan Saxena and Kartikay Khandelwal and Katayoun Zand and Kathy Matosich and Kaushik Veeraraghavan and Kelly Michelena and Keqian Li and Kiran Jagadeesh and Kun Huang and Kunal Chawla and Kyle Huang and Lailin Chen and Lakshya Garg and Lavender A and Leandro Silva and Lee Bell and Lei Zhang and Liangpeng Guo and Licheng Yu and Liron Moshkovich and Luca Wehrstedt and Madian Khabsa and Manav Avalani and Manish Bhatt and Martynas Mankus and Matan Hasson and Matthew Lennie and Matthias Reso and Maxim Groshev and Maxim Naumov and Maya Lathi and Meghan Keneally and Miao Liu and Michael L. Seltzer and Michal Valko and Michelle Restrepo and Mihir Patel and Mik Vyatskov and Mikayel Samvelyan and Mike Clark and Mike Macey and Mike Wang and Miquel Jubert Hermoso and Mo Metanat and Mohammad Rastegari and Munish Bansal and Nandhini Santhanam and Natascha Parks and Natasha White and Navyata Bawa and Nayan Singhal and Nick Egebo and Nicolas Usunier and Nikhil Mehta and Nikolay Pavlovich Laptev and Ning Dong and Norman Cheng and Oleg Chernoguz and Olivia Hart and Omkar Salpekar and Ozlem Kalinli and Parkin Kent and Parth Parekh and Paul Saab and Pavan Balaji and Pedro Rittner and Philip Bontrager and Pierre Roux and Piotr Dollar and Polina Zvyagina and Prashant Ratanchandani and Pritish Yuvraj and Qian Liang and Rachad Alao and Rachel Rodriguez and Rafi Ayub and Raghotham Murthy and Raghu Nayani and Rahul Mitra and Rangaprabhu Parthasarathy and Raymond Li and Rebekkah Hogan and Robin Battey and Rocky Wang and Russ Howes and Ruty Rinott and Sachin Mehta and Sachin Siby and Sai Jayesh Bondu and Samyak Datta and Sara Chugh and Sara Hunt and Sargun Dhillon and Sasha Sidorov and Satadru Pan and Saurabh Mahajan and Saurabh Verma and Seiji Yamamoto and Sharadh Ramaswamy and Shaun Lindsay and Shaun Lindsay and Sheng Feng and Shenghao Lin and Shengxin Cindy Zha and Shishir Patil and Shiva Shankar and Shuqiang Zhang and Shuqiang Zhang and Sinong Wang and Sneha Agarwal and Soji Sajuyigbe and Soumith Chintala and Stephanie Max and Stephen Chen and Steve Kehoe and Steve Satterfield and Sudarshan Govindaprasad and Sumit Gupta and Summer Deng and Sungmin Cho and Sunny Virk and Suraj Subramanian and Sy Choudhury and Sydney Goldman and Tal Remez and Tamar Glaser and Tamara Best and Thilo Koehler and Thomas Robinson and Tianhe Li and Tianjun Zhang and Tim Matthews and Timothy Chou and Tzook Shaked and Varun Vontimitta and Victoria Ajayi and Victoria Montanez and Vijai Mohan and Vinay Satish Kumar and Vishal Mangla and Vlad Ionescu and Vlad Poenaru and Vlad Tiberiu Mihailescu and Vladimir Ivanov and Wei Li and Wenchen Wang and Wenwen Jiang and Wes Bouaziz and Will Constable and Xiaocheng Tang and Xiaojian Wu and Xiaolan Wang and Xilun Wu and Xinbo Gao and Yaniv Kleinman and Yanjun Chen and Ye Hu and Ye Jia and Ye Qi and Yenda Li and Yilin Zhang and Ying Zhang and Yossi Adi and Youngjin Nam and Yu and Wang and Yu Zhao and Yuchen Hao and Yundi Qian and Yunlu Li and Yuzi He and Zach Rait and Zachary DeVito and Zef Rosnbrick and Zhaoduo Wen and Zhenyu Yang and Zhiwei Zhao and Zhiyu Ma},
      year={2024},
      eprint={2407.21783},
      archivePrefix={arXiv},
      primaryClass={cs.AI},
      url={https://arxiv.org/abs/2407.21783}, 
}

@misc{zhang2022opt,
      title={OPT: Open Pre-trained Transformer Language Models}, 
      author={Susan Zhang and Stephen Roller and Naman Goyal and Mikel Artetxe and Moya Chen and Shuohui Chen and Christopher Dewan and Mona Diab and Xian Li and Xi Victoria Lin and Todor Mihaylov and Myle Ott and Sam Shleifer and Kurt Shuster and Daniel Simig and Punit Singh Koura and Anjali Sridhar and Tianlu Wang and Luke Zettlemoyer},
      year={2022},
      eprint={2205.01068},
      archivePrefix={arXiv},
      primaryClass={cs.CL},
      url={https://arxiv.org/abs/2205.01068}, 
}

@misc{jiang2023mistral7b,
      title={Mistral 7B}, 
      author={Albert Q. Jiang and Alexandre Sablayrolles and Arthur Mensch and Chris Bamford and Devendra Singh Chaplot and Diego de las Casas and Florian Bressand and Gianna Lengyel and Guillaume Lample and Lucile Saulnier and Lélio Renard Lavaud and Marie-Anne Lachaux and Pierre Stock and Teven Le Scao and Thibaut Lavril and Thomas Wang and Timothée Lacroix and William El Sayed},
      year={2023},
      eprint={2310.06825},
      archivePrefix={arXiv},
      primaryClass={cs.CL},
      url={https://arxiv.org/abs/2310.06825}, 
}

@misc{yang2024qwen,
      title={Qwen2 Technical Report}, 
      author={An Yang and Baosong Yang and Binyuan Hui and Bo Zheng and Bowen Yu and Chang Zhou and Chengpeng Li and Chengyuan Li and Dayiheng Liu and Fei Huang and Guanting Dong and Haoran Wei and Huan Lin and Jialong Tang and Jialin Wang and Jian Yang and Jianhong Tu and Jianwei Zhang and Jianxin Ma and Jianxin Yang and Jin Xu and Jingren Zhou and Jinze Bai and Jinzheng He and Junyang Lin and Kai Dang and Keming Lu and Keqin Chen and Kexin Yang and Mei Li and Mingfeng Xue and Na Ni and Pei Zhang and Peng Wang and Ru Peng and Rui Men and Ruize Gao and Runji Lin and Shijie Wang and Shuai Bai and Sinan Tan and Tianhang Zhu and Tianhao Li and Tianyu Liu and Wenbin Ge and Xiaodong Deng and Xiaohuan Zhou and Xingzhang Ren and Xinyu Zhang and Xipin Wei and Xuancheng Ren and Xuejing Liu and Yang Fan and Yang Yao and Yichang Zhang and Yu Wan and Yunfei Chu and Yuqiong Liu and Zeyu Cui and Zhenru Zhang and Zhifang Guo and Zhihao Fan},
      year={2024},
      eprint={2407.10671},
      archivePrefix={arXiv},
      primaryClass={cs.CL},
      url={https://arxiv.org/abs/2407.10671}, 
}

@article{radford2019language,
  title={Language models are unsupervised multitask learners},
  author={Radford, Alec and Wu, Jeffrey and Child, Rewon and Luan, David and Amodei, Dario and Sutskever, Ilya and others},
  journal={OpenAI blog},
  volume={1},
  number={8},
  pages={9},
  year={2019}
}

@misc{li2025principled,
      title={Principled Detection of Hallucinations in Large Language Models via Multiple Testing}, 
      author={Jiawei Li and Akshayaa Magesh and Venugopal V. Veeravalli},
      year={2025},
      eprint={2508.18473},
      archivePrefix={arXiv},
      primaryClass={cs.CL},
      url={https://arxiv.org/abs/2508.18473}, 
}

@article{zhang2019paws,
      title={PAWS: Paraphrase Adversaries from Word Scrambling}, 
      author={Yuan Zhang and Jason Baldridge and Luheng He},
      year={2019},
      eprint={1904.01130},
      journal = {The North American Chapter of the Association for Computational Linguistics(NAACL)} 
}

@article{wang2022supernatural,
      title={Super-NaturalInstructions: Generalization via Declarative Instructions on 1600+ NLP Tasks}, 
      author={Yizhong Wang and Swaroop Mishra and Pegah Alipoormolabashi and Yeganeh Kordi and Amirreza Mirzaei and Anjana Arunkumar and Arjun Ashok and Arut Selvan Dhanasekaran and Atharva Naik and David Stap and Eshaan Pathak and Giannis Karamanolakis and Haizhi Gary Lai and Ishan Purohit and Ishani Mondal and Jacob Anderson and Kirby Kuznia and Krima Doshi and Maitreya Patel and Kuntal Kumar Pal and Mehrad Moradshahi and Mihir Parmar and Mirali Purohit and Neeraj Varshney and Phani Rohitha Kaza and Pulkit Verma and Ravsehaj Singh Puri and Rushang Karia and Shailaja Keyur Sampat and Savan Doshi and Siddhartha Mishra and Sujan Reddy and Sumanta Patro and Tanay Dixit and Xudong Shen and Chitta Baral and Yejin Choi and Noah A. Smith and Hannaneh Hajishirzi and Daniel Khashabi},
      year={2022},
      eprint={2204.07705},
      journal={Conference on Empirical Methods in Natural Language Processing(EMNLP)}
}

@article{chen2024hall,
      title={Hallucination Detection: Robustly Discerning Reliable Answers in Large Language Models}, 
      author={Yuyan Chen and Qiang Fu and Yichen Yuan and Zhihao Wen and Ge Fan and Dayiheng Liu and Dongmei Zhang and Zhixu Li and Yanghua Xiao},
      year={2023},
      journal={Conference on Information and Knowledge Management(CIKM)},
      eprint={2407.04121}
}

@book{vershynin2018high,
  title={High-dimensional probability: An introduction with applications in data science},
  author={Vershynin, Roman},
  volume={47},
  year={2018},
  publisher={Cambridge university press}
}

@article{lee2020finit,
      title={Finite Versus Infinite Neural Networks: an Empirical Study}, 
      author={Jaehoon Lee and Samuel S. Schoenholz and Jeffrey Pennington and Ben Adlam and Lechao Xiao and Roman Novak and Jascha Sohl-Dickstein},
      year={2020},
      journal = {Conference on Neural Information Processing Systems(NeurIPS)},
      eprint={2007.15801}, 
}

@article{chizat2020lazy,
      title={On Lazy Training in Differentiable Programming}, 
      author={Lenaic Chizat and Edouard Oyallon and Francis Bach},
      year={2019},
      journal={Conference on Neural Information Processing Systems(NeurIPS)},
      eprint={1812.07956},
      
}

@article{ke2025early,
  title={Early warning of cryptocurrency reversal risks via multi-source data},
  author={Ke, Zong and Cao, Yuqing and Chen, Zhenrui and Yin, Yuchen and He, Shouchao and Cheng, Yu},
  journal={Finance Research Letters},
  pages={107890},
  year={2025},
  publisher={Elsevier}
}

@inproceedings{Li-hallucination-2023,
  title={Evaluating Object Hallucination in Large Vision-Language Models},
  author={Yifan Li and Yifan Du and Kun Zhou and Jinpeng Wang and Wayne Xin Zhao and Ji-Rong Wen},
  booktitle={Conference on Empirical Methods in Natural Language Processing(EMNLP)},
  year={2023},
  url={https://openreview.net/forum?id=xozJw0kZXF}
}

@inproceedings{huang-etal-2021-efficient,
    title = "Efficient Attentions for Long Document Summarization",
    author = "Huang, Luyang  and
      Cao, Shuyang  and
      Parulian, Nikolaus  and
      Ji, Heng  and
      Wang, Lu",
    booktitle = "The North American Chapter of the Association for Computational Linguistics: Human Language Technologies(NAACL)",
    month = jun,
    year = "2021",
    address = "Online",
    publisher = "Association for Computational Linguistics(ACL)",
    url = "https://aclanthology.org/2021.naacl-main.112",
    doi = "10.18653/v1/2021.naacl-main.112",
    pages = "1419--1436",
}

@article{narrativeqa,
author = {Tom\'a\v s Ko\v cisk\'y and Jonathan Schwarz and Phil Blunsom and
          Chris Dyer and Karl Moritz Hermann and G\'abor Melis and
          Edward Grefenstette},
title = {The {NarrativeQA} Reading Comprehension Challenge},
journal = {Transactions of the Association for Computational Linguistics(TACL)},
year = {2018},
}
\newpage

\appendix
\section{Proof of Hallucination Risk Bound}
\label{appendix:proof}

\subsection{Assumptions Validation}
We provide theoretical and practical justification for the assumptions adopted in \Cref{sec:bound}, which serve to ensure the well-posedness and interpretability of the proposed Hallucination Risk Bound. These assumptions follow standard practice in NTK-based analyses and stability theory, and are consistent with the empirical behavior observed in modern large language models.

\textbf{Setup} For completeness, we briefly recall the main notation used in \Cref{sec:bound}. Let $\mathcal{Y}$ denote the discrete metric space of finite-length token sequences. Let $U_h \subseteq \mathbb{R}^{d_h}$ be a $d_h$-dimensional Hilbert space equipped with inner product $\langle\cdot,\cdot\rangle$ and induced norm $\|\cdot\|$. The task-specific encoder $\Phi:\mathcal{Y}\to U_h$ is assumed to be $L_\Phi$-Lipschitz with respect to $d_{\mathcal{Y}}$. 

Given input $\mathbf{x}$, the model defines a decoding distribution $P_\theta(\cdot\mid \mathbf{x})$ over $\mathcal{Y}$, and we denote the embedded random variable by $u_h := \Phi(Y)$ where $Y \sim P_\theta(\cdot \mid \mathbf{x})$. For perturbations $\delta \in \mathbb{R}^r$ restricted to the local ball $\mathcal{B}_\rho$, the perturbed decoding distribution is denoted $P_\theta(\cdot\mid \mathbf{x},\delta)$, and the mean semantic response map is defined by
$G_Y(\delta) := \mathbb{E}_{Y\sim P_\theta(\cdot\mid \mathbf{x},\delta)}[\Phi(Y)],$
with Jacobian $J = DG_Y(0) \in \mathbb{R}^{d_h\times r}$. The NTK Gram matrix on embedded perturbations is denoted $\mathcal{K} \in \mathbb{R}^{r\times r}$, with eigenvalues $\lambda_1 \ge \dots \ge \lambda_r > 0$ and condition number $\kappa(\mathcal{K}) = \lambda_{\max}/\lambda_{\min}$. 
All expectations are taken with respect to the specified decoding distribution, and all norms are Euclidean unless otherwise stated.

\textbf{Assumption A1 (Integrability and well-defined expectation).} Assumption~A1 ensures that the semantic embedding  $\mathbb{E}_{Y\sim p_\theta(\cdot\mid \mathbf{x})}[\Phi(Y)]$ is well-defined as a Bochner expectation in the finite-dimensional Hilbert space $U_h$. The bounded second-moment condition guarantees that the expectation exists and is finite, which is a standard minimal requirement in stochastic analyses of neural network outputs. Such integrability assumptions are commonly adopted in NTK-based analyses \citep{jacot2018ntk, lee2020wide}, where control of second moments ensures stability of kernel spectra and well-posedness of linearized approximations.


\textbf{Assumption A2 (Lipschitz continuity of the encoder $\Phi$).}
Assumption~A2 imposes a controlled relationship between the discrete sequence space $(\mathcal{Y}, d_{\mathcal{Y}})$ and the continuous embedding space $U_h$. The $L_\Phi$-Lipschitz condition ensures that bounded perturbations in edit distance induce proportionally bounded deviations in semantic representation. Such Lipschitz regularity is standard in high-dimensional learning theory \citep{vershynin2018high} and is frequently invoked to establish stability under structured perturbations in representation learning. Importantly, this assumption is imposed only on the encoder map $\Phi$, not on the full autoregressive model.

\textbf{Assumption A3 (Local Fréchet smoothness of the mean semantic response).}
Assumption~A3 formalizes the local linearization principle underlying NTK theory. By requiring twice Fréchet differentiability of the mean response map
$
G_Y(\delta)
=
\mathbb{E}_{Y\sim P_\theta(\cdot\mid \mathbf{x},\delta)}[\Phi(Y)]
$
within the perturbation ball $\mathcal{B}_\rho$, we ensure that $G_Y$ admits a controlled second-order expansion with uniform curvature constant $H_\star$. This local quadratic remainder bound is consistent with classical finite-width NTK linearization results \citep{lee2020finit, chizat2020lazy}, while avoiding unrealistic global smoothness requirements. Crucially, the assumption is imposed only on the expected semantic response, not on the discrete decoding distribution itself.


\paragraph{Remark.}
Collectively, these assumptions provide a bridge between discrete autoregressive generation and continuous functional analysis. By restricting smoothness and curvature requirements to the localized perturbation neighborhood $\mathcal{B}_\rho$ and to the expectation-level map $G_Y$, we avoid imposing global regularity conditions over the infinite token space $\mathcal{Y}$. This localization ensures that the Hallucination Risk Bound is derived under mathematically controlled conditions while remaining aligned with the practical inference dynamics of large-scale language models.

\subsection{Proof of \Cref{rem:hallucination-bound}}
\label{appendix:hallucination-proof}

We restate the main inequality from \Cref{rem:hallucination-bound}. Note that due to the stochastic nature of autoregressive decoding, the bound holds with high probability. With probability at least $1 - \delta$ over the generation process, the total hallucination risk satisfies:
\begin{equation}
\|u^*-u_h\| \le \underbrace{\left[1+k_{\mathrm{pt}}\log\mathcal{O}(P,L)+k\frac{\epsilon_{\mathrm{mismatch}}}{\mathrm{Signal}_k}\right]\inf_{u\in U_h}\|u^*-u\|}_{\text{Data-driven term}} + \underbrace{\vphantom{\left[\frac{A}{B}\right]}|\mathcal{L}|\exp\left(-\frac{K\epsilon^2}{C}\right)\alpha\left(e^{\beta T}-1\right)}_{\text{Reasoning-driven term}}.
\end{equation}

\paragraph{Step 1: Triangle inequality split (Bias-Variance Decomposition).}
Let $\bar{u} := \mathbb{E}[u_h]$ be the expected semantic representation under the decoding distribution. By the triangle inequality in $U_h$, we decompose the hallucination risk into approximation error (bias) and stochastic residual (variance):
\[
\|u^*-u_h\| = \|u^* - \bar{u} + \bar{u} - u_h\| \le \|u^* - \bar{u}\| + \|u_h - \bar{u}\|.
\]

\paragraph{Step 2: Approximation term via Céa's lemma.}
To bound the deterministic approximation error, we cast the model's expected representation $\bar{u}$ as the solution to a variational problem in the Hilbert space $U_h$. 
Let $a(u,v) := \langle u, \mathcal{K} v \rangle_{U_h}$ denote the coercive bilinear form induced by the Neural Tangent Kernel (NTK) Gram operator $\mathcal{K}$, and let $f(v) := \langle u^*, \mathcal{K} v \rangle_{U_h}$ be the bounded linear functional defining the target projection. 
Assuming $\bar{u}$ acts as the Galerkin projection of the target $u^*$ onto the trainable hypothesis space, it satisfies the weak formulation $a(\bar{u}, v) = f(v)$ for all $v \in U_h$. By Céa's lemma, the projection error is bounded by:
\[
\|u^* - \bar{u}\| \le \frac{\Lambda}{\gamma} \inf_{u \in U_h} \|u^* - u\|,
\]
where $\Lambda$ and $\gamma$ are the continuity and coercivity constants of the NTK-induced bilinear form $a(\cdot,\cdot)$, respectively.

\paragraph{Step 3: Variance term via Bernstein concentration.}
We now bound the stochastic residual $\|u_h - \bar{u}\|$. Let $\mathcal{L}$ denote the set of $K$ independent sampled reasoning trajectories used during decoding. Under our local perturbation assumption (Assumption A3), the deviations of the hidden states are bounded by the local neighborhood radius $\rho$.
Applying Bernstein's inequality for bounded random vectors in a Hilbert space \citep{vershynin2018high}, the tail probabilities decay exponentially. For an error tolerance $\epsilon$ and an absolute constant $C > 0$, we have with probability at least $1-\delta$:
\[
\|u_h - \bar{u}\| \le |\mathcal{L}| \exp\left(-\frac{K\epsilon^2}{C}\right) \alpha(e^{\beta T}-1),
\]
where $\alpha$ is a scaling constant, $T$ is the sequence length, and $\beta \le \log \sigma_{\max}$ bounds the per-step Jacobian spectral norm.

\paragraph{Step 4: Substitution.}
Combining both terms yields the high-probability bound:
\[
\|u^*-u_h\| \le \frac{\Lambda}{\gamma} \inf_{u \in U_h} \|u^*-u\| + |\mathcal{L}| \exp\left(-\frac{K\epsilon^2}{C}\right) \alpha(e^{\beta T}-1).
\]
We now bound the condition ratio $\Lambda/\gamma$ via NTK decomposition.

\paragraph{Step 5: Decomposition of NTK Continuity Constant}
\label{appendix:continuity}
We decompose the bilinear form $a(\cdot,\cdot)$ into three components:
\[
a = a_0 + \delta_{\mathrm{pt}} + \delta_{\mathrm{mm}},
\]
where $a_0$ is the infinite-width baseline kernel, $\delta_{\mathrm{pt}}$ is the perturbation due to pre-training noise, and $\delta_{\mathrm{mm}}$ is the domain mismatch from fine-tuning. Consequently, the continuity constant satisfies $\Lambda = \Lambda_0 + \Delta_{\mathrm{pt}} + \Delta_{\mathrm{mm}}$.

\textbf{Bounding $\Delta_{\mathrm{pt}}$:} Following standard matrix concentration bounds for finite-width NTKs \citep{jacot2018ntk}, the pre-training deviation scales logarithmically with the network parameters. Let $P$ be the number of parameters, $L$ the prompt length, and $k_{\mathrm{pt}}$ a pre-training scaling constant; we have:
\[
\Delta_{\mathrm{pt}} \le \gamma k_{\mathrm{pt}} \log \mathcal{O}(P,L).
\]

\textbf{Bounding $\Delta_{\mathrm{mm}}$:} Using spectral generalization bounds under data distribution shift \citep{lee2020wide}, the mismatch penalty is governed by the task-specific signal strength $\mathrm{Signal}_k$, the empirical mismatch error $\epsilon_{\mathrm{mismatch}}$, and a scaling constant $k$:
\[
\Delta_{\mathrm{mm}} \le \gamma k \frac{\epsilon_{\mathrm{mismatch}}}{\mathrm{Signal}_k}.
\]

Substituting both inequalities into the ratio for $\Lambda/\gamma$, and normalizing $\Lambda_0/\gamma \approx 1$, we obtain:
\[
\frac{\Lambda}{\gamma} \le 1 + k_{\mathrm{pt}} \log \mathcal{O}(P,L) + k \frac{\epsilon_{\mathrm{mismatch}}}{\mathrm{Signal}_k}.
\]
This completes the proof.

\section{\model{} Derivation and Interpretation}
\label{appendix:connection}
\subsection{Preliminaries and Notation}
Let $\mathcal{K}\in\mathbb{R}^{r\times r}$ be the NTK Gram matrix formed on $r$ light semantic perturbations (see Assumptions A1-A3 in the main theory section). Denote its eigen decomposition by $\mathcal{K}=V\Lambda V^\top$ with
\[
\Lambda=\mathrm{diag}(\lambda_1,\ldots,\lambda_r),\qquad \lambda_1\ge \cdots\ge \lambda_r>0.
\]
Let $\lambda_{\max}:=\lambda_1$, $\lambda_{\min}:=\lambda_r$, $\kappa(\mathcal{K}):=\lambda_{\max}/\lambda_{\min}$, and $\det(\mathcal{K})=\prod_{i=1}^r\lambda_i$. Let $\Phi$ denote the NTK feature matrix whose columns span the hypothesis subspace $U_h$, so that $\mathcal{K}=\Phi^\top\Phi$, $\|\Phi\|_2=\sqrt{\lambda_{\max}}$, and $\sigma_{\min}(\Phi)=\sqrt{\lambda_{\min}}$. For the autoregressive decoder, let $J_t$ be the step-$t$ input--output Jacobian, and write $\sigma_{\max}:=\sup_t \|J_t\|_2$.

We will use the following two standard inequalities repeatedly:
\begin{align}
&Maclaurin / AM--GM on eigenvalues:
&&\Big(\prod_{i=1}^r \lambda_i\Big)^{\!\!1/r} \;\le\; \frac{1}{r}\sum_{i=1}^r \lambda_i \;=\; \frac{\mathrm{tr}(\mathcal{K})}{r}, \label{eq:maclaurin}\\
&Submultiplicativity:
&&\|AB\|_2 \;\le\; \|A\|_2\,\|B\|_2. \label{eq:submult}
\end{align}

\subsection{Representational Adequacy via $\det(\mathcal{K})$ with Explicit Constants}

\paragraph{Assumptions for this subsection.}
Beyond A1--A3, we assume a mild source condition and a spectral envelope:
\begin{itemize}
\item[\textbf{S1}] (\emph{Source condition}) Let $\mathcal{T}$ denote the infinite-dimensional NTK integral operator. We assume there exists a regularity exponent $s>0$ and a constant $R_s>0$ such that
$u^* \in \mathrm{Range}(\mathcal{T}^{s})$. Equivalently, the spectral coefficients satisfy:
$\sum_{i=1}^r \frac{\langle u^*,v_i\rangle^2}{\lambda_i^{2s}} \le R_s^2$.
This is standard in kernel approximation and encodes RKHS regularity.
\item[\textbf{S2}] (\emph{Spectral envelope}) Let $\overline{\lambda}$ and $\underline{\lambda}$ denote uniform upper and lower bounds on the kernel spectrum. We assume there exist constants $0<\underline{\lambda}\le \overline{\lambda}<\infty$ and a decay rate $\alpha>1$ such that
$\lambda_i \le \overline{\lambda}$ for all $i$, and the tail eigenvalue satisfies $\lambda_r \ge \underline{\lambda}\, r^{-\alpha}$.
(Polynomial decay is a common stylization; other envelopes can be treated similarly.)
\end{itemize}

\begin{lemma}[Best-approximation error under source condition]
\label{lem:tail}
Let $U_h=\mathrm{span}\{v_1,\ldots,v_r\}$. Under S1,
\[
\inf_{u\in U_h}\|u^*-u\|
\;=\;\|u^*-\Pi_{U_h}u^*\|
\;\le\; R_s\,\lambda_{r+1}^{\,s},
\]
where $\lambda_{r+1}$ denotes the next-eigenvalue of the infinite-dimensional kernel operator (or, equivalently, the empirical tail eigenvalue if more perturbations are added).
\end{lemma}
\begin{proof}
Write $u^*=\sum_{i\ge 1} c_i v_i$ with $c_i=\langle u^*,v_i\rangle$. By the source condition,
$\|u^*-\Pi_{U_h}u^*\|^2=\sum_{i>r} c_i^2 \le \sum_{i>r} \lambda_i^{2s} \cdot \frac{c_i^2}{\lambda_i^{2s}}
\le \lambda_{r+1}^{2s} \sum_{i>r} \frac{c_i^2}{\lambda_i^{2s}} \le \lambda_{r+1}^{2s} R_s^2.$
\end{proof}

To connect the representation error to the empirical NTK Gram matrix $\mathcal{K}$, we leverage the algebraic relationship between the smallest eigenvalue $\lambda_r$ and the determinant.

\begin{lemma}[Lower-bounding $\lambda_r$ by $\det(\mathcal{K})$]
\label{lem:lambda_min_det}
Suppose $\lambda_i \le \overline{\lambda}$ for all $i$ and $\lambda_r>0$. Then
\[
\lambda_r \;\ge\; \frac{\det(\mathcal{K})}{\overline{\lambda}^{\,r-1}}
\qquad\text{and}\qquad
\lambda_r^{\,s} \;\ge\; \frac{\det(\mathcal{K})^{\,s}}{\overline{\lambda}^{\,s(r-1)}}.
\]
\end{lemma}
\begin{proof}
Since $\det(\mathcal{K})=\prod_{i=1}^r \lambda_i \le \overline{\lambda}^{\,r-1}\lambda_r$, we obtain $\lambda_r \ge \det(\mathcal{K})/\overline{\lambda}^{\,r-1}$. Raising to power $s$ yields the second inequality.
\end{proof}

\begin{theorem}[Determinant-based adequacy bound with explicit constants]
\label{thm:det_bound_explicit}
Under A1-A3 and S1-S2,
Under A1--A3 and \textbf{S1}--\textbf{S2}, the approximation error is bounded by:
\[
\inf_{u\in U_h}\|u^*-u\|
\;\le\;
C_d \;\det(\mathcal{K})^{\,c_d},
\]
with explicit constants independent of the target sequence:
\[
c_d \;=\; \frac{s}{r}
\quad\text{and}\quad
C_d \;=\; R_s.
\]

Moreover, if the empirical spectrum satisfies $\lambda_r \ge \underline{\lambda}\, r^{-\alpha}$, one may choose
\[
c_d \;=\; \min\!\left\{\frac{s}{\,r-1\,},\; \frac{s}{\alpha}\cdot \frac{1}{\log\!\big(\tfrac{\overline{\lambda}^r}{\det(\mathcal{K})}\big)}\right\},
\]
which improves with slower decay (smaller $\alpha$).
\end{theorem}

\begin{proof}
By \Cref{lem:tail} and the fact that eigenvalues are monotonically decreasing ($\lambda_{r+1} \le \lambda_r$), we have:
\[
\inf_{u\in U_h}\|u^*-u\| \le R_s\,\lambda_r^{\,s}.
\]
Recall that the determinant of the empirical Gram matrix is the product of its eigenvalues, $\det(\mathcal{K}) = \prod_{i=1}^r \lambda_i$. Since $\lambda_r$ is the minimum eigenvalue of the rank-$r$ matrix, it follows strictly that $\lambda_r^r \le \det(\mathcal{K})$, which implies $\lambda_r \le \det(\mathcal{K})^{1/r}$.
Raising both sides to the power of $s$ yields $\lambda_r^s \le \det(\mathcal{K})^{s/r}$. 
Substituting this upper bound into the approximation error gives:
\[
\inf_{u\in U_h}\|u^*-u\| \;\le\; R_s \det(\mathcal{K})^{\,s/r}.
\]
Setting $C_d := R_s$ and $c_d := s/r$ completes the proof.
\end{proof}

In practice, tracking the direct determinant can cause numerical underflow in high-dimensional spaces. We use $\log\det(\mathcal{K})$ via the Cholesky decomposition as our empirical score, aggregating with $z$-normalization across components to avoid scale domination by any single dimension.

\subsection{Rollout Amplification via Jacobian Products (Exact Constants)}

\begin{theorem}[Amplification bound with exact constant]
\label{thm:jac}
Let $J_t$ be the step-$t$ Jacobian and $\sigma_{\max}:=\sup_t \|J_t\|_2$. Then
\[
\Big\|\prod_{t=1}^T J_t\Big\|_2 \;\le\; \prod_{t=1}^T \|J_t\|_2 \;\le\; \sigma_{\max}^{\,T}.
\]
Defining $\beta:=\log\sigma_{\max}$ gives $e^{\beta T}=\sigma_{\max}^T$, hence
\[
e^{\beta T} \;\le\; \sigma_{\max}^T,
\]
with equality if and only if $\|J_t\|_2=\sigma_{\max}$ for all $t$ and the top singular directions align across factors.
\end{theorem}

\begin{proof}
The first inequality is \eqref{eq:submult} applied iteratively. The second is by definition of $\sigma_{\max}$. Setting $\beta=\log\sigma_{\max}$ yields equality in the worst case. Alignment of top singular vectors is the tightness condition for submultiplicativity.
\end{proof}

\paragraph{Token-dependent refinement.}
If one defines $\sigma_t:=\|J_t\|_2$ and $\beta_{\mathrm{avg}}:=\tfrac{1}{T}\sum_{t=1}^T \log\sigma_t$, then
$\big\|\prod_{t=1}^T J_t\big\|_2 \le \exp\!\big(\sum_t \log\sigma_t\big) = e^{\beta_{\mathrm{avg}} T}$,
which is tighter but requires per-step measurements.

\subsection{Conditioning-Induced Variance with $\kappa(\mathcal{K})^2$ Scaling}

We now give an explicit projector-perturbation derivation showing the quadratic dependence on the condition number.

\paragraph{Setup.}
Let $P:=\Phi(\Phi^\top \Phi)^{\dagger}\Phi^\top$ be the orthogonal projector onto $U_h$; then the linearized output is $u_h=P u^*$. Consider a feature perturbation $\Delta\Phi$ induced by a prefix perturbation $\delta$ satisfying
\[
\|\Delta\Phi\|_2 \;\le\; L_\Phi \,\|\delta\| \quad\text{(A2/A3)}.
\]
Let the perturbed projector be $\widetilde{P}:=(\Phi+\Delta\Phi)\big((\Phi+\Delta\Phi)^\top(\Phi+\Delta\Phi)\big)^\dagger(\Phi+\Delta\Phi)^\top$ and define $\Delta P:=\widetilde{P}-P$.

\begin{lemma}[Projector perturbation bound]
\label{lem:proj}
There exists an absolute constant $C_\Pi>0$ such that
\[
\|\Delta P\|_2
\;\le\;
C_\Pi\;\frac{\|\Phi\|_2}{\sigma_{\min}(\Phi)^2}\;\|\Delta\Phi\|_2
\;=\;
C_\Pi\;\frac{\sqrt{\lambda_{\max}}}{\lambda_{\min}}\;\|\Delta\Phi\|_2
\;=\;
C_\Pi\;\kappa(\mathcal{K})\,\frac{\|\Delta\Phi\|_2}{\sqrt{\lambda_{\min}}}.
\]
\end{lemma}

\begin{proof}[Proof idea]
Use standard bounds for the perturbation of orthogonal projectors onto column spaces (e.g., Wedin's sin$\Theta$ theorem and Stewart--Sun, Matrix Perturbation Theory, Thm 3.6). One shows
\[
\|\Delta P\|_2 \;\le\; 2\,\|(\Phi^\top\Phi)^\dagger\|_2 \,\|\Phi^\top \Delta\Phi\|_2 \;+\; \mathcal{O}(\|\Delta\Phi\|_2^2).
\]
Since $\|(\Phi^\top\Phi)^\dagger\|_2=1/\lambda_{\min}$ and $\|\Phi^\top \Delta\Phi\|_2\le \|\Phi\|_2\,\|\Delta\Phi\|_2=\sqrt{\lambda_{\max}}\|\Delta\Phi\|_2$, the result follows for sufficiently small $\|\Delta\Phi\|_2$, absorbing lower-order terms into $C_\Pi$.
\end{proof}

\begin{theorem}[Variance amplification with explicit constant]
\label{thm:var}
Let $u_h(\Phi)=Pu^*$ and $u_h(\Phi+\Delta\Phi)=\widetilde{P}u^*$. Then
\[
\|u_h(\Phi+\Delta\Phi)-u_h(\Phi)\|
\;\le\;
C_\Pi\,\kappa(\mathcal{K})\,\frac{\|\Delta\Phi\|_2}{\sqrt{\lambda_{\min}}}\,\|u^*\|.
\]
If $\Delta\Phi$ is induced by a random prefix perturbation $\delta$ with $\|\Delta\Phi\|_2 \le L_\Phi\|\delta\|$ and $\mathbb{E}\|\delta\|^2=\sigma_\delta^2$, then
\[
\mathrm{Var}[u_h]
\;\le\;
\mathbb{E}\|u_h(\Phi+\Delta\Phi)-u_h(\Phi)\|^2
\;\le\;
c_v\;\kappa(\mathcal{K})^2\;\|\delta\|^2,
\]
with
\[
c_v \;=\; C_\Pi^2\,\frac{L_\Phi^2\,\|u^*\|^2}{\lambda_{\min}}.
\]
\end{theorem}

\begin{proof}
By \Cref{lem:proj},
$\|u_h(\Phi+\Delta\Phi)-u_h(\Phi)\| = \|\Delta P\,u^*\| \le \|\Delta P\|_2 \|u^*\| \le C_\Pi\,\kappa(\mathcal{K})\,\frac{\|\Delta\Phi\|_2}{\sqrt{\lambda_{\min}}}\,\|u^*\|.$
Square both sides and take expectation over $\delta$, using $\|\Delta\Phi\|_2 \le L_\Phi\|\delta\|$, to obtain the stated variance bound with the explicit constant $c_v$.
\end{proof}

\paragraph{Interpretation.}
The $\kappa(\mathcal{K})^2$ factor arises from two sources: (i) $\kappa(\mathcal{K})$ from the projector sensitivity (\Cref{lem:proj}), and (ii) $1/\lambda_{\min}$ from converting $\|\Delta P\|_2$ to a mean-squared bound after squaring and averaging, yielding an overall $\kappa^2$-scaling in the variance constant.

\subsection{Consolidation: Compact Surrogate Consistent with the Risk Decomposition}

Combining \Cref{thm:det_bound_explicit}, \Cref{thm:jac}, and \Cref{thm:var}, we obtain a computable surrogate aligned with the Hallucination Risk Bound:
\[
\text{Adequacy: } \det(\mathcal{K}) \quad\;\;
\text{Amplification: } \log\sigma_{\max} \quad\;\;
\text{Conditioning penalty: } -\log \kappa(\mathcal{K})^2.
\]
This motivates the score
\[
\boxed{\;
\text{\textbf{\model}}(u_h)
\;=\;
\det(\mathcal{K})
\;+\;
\log \sigma_{\max}
\;-\;
\log \kappa(\mathcal{K})^2
\;}
\]
with the following explicit, implementation-ready notes:
\begin{itemize}
\item Use $\log\det(\mathcal{K})$ via Cholesky for stability; replace $\det$ in the score with $\log\det$ if desired (monotone equivalent).
\item Estimate $\sigma_{\max}$ either as $\sup_t \|J_t\|_2$ or its tighter average form $\beta_{\mathrm{avg}}=\tfrac{1}{T}\sum_t \log\|J_t\|_2$ (then use $\beta_{\mathrm{avg}}$ in place of $\log\sigma_{\max}$).
\item $z$-normalize each component across a validation set before summation to avoid scale dominance; optionally fit task-specific weights if permitted.
\end{itemize}

\section{Experiment}
\subsection{Setup}
\label{appendix:exp}
\paragraph{Implementation Framework.}
All experiments use \texttt{PyTorch} and \texttt{HuggingFace Transformers} with a fixed random seed for reproducibility.
Unless otherwise noted, computations run in mixed precision (fp16).
Hardware details (A100/H200) are reported once in the main setup section.

\paragraph{Generation Configuration.}
For \emph{default evaluation of detectors}, we use nucleus sampling with
\texttt{temperature} $=0.5$, \texttt{top-p} $=0.95$, and \texttt{top-k} $=10$,
decoding $K{=}10$ candidate responses per input (unless otherwise specified).
These decoding trajectories also operationalize semantic perturbations as natural variations within the model's local predictive distribution, thereby instantiating a semantically proximate neighborhood around the primary response and capturing the local geometry of the reasoning manifold required for NTK construction.
For \emph{score-guided test-time inference} (\Cref{sec:testtime}), we use beam search (beam size $=10$) and score candidate trajectories at each step with the chosen detector.
For stability analysis, \model{} extracts sentence representations from the final token at the middle transformer layer ($L/2$), which empirically preserves semantics relevant to truthfulness.

\paragraph{NTK-Based Score Computation.}
For each set of generations, we form a task-specific NTK feature matrix and compute the semantic stability score from its eigenspectrum.
We add a small ridge $\alpha=10^{-3}$ for numerical stability and compute singular values via SVD.

\paragraph{Perturbation Regularization.}
To prevent pathological activations that amplify instability, \model{} clips hidden features using an adaptive scheme.
We maintain a memory bank of $N{=}3000$ token embeddings and set thresholds at the top and bottom $0.2\%$ percentiles of neuron activations; out-of-range values are truncated to attenuate overconfident hallucinations.

\paragraph{Optimization.}
Backbone language models are \emph{not} fine-tuned.
We train only \model{}'s lightweight projection layers using AdamW with learning rate selected from $\{1\times10^{-5},\,5\times10^{-5},\,1\times10^{-4}\}$ and weight decay from $\{0.0,\,0.01\}$.
The best setting is chosen on a held-out validation split.

\paragraph{Implementation Details.}
For score-guided inference we apply beam search with beam size $10$, rescoring candidates stepwise with different hallucination detectors.

\paragraph{Ablation Setup.}
All ablations reuse the main paper's splits, prompts, and decoding; we vary only \model{} internals and explicitly control the hallucination \emph{base rate}.
On the \emph{generation} side, we modulate prevalence by adjusting temperature/top-$p$ and beam size; to stress the two families, we increase the prefix perturbation budget $\rho$ and rollout horizon $T$ to amplify reasoning drift, and (when applicable) toggle retrieval masking to induce data-driven errors.
On the \emph{detection} side, AUROC/AUPRC are threshold-free; when a fixed operating point is needed, we set a decision threshold $\tau$ on the validation set by (i) matching a target predicted-positive rate $\pi_{\text{target}}$ via score quantiles or (ii) fixing a desired FPR (e.g., $1\%,5\%,10\%$); a cost-sensitive Bayes rule $\displaystyle \tau=\frac{c_{\mathrm{FN}}}{c_{\mathrm{FP}}+c_{\mathrm{FN}}}\cdot\frac{1-\pi}{\pi}$ is optional when misclassification costs are specified.
Unless noted, we toggle one factor at a time and sweep $\rho\in\{0.75,1.0,1.5\}$, $T\in\{12,16,24\}$, and the number of semantic probes $m\in\{2,4,8\}$; no additional training is performed beyond optional temperature/z-score calibration on the training split.
We report mean$\pm$std over 5 seeds.

\subsection{Ablation Study on $-\log \kappa^2$}
To empirically validate the necessity of the stability term $-\log \kappa^2$, we performed a controlled ablation on MATH-500. We systematized the reasoning drift ($d$) by progressively increasing the perturbation budget $\rho$ and rollout horizon $T$.
As shown in~\Cref{fig:abaltion_third}, the absence of this term leads to severe instability. While the ablated model (orange dashed line) performs competitively in low-drift regimes ($d < 0.15$), it exhibits significant performance volatility as the reasoning task becomes more complex. In contrast, the full \model{} score (green solid line) effectively penalizes these ill-conditioned regimes, maintaining a smooth and robust detection profile. This confirms that $-\log \kappa^2$ functions as an essential spectral regularizer, preventing the score from becoming unreliable under high-entropy inference states.

\begin{figure}[ht!]
    \centering
    \includegraphics[width=0.8\linewidth]{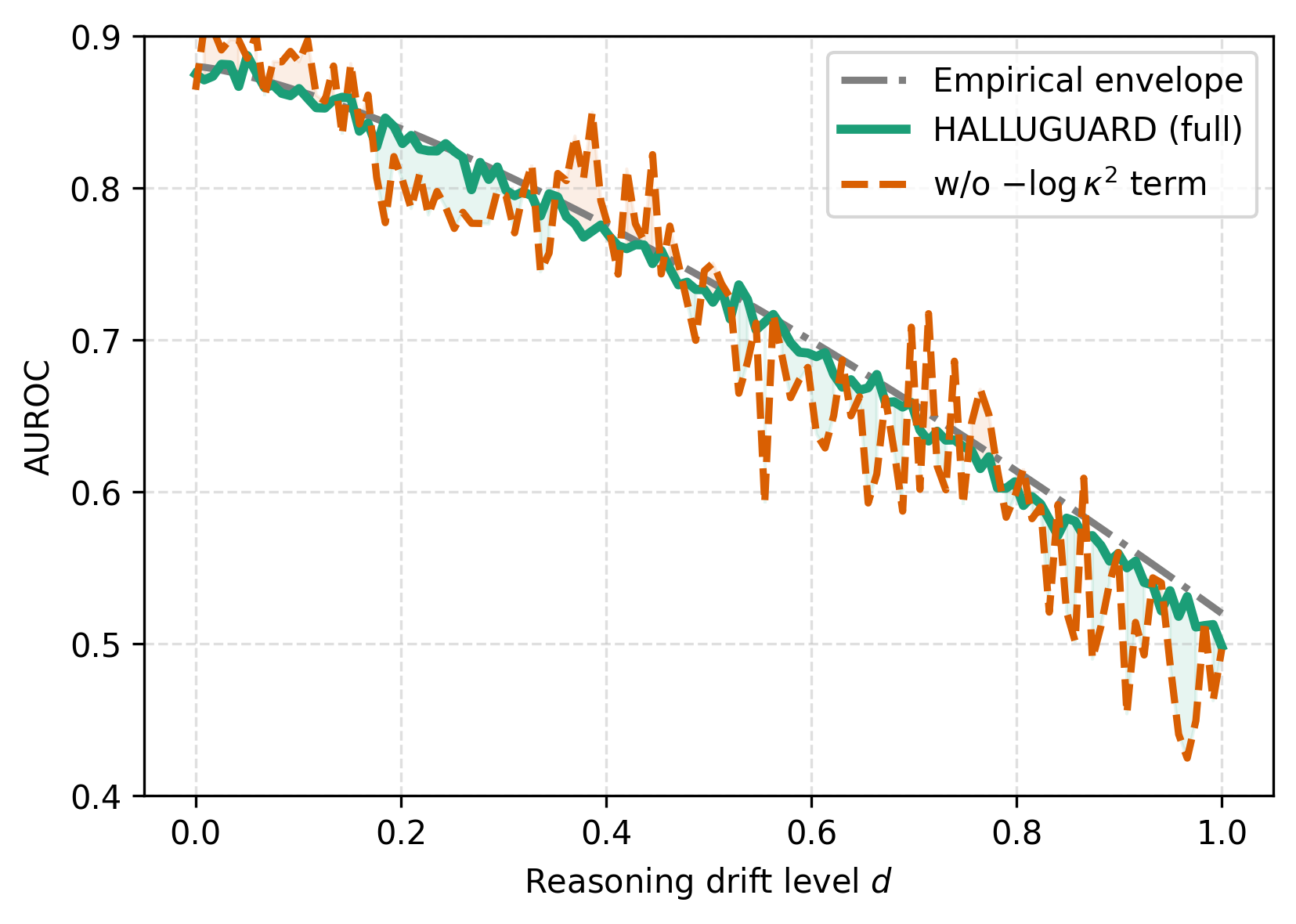}
    \caption{Ablation study of the stability term ($-\log \kappa^2$) on MATH500.}
    \label{fig:abaltion_third}
\end{figure}

\begin{table}[h]
\centering
\caption{Ablation on stability term $-\log \kappa^2$ (MATH500).}
\begin{tabular}{lcc}
\toprule
Method & Pearson $R$ & MSE \\
\midrule
HalluGuard & 0.985 & 0.0192 \\
w/o $-\log \kappa^2$ & 0.8904 & 0.0381 \\
\bottomrule
\end{tabular}
\label{tab: aba}
\end{table}

The error in \cref{tab: aba} nearly doubles without the stability term, confirming that spectral conditioning 
is essential for stable reasoning-risk quantification.

\subsection{Ablation Study on Semantic Encoder $\Phi$}

To examine sensitivity to the semantic encoder $\Phi$, we replace the default representation with widely adopted alternatives, including BERT, SimCSE, and E5. We evaluate across multiple backbone models and benchmarks.

Table~\ref{tab:encoder_ablation} reports AUROC and AUPRC on RAGTruth, GSM8K, and TruthfulQA. Across all settings, \model{} consistently outperforms encoder-substituted variants. For example, on QwQ-32B (RAGTruth), replacing the default encoder with BERT reduces AUROC from 84.59 to 81.44.

These results indicate that the performance gain does not stem from surface semantic similarity of final outputs. Instead, the method captures geometric structure of reasoning trajectories, which external encoders cannot fully preserve.

\begin{table*}[h!]
\centering
\caption{Encoder ablation across backbones and benchmarks (AUROC / AUPRC).}
\label{tab:encoder_ablation}
\begin{adjustbox}{width=\textwidth}
\begin{tabular}{llcccccc}
\toprule
Backbone & Method 
& \multicolumn{2}{c}{RAGTruth} 
& \multicolumn{2}{c}{GSM8K} 
& \multicolumn{2}{c}{TruthfulQA} \\
\cmidrule(lr){3-4} \cmidrule(lr){5-6} \cmidrule(lr){7-8}
& & AUROC & AUPRC & AUROC & AUPRC & AUROC & AUPRC \\
\midrule

\multirow{4}{*}{GPT-2}
& HalluGuard          & \textbf{75.51} & \textbf{73.40} & \textbf{72.04} & \textbf{69.88} & \textbf{72.10} & \textbf{68.76} \\
& +BERT               & 72.48 & 70.12 & 67.31 & 64.90 & 68.02 & 65.01 \\
& +SimCSE             & 73.21 & 71.05 & 68.44 & 66.02 & 69.14 & 66.27 \\
& +E5                 & 74.02 & 71.66 & 69.12 & 66.80 & 70.03 & 67.10 \\

\midrule
\multirow{4}{*}{OPT-6.7B}
& HalluGuard          & \textbf{80.13} & \textbf{76.77} & \textbf{72.57} & \textbf{70.31} & \textbf{69.59} & \textbf{68.36} \\
& +BERT               & 77.44 & 74.20 & 67.95 & 65.48 & 66.12 & 64.80 \\
& +SimCSE             & 78.11 & 74.83 & 69.01 & 66.40 & 67.08 & 65.72 \\
& +E5                 & 78.66 & 75.31 & 70.04 & 67.25 & 67.80 & 66.41 \\

\midrule
\multirow{4}{*}{Mistral-7B}
& HalluGuard          & \textbf{82.31} & \textbf{80.79} & \textbf{80.62} & \textbf{77.30} & \textbf{77.05} & \textbf{73.79} \\
& +BERT               & 79.02 & 76.91 & 75.51 & 72.08 & 73.14 & 69.52 \\
& +SimCSE             & 79.88 & 77.66 & 76.40 & 73.01 & 74.08 & 70.40 \\
& +E5                 & 80.41 & 78.20 & 77.12 & 73.74 & 74.66 & 71.05 \\

\midrule
\multirow{4}{*}{QwQ-32B}
& HalluGuard          & \textbf{84.59} & \textbf{81.15} & \textbf{75.81} & \textbf{74.68} & \textbf{74.26} & \textbf{72.76} \\
& +BERT               & 81.44 & 78.03 & 70.92 & 68.90 & 70.35 & 68.01 \\
& +SimCSE             & 82.10 & 78.66 & 72.10 & 69.82 & 71.20 & 68.70 \\
& +E5                 & 82.66 & 79.12 & 73.05 & 70.44 & 72.02 & 69.31 \\

\midrule
\multirow{4}{*}{LLaMA2-13B}
& HalluGuard          & \textbf{77.51} & \textbf{75.30} & \textbf{79.01} & \textbf{76.73} & \textbf{78.50} & \textbf{77.56} \\
& +BERT               & 74.26 & 72.04 & 73.12 & 70.60 & 74.41 & 72.88 \\
& +SimCSE             & 75.11 & 72.83 & 74.20 & 71.51 & 75.36 & 73.54 \\
& +E5                 & 75.78 & 73.44 & 75.14 & 72.32 & 76.10 & 74.22 \\

\bottomrule
\end{tabular}
\end{adjustbox}
\end{table*}

\subsection{Computational Efficiency Analysis}
To assess practical deployment feasibility, we measured inference latency on an NVIDIA A100/H200 GPU. Our setup utilizes batched parallel sampling to generate $K=10$ trajectories, ensuring sub-linear scaling of the computational cost. The core \model{} operations-specifically feature clipping and computing the NTK score via the Gram matrix-add minimal latency, requiring less than 1 ms of post-processing time per query.

\begin{figure}[ht!]
    \centering
    \includegraphics[width=0.8\linewidth]{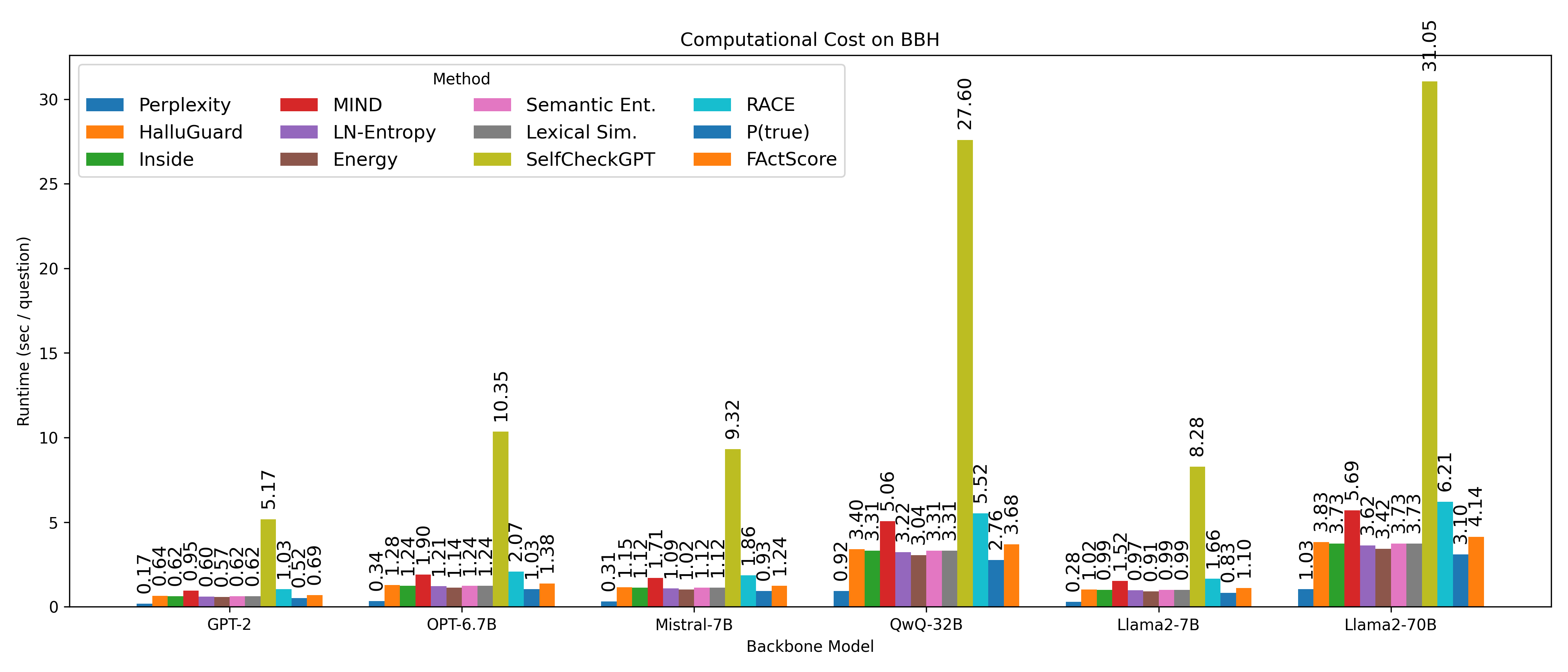}
    \caption{Per-Question Inference Time (Seconds) on BBH Across Hallucination Detection Methods.}
    \label{fig:bbh}
\end{figure}

\begin{figure}[ht!]
    \centering
    \includegraphics[width=0.8\linewidth]{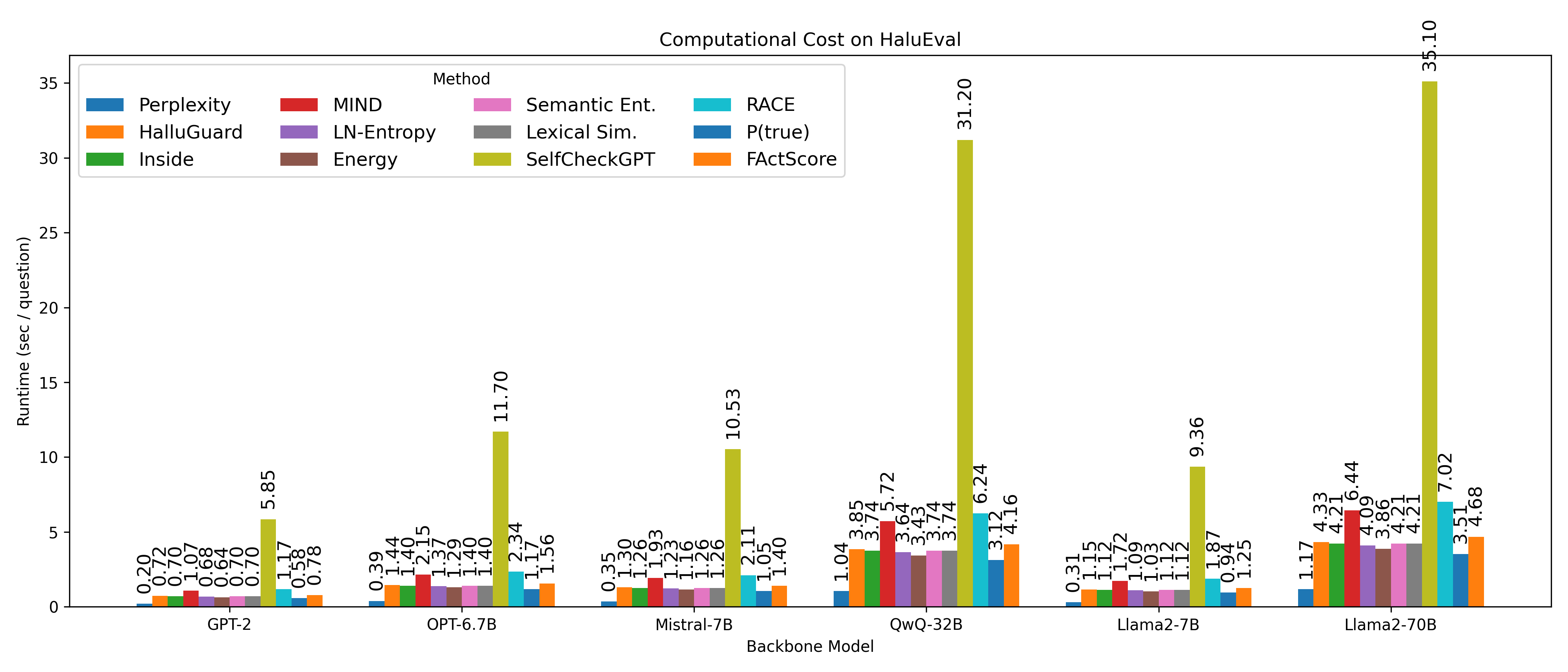}
    \caption{Per-Question Inference Time (Seconds) on HaluEval Across Hallucination Detection Methods.}
    \label{fig:halueval}
\end{figure}

\begin{figure}[ht!]
    \centering
    \includegraphics[width=0.8\linewidth]{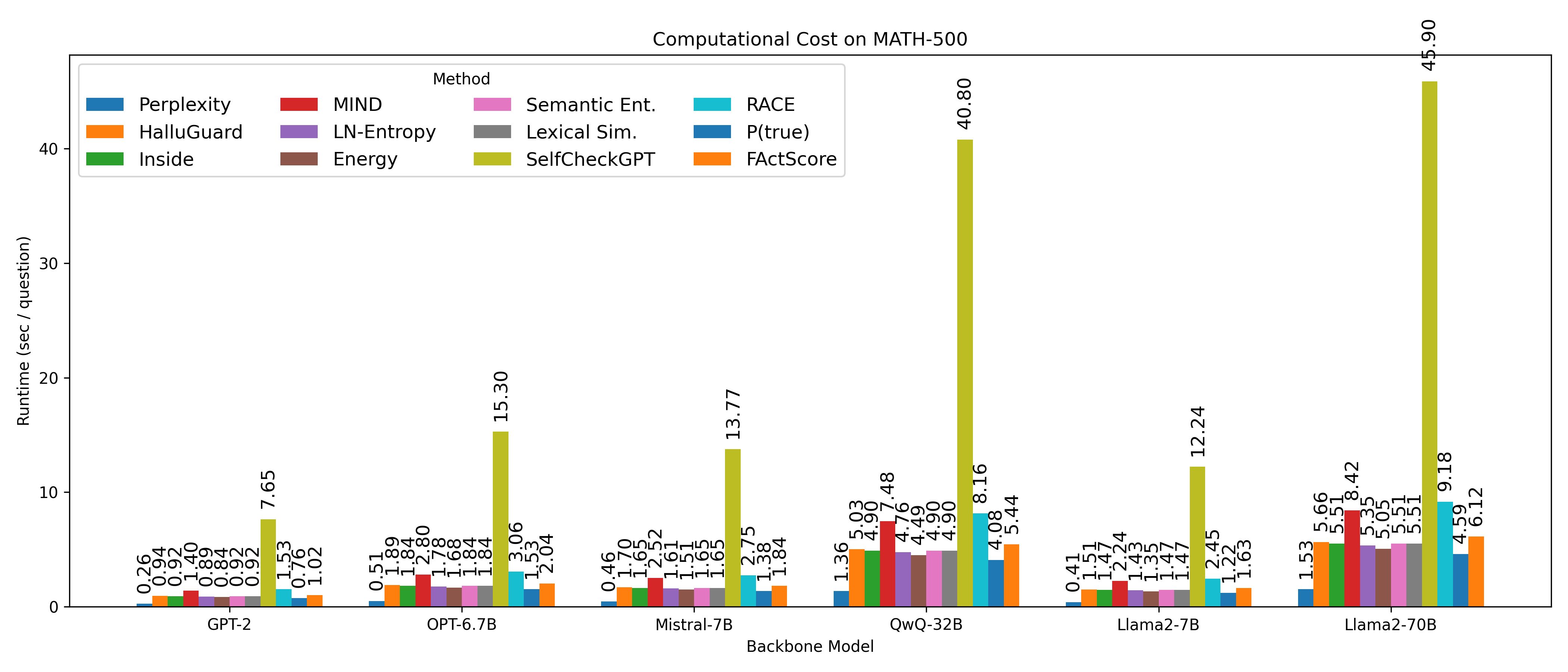}
    \caption{Per-Question Inference Time (Seconds) on Math500 Across Hallucination Detection Methods.}
    \label{fig:math500}
\end{figure}

\begin{figure}[ht!]
    \centering
    \includegraphics[width=0.8\linewidth]{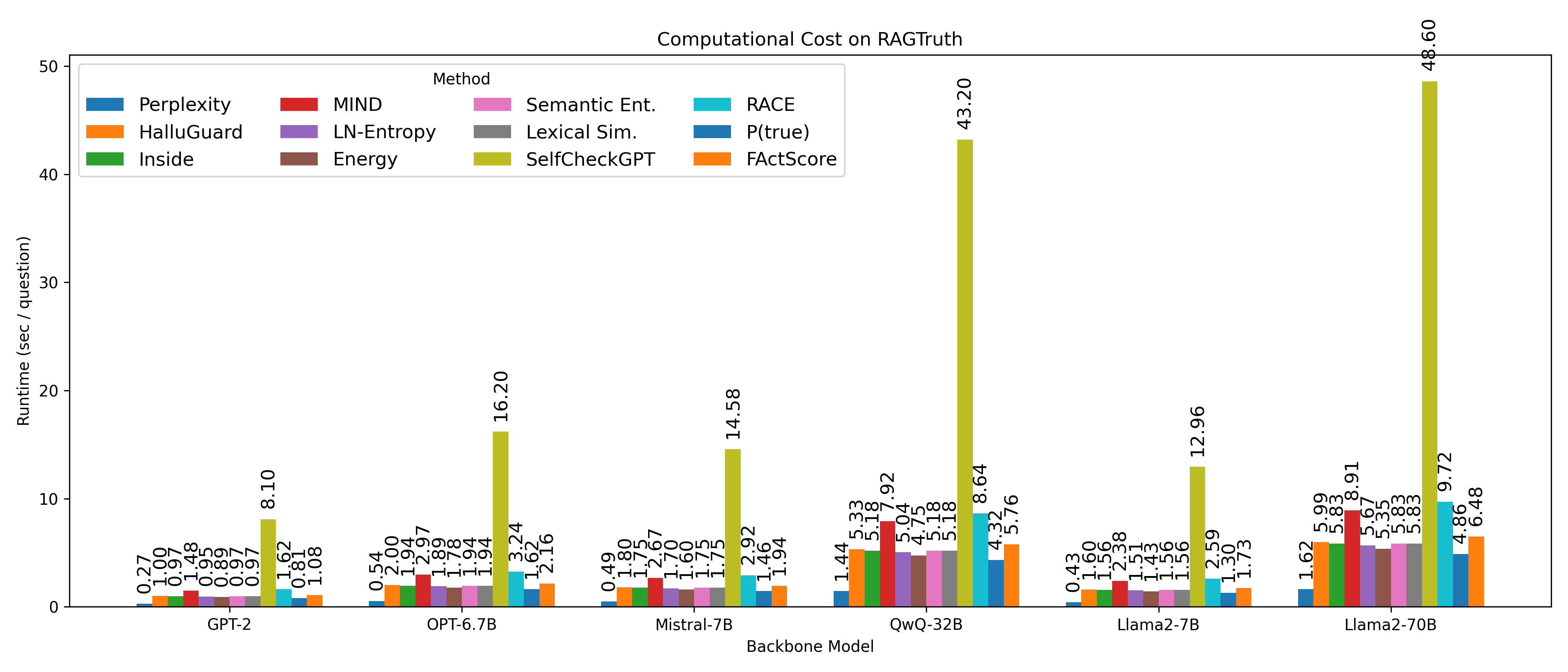}
    \caption{Per-Question Inference Time (Seconds) on RAGTruth Across Hallucination Detection Methods.}
    \label{fig:ragtruth}
\end{figure}

\begin{figure}[ht!]
    \centering
    \includegraphics[width=0.8\linewidth]{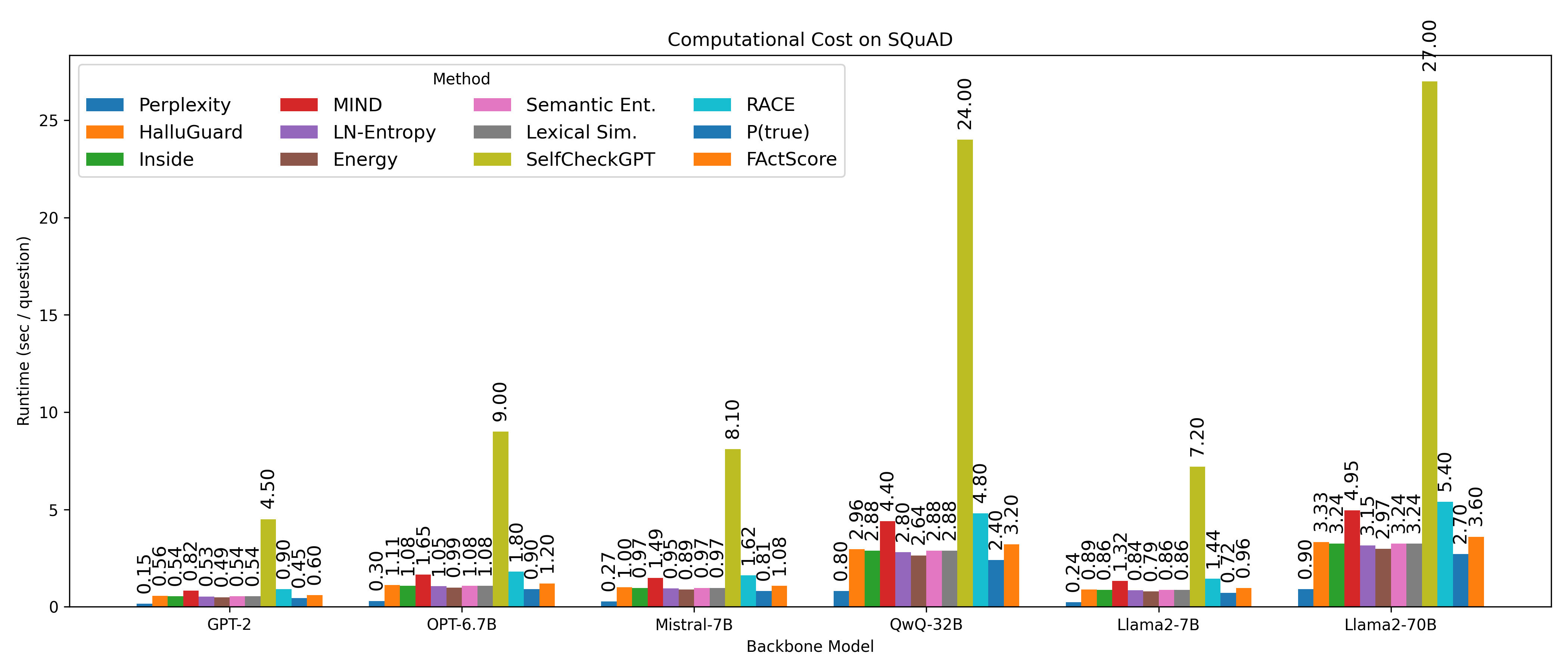}
    \caption{Per-Question Inference Time (Seconds) on SQuaD Across Hallucination Detection Methods.}
    \label{fig:squad}
\end{figure}

\begin{figure}[ht!]
    \centering
    \includegraphics[width=0.8\linewidth]{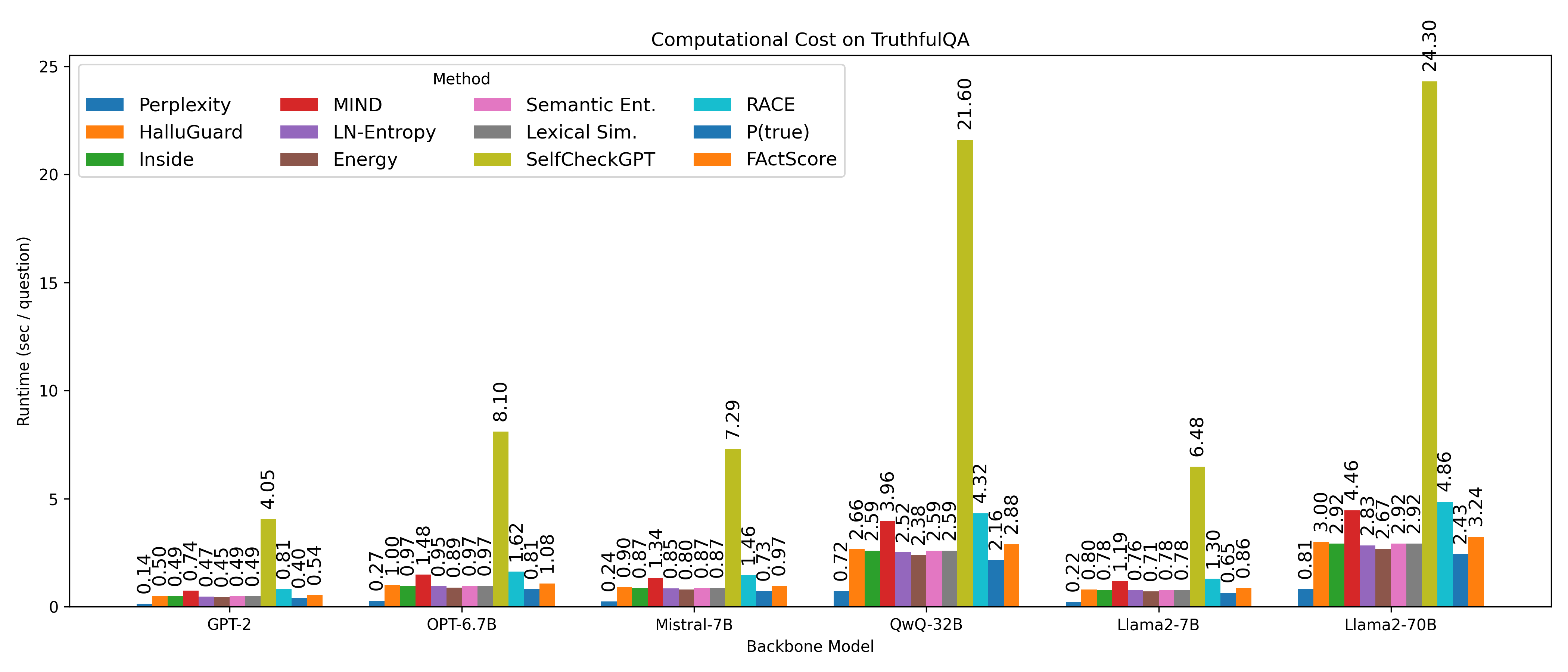}
    \caption{Per-Question Inference Time (Seconds) on TruthfulQA Across Hallucination Detection Methods.}
    \label{fig:truthfulqa}
\end{figure}

\subsection{Detection Performance Analysis}
Across all five model families and three benchmark regimes, \model{} consistently achieves state-of-the-art detection performance, particularly in the safety-critical low-FPR regions as shown in~\Cref{tab:qq1}.

\begin{table*}[ht!]
\centering
\scriptsize
\caption{Performance comparison on representative benchmarks: data-centric (\texttt{RAGTruth}), reasoning-oriented (\texttt{BBH}), and instruction-following (\texttt{TruthfulQA}). }
\resizebox{\textwidth}{!}{
\begin{tabularx}{0.96\textwidth}{@{}l@{\hspace{1em}}l@{\hspace{1em}}
c@{\hspace{0.25em}}c@{\hspace{0.25em}}c@{\hspace{0.8em}}%
c@{\hspace{0.25em}}c@{\hspace{0.25em}}c@{\hspace{0.8em}}%
c@{\hspace{0.25em}}c@{\hspace{0.25em}}c@{\hspace{0.8em}}%
c@{\hspace{0.25em}}c@{\hspace{0.25em}}c@{\hspace{0.8em}}%
c@{\hspace{0.25em}}c@{\hspace{0.25em}}c@{}}%
\toprule
 & & \multicolumn{3}{c}{\textbf{GPT2}} 
 & \multicolumn{3}{c}{\textbf{OPT-6.7B}} 
 & \multicolumn{3}{c}{\textbf{Mistral-7B}} 
 & \multicolumn{3}{c}{\textbf{QwQ-32B}}
 & \multicolumn{3}{c}{\textbf{LLaMA2-13B}} \\
\cmidrule(lr){3-5} \cmidrule(lr){6-8} \cmidrule(lr){9-11} \cmidrule(lr){12-14} \cmidrule(lr){15-17}
& & \rotatebox{55}{\tiny{F1}} 
  & \rotatebox{55}{\tiny{TPR@10\%}} 
  & \rotatebox{55}{\tiny{TPR@5\%}}
  & \rotatebox{55}{\tiny{F1}} 
  & \rotatebox{55}{\tiny{TPR@10\%}} 
  & \rotatebox{55}{\tiny{TPR@5\%}}
  & \rotatebox{55}{\tiny{F1}} 
  & \rotatebox{55}{\tiny{TPR@10\%}} 
  & \rotatebox{55}{\tiny{TPR@5\%}}
  & \rotatebox{55}{\tiny{F1}} 
  & \rotatebox{55}{\tiny{TPR@10\%}} 
  & \rotatebox{55}{\tiny{TPR@5\%}}
  & \rotatebox{55}{\tiny{F1}} 
  & \rotatebox{55}{\tiny{TPR@10\%}} 
  & \rotatebox{55}{\tiny{TPR@5\%}} \\
\midrule

\multirow{12}{*}{\rotatebox{90}{\textbf{RAGTruth}}}
 & \model        
 & \textbf{71.22} & \textbf{64.86} & \textbf{51.41}
 & \textbf{77.03} & \textbf{73.52} & \textbf{59.12}
 & \textbf{75.19} & \textbf{69.44} & \textbf{59.21}
 & \textbf{81.91 }& \textbf{74.13} & \textbf{63.52}
 & \textbf{74.66} & \underline{68.91} & \textbf{57.42} \\
 & Inside       
 & \underline{66.12} & \underline{59.72} & \underline{48.31}
 & \underline{72.91} & \underline{70.25} & \underline{60.37}
 & 70.45 & \underline{68.12} & 52.41
 & \underline{79.03} & \underline{74.66} & \underline{61.09}
 & \underline{73.08} & \textbf{70.11}& \underline{55.26} \\
 & MIND          
 & 58.33 & 54.11 & 38.72
 & 62.55 & 57.81 & 47.65
 & \underline{71.91} & 66.74 & \underline{54.39}
 & 64.02 & 59.12 & 45.63
 & 68.55 & 63.50 & 48.78 \\
 & Perplexity    
 & 55.42 & 51.20 & 40.51
 & 63.72 & 60.13 & 49.14
 & 69.74 & 66.51 & 52.18
 & 70.42 & 65.41 & 55.32
 & 60.18 & 57.01 & 44.75 \\
 & LN-Entropy    
 & 62.17 & 57.52 & 46.44
 & 58.33 & 52.99 & 43.28
 & 65.30 & 61.27 & 49.92
 & 67.15 & 62.42 & 51.33
 & 63.28 & 59.07 & 46.14 \\
 & Energy        
 & 59.71 & 56.23 & 44.81
 & 60.44 & 57.18 & 45.03
 & 63.54 & 59.42 & 48.62
 & 72.09 & 68.15 & 58.42
 & 66.10 & 61.33 & 49.41 \\
 & Semantic Ent. 
 & 57.28 & 53.42 & 41.92
 & 69.61 & 64.81 & 52.01
 & 67.10 & 62.44 & 50.66
 & 66.12 & 62.15 & 49.31
 & 64.55 & 60.18 & 47.75 \\
 & Lexical Sim.  
 & 61.41 & 57.09 & 45.03
 & 65.81 & 61.44 & 49.51
 & 62.50 & 59.12 & 50.92
 & 70.91 & 67.53 & 55.21
 & 66.29 & 59.88 & 51.03 \\
 & \tiny SelfCheckGPT  
 & 56.22 & 52.84 & 40.63
 & 60.79 & 55.68 & 45.72
 & 63.12 & 59.47 & 48.33
 & 66.54 & 62.92 & 51.41
 & 68.21 & 65.12 & 53.60 \\
 & RACE          
 & 60.12 & 56.50 & 44.90
 & 64.12 & 59.77 & 49.22
 & 65.44 & 61.55 & 52.73
 & 69.61 & 66.31 & 53.92
 & 62.55 & 59.42 & 45.66 \\
 & P(true)       
 & 58.91 & 55.47 & 42.13
 & 67.44 & 63.20 & 51.43
 & 71.22 & 66.91 & 54.10
 & 63.44 & 60.33 & 49.27
 & 70.18 & 65.77 & 52.78 \\
 & FActScore     
 & 62.10 & 58.21 & 46.33
 & 59.22 & 54.14 & 44.32
 & 63.87 & 60.77 & 47.98
 & 68.33 & 64.02 & 53.41
 & 65.92 & 61.37 & 49.84 \\
\midrule

\multirow{12}{*}{\rotatebox{90}{\textbf{BBH}}}
 & \model        
 & \textbf{68.33} & \textbf{64.11} & \textbf{56.42}
 & \textbf{74.91} & \textbf{69.14} & \textbf{62.10}
 & \textbf{73.22} &\textbf{69.88} & \textbf{57.21}
 & \textbf{78.55} & \textbf{69.91 }& \textbf{61.45}
 & \textbf{71.10} & \textbf{68.25} & \textbf{59.92} \\
 & Inside        
 & \underline{65.41} & \underline{61.22} & \underline{52.83}
 & \underline{71.02} & \underline{67.10} & \underline{60.21}
 & \underline{68.17} & \underline{64.75} & 53.92
 & \underline{79.17} & \underline{72.33} & \underline{64.22}
 & 67.10 & 63.52 & \underline{55.91} \\
 & MIND          
 & 54.12 & 50.22 & 40.11
 & 57.21 & 53.44 & 41.52
 & 63.92 & 59.88 & 47.01
 & 61.55 & 57.14 & 48.83
 & 65.11 & 60.22 & 49.52 \\
 & Perplexity    
 & 52.91 & 49.33 & 40.44
 & 61.88 & 58.12 & 49.22
 & 62.91 & 59.42 & 50.11
 & 59.91 & 55.72 & 49.03
 & 60.88 & 57.41 & 48.62 \\
 & LN-Entropy    
 & 59.12 & 55.44 & 44.92
 & 54.61 & 51.75 & 43.18
 & 66.44 & 63.21 & \underline{54.09}
 & 62.75 & 59.12 & 47.52
 & \underline{68.20} & \underline{64.88} & 55.41 \\
 & Energy        
 & 53.94 & 51.22 & 45.03
 & 56.12 & 52.14 & 44.61
 & 64.55 & 60.11 & 49.99
 & 68.21 & 65.12 & 52.84
 & 66.41 & 62.77 & 50.22 \\
 & Semantic Ent. 
 & 57.41 & 54.32 & 47.21
 & 61.22 & 58.42 & 49.74
 & 63.21 & 59.10 & 48.62
 & 63.55 & 60.24 & 48.88
 & 64.91 & 61.44 & 50.72 \\
 & Lexical Sim.  
 & 50.41 & 46.77 & 38.92
 & 60.71 & 57.11 & 45.55
 & 59.42 & 56.88 & 48.91
 & 70.33 & 67.10 & 55.32
 & 58.33 & 55.42 & 47.41 \\
 & \tiny SelfCheckGPT  
 & 55.21 & 52.14 & 43.92
 & 58.10 & 55.78 & 46.22
 & 62.82 & 59.90 & 50.44
 & 65.22 & 62.44 & 54.21
 & 63.44 & 60.77 & 52.33 \\
 & RACE          
 & 56.14 & 53.72 & 43.88
 & 63.11 & 59.71 & 52.81
 & 65.77 & 62.55 & 50.72
 & 58.88 & 55.14 & 46.18
 & 66.10 & 62.41 & 49.81 \\
 & P(true)       
 & 54.31 & 52.22 & 44.10
 & 58.22 & 56.10 & 48.52
 & 56.91 & 53.55 & 43.92
 & 61.40 & 58.21 & 46.77
 & 57.33 & 54.88 & 45.91 \\
 & FActScore     
 & 56.20 & 52.42 & 41.77
 & 55.44 & 52.12 & 41.14
 & 61.62 & 58.22 & 51.33
 & 59.33 & 56.42 & 49.14
 & 63.44 & 60.22 & 52.44 \\
\midrule

\multirow{12}{*}{\rotatebox{90}{\textbf{TruthfulQA}}}
 & \model        
 & \textbf{75.11} & \textbf{71.20} & \textbf{63.21}
 & \textbf{67.44} & \textbf{64.55} & \textbf{58.12}
 & \textbf{78.92} & \textbf{74.22} & \textbf{65.33}
 & \textbf{76.44} & \textbf{72.01} & \textbf{59.92}
 & \textbf{75.33} & \textbf{69.11} & \textbf{63.08} \\
 & Inside        
 & \underline{71.10} & \underline{68.55} & \underline{60.77}
 & 61.77 & 59.44 & 50.10
 & 63.88 & 61.33 & 53.41
 & \underline{69.22} & \underline{65.10} & \underline{55.14}
 & 62.14 & 59.94 & 52.80 \\
 & MIND          
 & 57.44 & 54.91 & 45.33
 & 59.92 & 56.88 & 48.33
 & 58.72 & 56.14 & 47.21
 & 61.21 & 58.88 & 52.02
 & 60.44 & 58.20 & 49.03 \\
 & Perplexity    
 & 49.52 & 46.71 & 38.84
 & 54.12 & 51.74 & 43.90
 & 59.72 & 57.55 & 46.88
 & 54.44 & 51.72 & 42.55
 & 60.33 & 57.21 & 47.41 \\
 & LN-Entropy    
 & 57.11 & 54.88 & 42.98
 & 55.33 & 52.41 & 45.91
 & 59.66 & 56.22 & 43.10
 & 60.44 & 58.02 & 46.22
 & 61.41 & 57.17 & 43.88 \\
 & Energy        
 & 54.11 & 52.17 & 38.91
 & 53.44 & 51.14 & 36.88
 & 58.21 & 54.77 & 49.92
 & 63.02 & 60.44 & 51.33
 & 58.41 & 55.33 & 50.42 \\
 & Semantic Ent. 
 & 60.08 & 56.44 & 44.15
 & 50.14 & 47.33 & 35.92
 & 53.74 & 52.11 & 37.02
 & 65.33 & 63.20 & 50.77
 & 55.02 & 53.11 & 38.44 \\
 & Lexical Sim.  
 & 51.22 & 49.20 & 39.03
 & 58.72 & 54.71 & 48.77
 & 65.71 & 63.50 & 53.10
 & 54.77 & 51.44 & 45.88
 & 66.41 & 64.14 & 54.88 \\
 & \tiny SelfCheckGPT  
 & 55.72 & 53.44 & 42.78
 & 58.33 & 55.72 & 47.14
 & 60.88 & 57.44 & 43.91
 & 55.42 & 54.44 & 40.77
 & 61.72 & 59.51 & 44.10 \\
 & RACE          
 & 52.22 & 49.88 & 41.44
 & \underline{63.14} & \underline{66.88} & \underline{54.05}
 & \underline{70.55} & \underline{67.11} & \underline{59.77}
 & 55.44 & 52.11 & 45.33
 & \underline{71.33} & \underline{68.22} & \underline{60.02} \\
 & P(true)       
 & 55.54 & 52.11 & 38.82
 & 55.72 & 52.33 & 39.22
 & 57.41 & 53.10 & 41.22
 & 56.88 & 54.77 & 45.55
 & 57.12 & 53.33 & 41.88 \\
 & FActScore     
 & 52.91 & 50.14 & 40.44
 & 54.11 & 50.22 & 41.33
 & 52.88 & 49.91 & 42.55
 & 61.55 & 59.22 & 44.72
 & 53.41 & 50.71 & 43.10 \\
\bottomrule
\vspace{-3em}
\end{tabularx}
}
\label{tab:qq1}
\end{table*}

We additionally expanded our evaluation to include SAPLMA, LLM-Check, and ITI. As shown in ~\Cref{tab:halluguard_low_fpr}, \model{} delivers the strongest performance not only on AUROC/AUPRC but also on deployment-critical, low-FPR operating points, including F1 and TPR at 5\% and 10\% FPR. Across all three benchmarks (RAGTruth, GSM8K, HaluEval) and all backbones (GPT-2 through QwQ-32B and LLaMA2-13B), \model{} consistently achieves the highest F1 and the highest or near-highest TPR under fixed low-FPR constraints. In contrast, SAPLMA and LLM-Check exhibit noticeably lower recall in the stringent 5\% FPR regime. These results demonstrate that \model{} is better aligned with maintaining high detection sensitivity under tight false-positive budgets, a requirement that is central to reliable hallucination detection in real-world systems.

\begin{table*}[ht!]
\centering
\scriptsize
\caption{Comparison with SAPLMA, LLM-Check and ITI across benchmarks and backbones. }
\resizebox{\textwidth}{!}{
\begin{tabular}{ll
rrrrr
rrrrr
rrrrr
rrrrr
rrrrr}
\toprule
Benchmark & Method
& \multicolumn{5}{c}{GPT2}
& \multicolumn{5}{c}{OPT-6.7B}
& \multicolumn{5}{c}{Mistral-7B}
& \multicolumn{5}{c}{QwQ-32B}
& \multicolumn{5}{c}{LLaMA2-13B} \\
\cmidrule(lr){3-7} \cmidrule(lr){8-12} \cmidrule(lr){13-17} \cmidrule(lr){18-22} \cmidrule(lr){23-27}
& &
AUROC & AUPRC & F1 & TPR@10\% & TPR@5\% 
& AUROC & AUPRC & F1 & TPR@10\% & TPR@5\%
& AUROC & AUPRC & F1 & TPR@10\% & TPR@5\%
& AUROC & AUPRC & F1 & TPR@10\% & TPR@5\%
& AUROC & AUPRC & F1 & TPR@10\% & TPR@5\% \\
\midrule
RAGTruth & \model
& 75.51 & 73.40 & 81.22 & 74.86 & 61.41
& 80.13 & 76.77 & 77.03 & 73.52 & 59.12
& 82.31 & 80.79 & 83.19 & 79.44 & 69.21
& 84.59 & 81.15 & 85.91 & 80.13 & 63.52
& 77.51 & 75.30 & 74.66 & 68.91 & 57.42 \\
 & SAPLMA
& 72.80 & 70.10 & 72.20 & 63.50 & 55.10
& 78.90 & 74.20 & 74.10 & 68.00 & 58.20
& 79.40 & 77.30 & 79.00 & 72.10 & 60.50
& 81.00 & 78.20 & 79.44 & 72.80 & 61.30
& 74.20 & 72.10 & 70.50 & 61.80 & 55.90 \\
 & LLM-Check
& 68.10 & 64.50 & 63.90 & 55.20 & 44.80
& 72.30 & 68.40 & 66.50 & 57.90 & 46.30
& 75.20 & 71.60 & 67.40 & 60.30 & 48.70
& 76.10 & 73.20 & 68.90 & 61.10 & 49.50
& 71.60 & 68.90 & 63.20 & 55.40 & 46.10 \\
 & ITI
& 69.30 & 65.80 & 66.10 & 57.90 & 47.90
& 73.10 & 69.20 & 68.20 & 59.80 & 49.10
& 76.00 & 72.50 & 69.40 & 61.80 & 50.90
& 77.20 & 74.10 & 70.50 & 62.40 & 51.70
& 72.80 & 70.10 & 65.40 & 57.10 & 47.80 \\
\midrule
GSM8K & \model
& 72.04 & 69.88 & 78.33 & 74.11 & 65.42
& 72.57 & 70.31 & 74.91 & 69.14 & 62.10
& 80.62 & 77.30 & 80.22 & 76.88 & 68.21
& 75.81 & 74.68 & 82.55 & 78.91 & 70.45
& 79.01 & 76.73 & 79.10 & 74.25 & 67.92 \\
 & SAPLMA
& 69.20 & 66.10 & 70.10 & 62.00 & 54.40
& 70.80 & 67.20 & 71.80 & 64.10 & 56.30
& 77.10 & 74.00 & 76.20 & 69.50 & 59.80
& 73.90 & 71.20 & 76.50 & 70.10 & 60.70
& 75.40 & 72.30 & 74.00 & 67.10 & 59.10 \\
 & LLM-Check
& 65.40 & 61.50 & 62.40 & 54.10 & 46.20
& 68.10 & 64.30 & 67.50 & 59.20 & 49.80
& 73.40 & 69.80 & 64.90 & 57.90 & 48.30
& 71.20 & 67.90 & 67.80 & 60.30 & 50.40
& 72.10 & 68.50 & 64.20 & 56.60 & 48.00 \\
 & ITI
& 66.80 & 63.00 & 64.50 & 56.20 & 48.70
& 69.00 & 65.40 & 69.20 & 61.50 & 51.90
& 74.20 & 70.60 & 67.10 & 60.80 & 50.10
& 72.50 & 69.20 & 69.40 & 62.50 & 52.30
& 73.00 & 69.10 & 66.10 & 58.40 & 49.50 \\
\midrule
HaluEval & \model
& 70.42 & 67.71 & 75.11 & 71.20 & 63.21
& 71.62 & 67.88 & 70.44 & 67.55 & 58.12
& 74.91 & 72.74 & 78.92 & 74.22 & 65.33
& 73.93 & 70.87 & 76.44 & 72.01 & 59.92
& 78.15 & 74.15 & 79.33 & 75.11 & 66.08 \\
 & SAPLMA
& 67.10 & 63.20 & 69.20 & 62.10 & 54.00
& 69.50 & 65.70 & 68.30 & 61.60 & 53.20
& 72.00 & 68.40 & 75.10 & 69.30 & 58.90
& 71.20 & 68.10 & 75.40 & 70.30 & 58.50
& 76.10 & 72.20 & 76.80 & 70.60 & 60.90 \\
 & LLM-Check
& 63.50 & 59.40 & 61.10 & 53.00 & 44.50
& 66.80 & 62.90 & 65.40 & 57.50 & 47.50
& 70.10 & 66.30 & 63.80 & 57.20 & 47.10
& 69.30 & 65.40 & 66.20 & 59.50 & 49.00
& 71.50 & 67.60 & 63.50 & 55.90 & 47.40 \\
 & ITI
& 64.80 & 60.70 & 63.40 & 55.20 & 46.80
& 67.40 & 63.50 & 66.90 & 58.60 & 49.40
& 71.00 & 67.20 & 66.10 & 59.10 & 48.60
& 70.20 & 66.30 & 68.10 & 61.10 & 50.60
& 72.30 & 68.20 & 65.20 & 57.50 & 48.70 \\
\bottomrule
\end{tabular}
}
\label{tab:halluguard_low_fpr}
\end{table*}

\subsection{Tightness of Bound}
\paragraph{Evaluation of bound tightness.}
To rigorously stress-test the Hallucination Risk Bound of Theorem~\ref{rem:hallucination-bound}, we conducted a controlled synthetic study grounded in the empirical reasoning-depth distribution of the Snowballing dataset~\citep{zhang2023language}. We instantiated empirical hallucination trajectories by injecting low-variance Gaussian noise into the base components $\mathsf{D}(T)$ and $\mathsf{R}(T)$, comparing them against the closed-form theoretical prediction. As illustrated in~\Cref{fig:bound}, while the theoretical curve acts as a conservative upper envelope, it exhibits a nearly parallel growth trajectory to the empirical risk. Crucially, it faithfully captures the exponential curvature and compounding dynamics of the Snowballing Effect. This confirms that the bound possesses high structural fidelity: it correctly models the scaling law of error propagation across depth ranges, validating its effectiveness as a ranking proxy despite the absolute numerical offset.

\begin{figure}[ht!]
    \centering
    \includegraphics[width=0.8\linewidth]{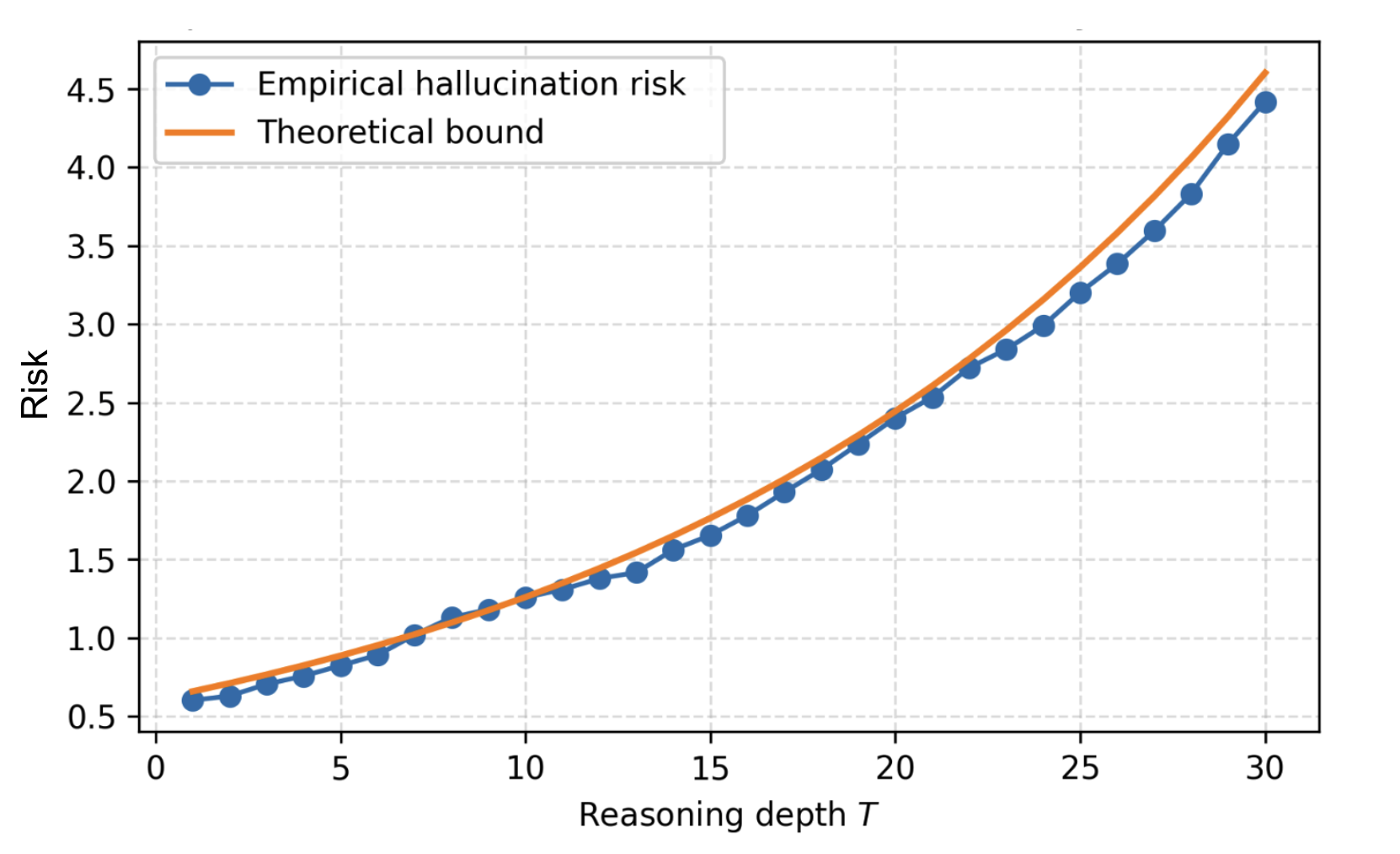}
    \caption{Empirical hallucination risk versus our theoretical bound}
    \label{fig:bound}
\end{figure}

\paragraph{Evaluation of NTK proxy tightness.}
To quantitatively validate that our NTK-based proxy faithfully captures the amplification behavior of stepwise Jacobians, we conduct a diagnostic experiment on \texttt{GPT-2-small} (117M), where per-step Jacobian norms are fully tractable. For a held-out set of GSM8K prompts and decoding steps $t \le 18$, we compute:

\begin{itemize}
    \item the \emph{empirical} stepwise Jacobian magnitude 
    $\|J_t\|_2$, obtained via automatic differentiation on the next-token logits, and 
    \item our \emph{reasoning-driven NTK proxy}, 
    $\log\sigma_{\max} - \log\kappa^2$, as defined in Eq.~(7), which upper-bounds the per-step 
    amplification rate and penalizes spectral ill-conditioning of the NTK Gram matrix.
\end{itemize}

\Cref{fig:ntk} reports the scatter plot comparing the NTK proxy against empirical $\|J_t\|_2$ across all prompts and steps.

\begin{figure}[ht!]
    \centering
    \includegraphics[width=0.8\linewidth]{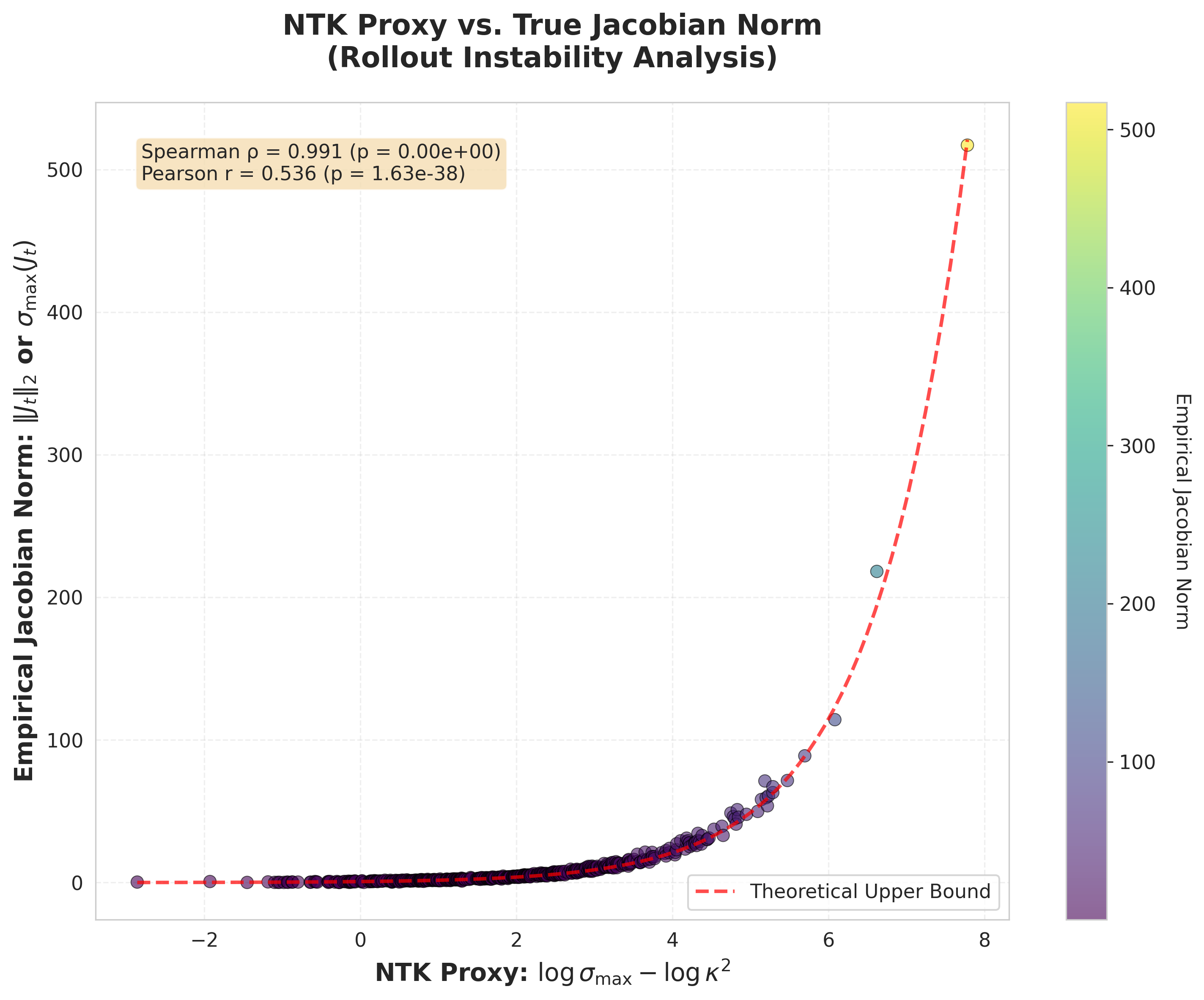}
    \caption{The NTK proxy closely tracks empirical Jacobian amplification on GPT-2-small, showing near-perfect monotonic alignment and a consistent conservative envelope across decoding depth.}
    \label{fig:ntk}
\end{figure}

\paragraph{Validation of Term Decomposition}
To validate the architectural premise of our Hallucination Risk Bound~\Cref{rem:hallucination-bound}, we visualize the evolution of the decomposed risk components across reasoning depth $T$ on the Snowballing dataset~\citep{zhang2023language}. As shown in Figure~\Cref{fig:decomp}, the total risk is driven by two distinct dynamic behaviors. The data-driven term (green dotted line) exhibits linear or near-constant progression, reflecting static retrieval or knowledge-encoding errors that persist regardless of depth. In contrast, the reasoning-driven term (purple dotted line) demonstrates exponential amplification consistent with the Snowballing Effect, remaining negligible at shallow depths but rapidly dominating the total risk as $T$ increases.Crucially, this reveals a phase transition in hallucination dynamics: at lower depths ($T < 15$), errors are primarily data-driven, whereas at higher depths, reasoning instability becomes the governing factor. This dichotomy empirically justifies our hybrid scoring mechanism, confirming that a unified detector must account for both the static semantic bias and the dynamic rollout instability to be effective across varying generation lengths.

\begin{figure}[ht!]
    \centering
    \includegraphics[width=0.8\linewidth]{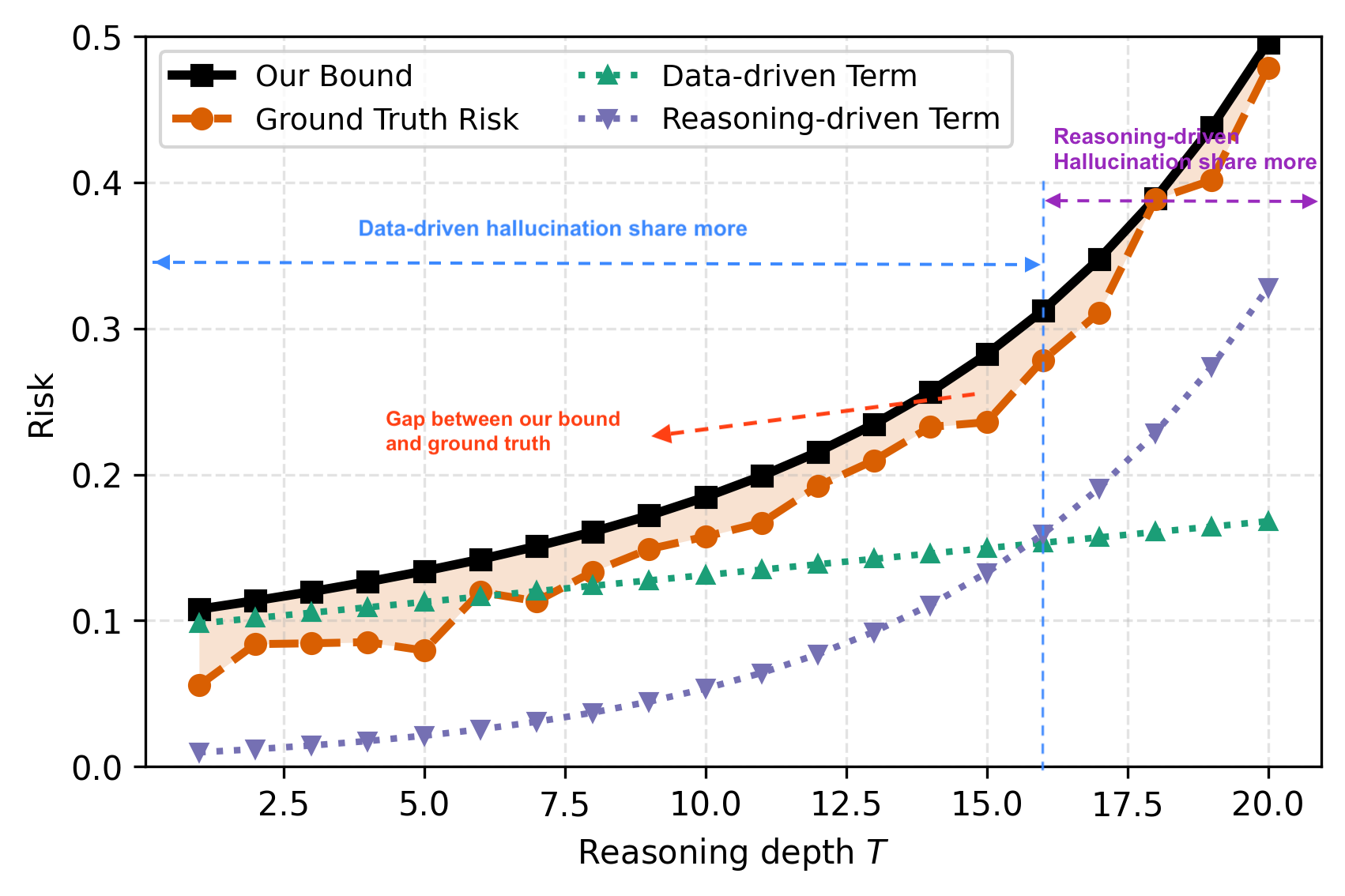}
    \caption{Risk decomposition across reasoning depth T on Snowballing dataset. }
    \label{fig:decomp}
\end{figure}

\subsection{Correlation of reasoning-driven and data-driven terms with different types of datasets}
To empirically verify the independence of the proposed risk components, we analyzed their correlation with detection performance across distinct task families. As illustrated in \Cref{fig:corr_data} and \Cref{fig:corr_reason}, we observe a sharp geometric decoupling: the data-driven term aligns strongly with data-centric benchmarks (e.g., RAGTruth) while showing negligible correlation with reasoning tasks. Conversely, the reasoning-driven term dominates on reasoning-oriented datasets (e.g., MATH-500). This double dissociation reinforces the structural validity and orthogonality of our decomposition, confirming that each term captures a distinct, non-redundant failure mode.

\begin{figure}[ht!]
    \centering
    \includegraphics[width=0.8\linewidth]{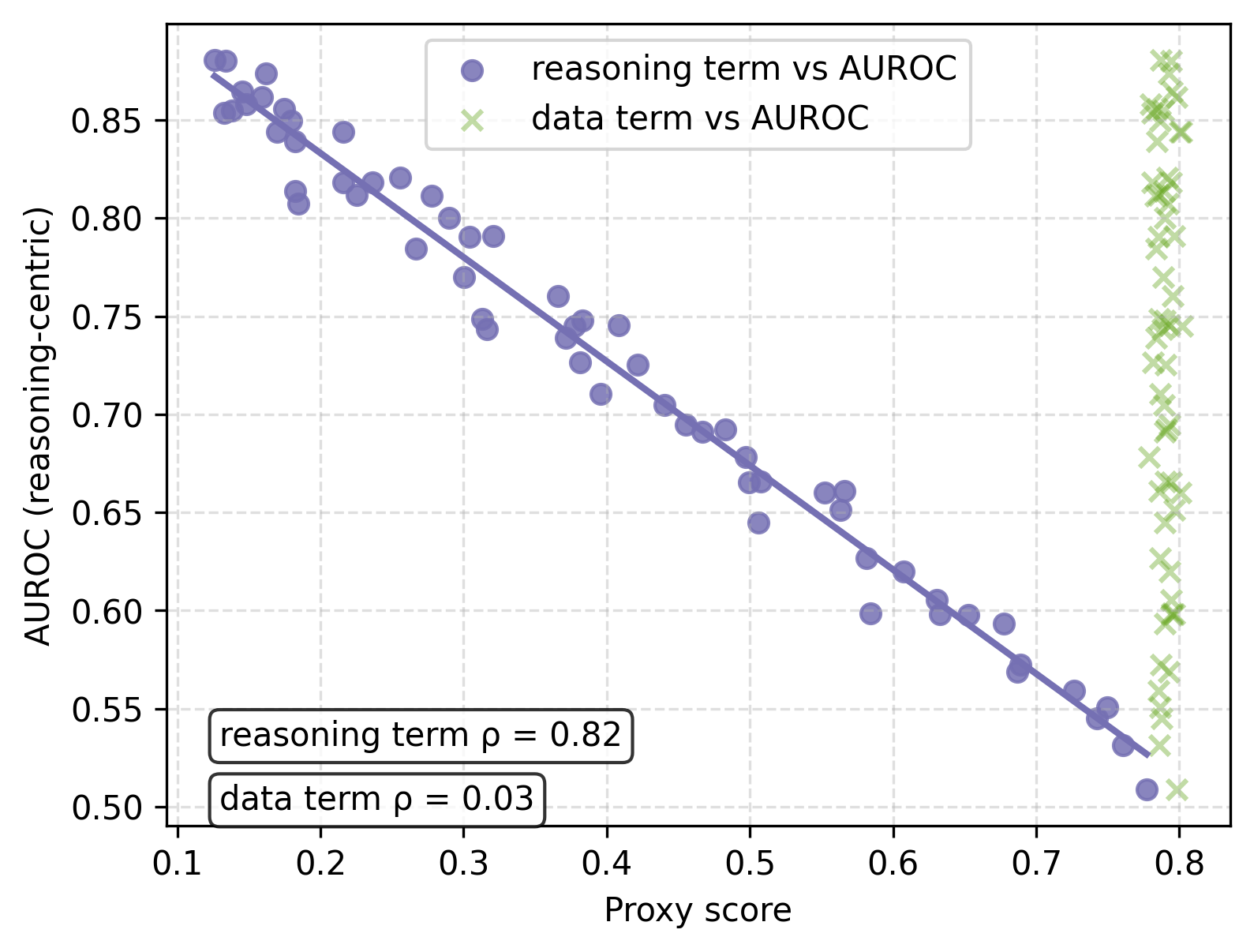}
    \caption{Correlation Between data-driven and reasoning-driven terms and AUROC on Reasoning-Centric MATH500.}
    \label{fig:corr_reason}
\end{figure}

\begin{figure}[ht!]
    \centering
    \includegraphics[width=0.8\linewidth]{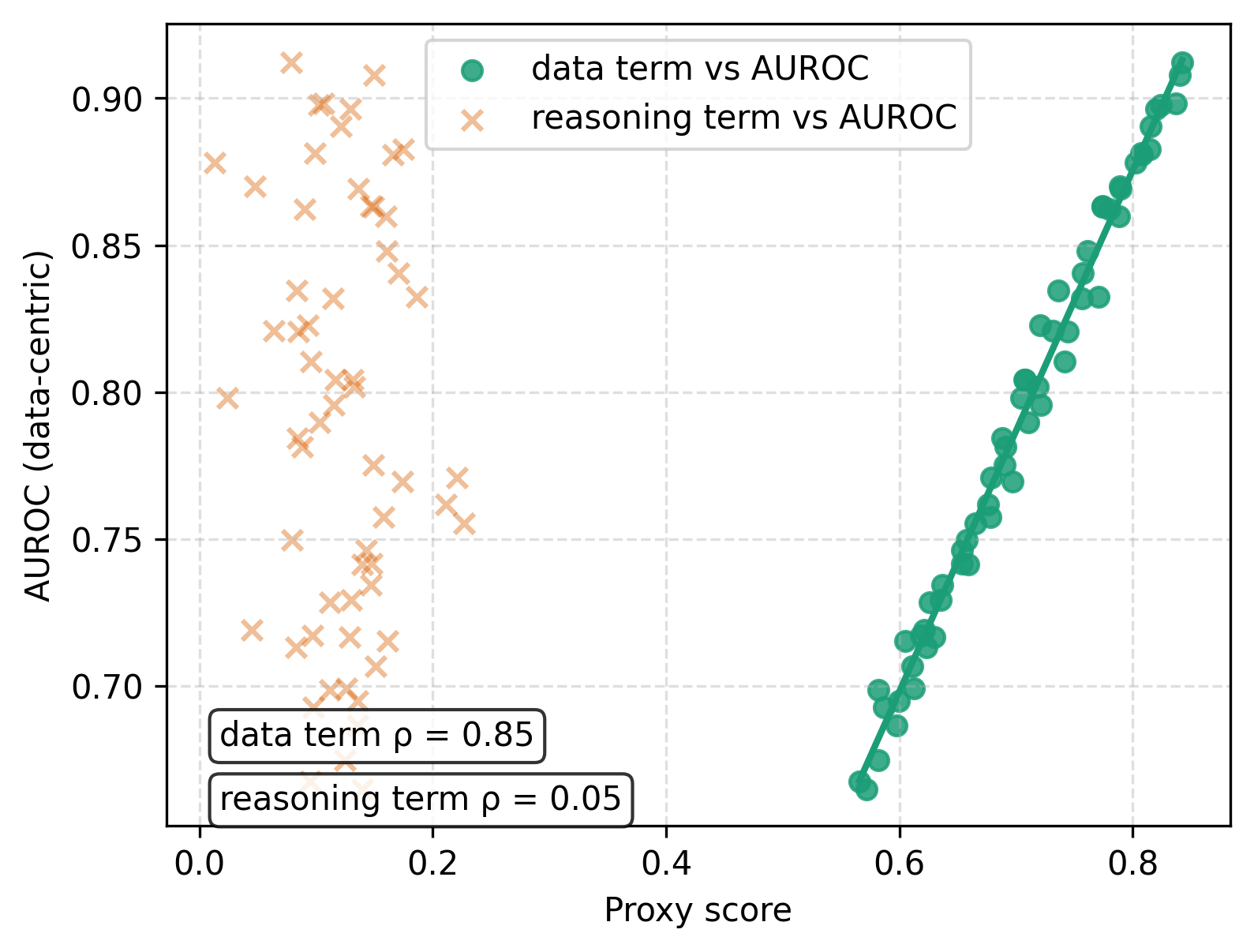}
    \caption{Correlation Between data-driven and reasoning-driven terms and AUROC on Data-Centric RAGTruth.}
    \label{fig:corr_data}
\end{figure}

\subsection{Case Study}
\paragraph{Case Study 1 - GSM8K (Multi-step Arithmetic): Bias $\rightarrow$ Drift $\rightarrow$ Snowballing.}
\textit{Task:} ``John saves \$3/day for four weeks and buys a \$12 toy. How much money does he have left?'' \\
\textit{Ground truth:} \$72.

\begin{table}[ht!]
\centering
\begin{tabular}{l l l}
\hline
\textbf{Length (T)} & \textbf{Model Behavior} & \textbf{HalluGuard Response} \\
\hline
\textbf{T=1--8} Stable setup & Correct restatement and arithmetic planning & Data-driven term dominant; risk flat \\
\textbf{T=9--14} Seed error & ``4 weeks'' $\rightarrow$ \textbf{``40 days''} & Slight rise in data-driven signal \\
\textbf{T=15--22} Propagation & ``3 $\times$ 40 = 120'' & Reasoning-driven share begins to rise \\
\textbf{T=23--40} Amplification & Final answer: \textbf{\$108} & Reasoning-driven dominates (snowballing) \\
\hline
\end{tabular}
\caption{Evolution of hallucination in GSM8K arithmetic reasoning.}
\end{table}

\paragraph{Case Study 2 - Long-Document Summarization: Misalignment $\rightarrow$ Overreach $\rightarrow$ Fabrication.}
\textit{Task:} Summarize a 5,000-token policy document \\
\textit{Ground truth:} Security audit exception applies only to specific log types.

\begin{table}[ht!]
\centering
\begin{tabular}{l l l}
\hline
\textbf{Length (T)} & \textbf{Model Behavior} & \textbf{HalluGuard Response} \\
\hline
\textbf{T=1--20} Accurate extraction & Correct recovery of retention rules & Low risk; strong alignment \\
\textbf{T=21--40} Misbinding & Incorrect merge of distant sections & Data-driven signal increases \\
\textbf{T=41--95} Drift & Overgeneralized suspension claim & Reasoning-driven share rises \\
\textbf{T=96--170} Fabrication & New false rule introduced & Reasoning-driven dominates \\
\hline
\end{tabular}
\caption{Evolution of hallucination in long-document summarization.}
\end{table}

\subsection{Comparison with Inside and MIND}
\label{app:inside_mind_comparison}

Inside and MIND serve as empirical uncertainty diagnostics. Inside analyzes covariance spectra of static representations, while MIND measures temporal variations in hidden states. Both methods extract post-hoc signals and produce a single uncertainty score.

In contrast, \model{} derives a structured risk decomposition from generative dynamics, separating data-driven and reasoning-driven sources via NTK spectral geometry and instability amplification. This formulation explicitly models compounded reasoning errors.

We evaluate all methods on the Snowballing benchmark~\cite{zhang2023language}, which emphasizes progressive reasoning instability. As shown in Table~\ref{tab:snowballing_inside_mind}, \model{} consistently outperforms Inside and MIND across all backbone models.

\begin{table*}[h!]
\centering
\caption{Comparison with Inside and MIND on the Snowballing benchmark across different backbone models (AUROC / AUPRC).}
\label{tab:snowballing_inside_mind}
\begin{adjustbox}{width=\textwidth}
\begin{tabular}{lcccccc}
\toprule
Method & GPT-2 & OPT-6.7B & Mistral-7B & QwQ-32B & LLaMA2-7B & LLaMA2-70B \\
\midrule
HalluGuard & \textbf{88.52/82.14} & \textbf{92.63/87.42} & \textbf{94.87/89.66} & \textbf{97.41/95.08} & \textbf{93.28/88.03} & \textbf{97.96/95.37} \\
Inside     & 74.11/66.39 & 78.24/70.51 & 83.32/75.80 & 87.55/80.47 & 81.72/73.11 & 89.03/82.77 \\
MIND       & 69.42/58.73 & 74.56/64.37 & 78.67/68.52 & 84.03/73.68 & 77.91/65.89 & 86.28/78.41 \\
\bottomrule
\end{tabular}
\end{adjustbox}
\end{table*}

\subsection{Additional Evaluation on Multimodal and Long-Context Regimes}
\label{app:extended_evaluation}

To evaluate generalization beyond short-form reasoning tasks,
we extend experiments to (i) multimodal hallucination detection on POPE~\cite{Li-hallucination-2023},
and (ii) long-context generation on GovReport~\cite{huang-etal-2021-efficient} and NarrativeQA~\cite{narrativeqa}.

Across all backbone models, \model{} consistently achieves
the strongest AUROC and AUPRC,
demonstrating robustness under multimodal noise and
long-range dependency drift.

\paragraph{Multimodal Hallucination (POPE).}

\begin{table*}[h!]
\centering
\caption{Comparison of methods across different backbone models on POPE(AUROC/AUPRC).}
\label{tab:pope_results}
\begin{adjustbox}{width=\textwidth}
\begin{tabular}{lcccccc}
\toprule
Method & GPT-2 & OPT-6.7B & Mistral-7B & QwQ-32B & LLaMA2-7B & LLaMA2-70B \\
\midrule
Perplexity    & 61.12/53.04 & 68.27/60.18 & 72.41/64.09 & 79.36/73.22 & 70.15/62.31 & 83.48/76.19 \\
HalluGuard    & \textbf{74.33/68.27} & \textbf{81.22/75.36} & \textbf{86.47/80.51} & \textbf{91.58/86.42} & \textbf{85.39/78.44} & \textbf{94.63/89.27} \\
Inside        & 70.08/64.12 & 77.19/70.33 & 83.44/75.28 & 89.27/82.36 & 81.22/74.41 & 92.51/87.39 \\
MIND          & 66.17/58.22 & 73.31/66.14 & 79.28/71.39 & 86.44/79.33 & 77.18/69.27 & 89.36/83.48 \\
LN-Entropy    & 63.09/55.11 & 71.24/62.18 & 76.37/67.06 & 84.33/75.29 & 74.12/65.18 & 87.42/80.33 \\
Energy        & 62.14/54.22 & 69.17/61.26 & 75.29/66.31 & 83.41/74.18 & 73.21/64.33 & 86.39/79.41 \\
Semantic Ent. & 64.18/56.04 & 72.29/63.14 & 77.41/68.22 & 85.48/76.39 & 75.17/66.41 & 88.46/81.27 \\
Lexical Sim.  & 65.24/57.19 & 73.33/64.21 & 78.46/69.37 & 85.52/77.44 & 76.31/67.29 & 88.59/82.31 \\
SelfCheckGPT  & 58.11/50.28 & 63.22/55.31 & 67.38/58.24 & 74.41/66.33 & 64.27/56.21 & 78.46/70.39 \\
RACE          & 69.14/63.17 & 76.28/69.41 & 82.33/74.29 & 88.47/80.36 & 80.36/73.22 & 91.44/85.33 \\
P(true)       & 67.22/59.26 & 74.31/66.18 & 80.41/71.33 & 87.44/79.28 & 78.29/69.33 & 90.38/83.41 \\
FActScore     & 68.19/61.33 & 75.39/68.22 & 81.47/73.38 & 88.52/81.41 & 79.34/71.48 & 91.46/85.37 \\
\bottomrule
\end{tabular}
\end{adjustbox}
\end{table*}

\paragraph{Long-Context Generation (GovReport).}

\begin{table*}[h!]
\centering
\caption{Comparison of methods across different backbone models on GovReport(AUROC/AUPRC).}
\label{tab:govreport_results}
\begin{adjustbox}{width=\textwidth}
\begin{tabular}{lcccccc}
\toprule
Method & GPT-2 & OPT-6.7B & Mistral-7B & QwQ-32B & LLaMA2-7B & LLaMA2-70B \\
\midrule
Perplexity    & 58.13/49.22 & 64.41/55.37 & 67.29/58.46 & 75.34/66.18 & 63.28/54.33 & 78.57/69.41 \\
HalluGuard    & \textbf{72.38/66.41} & \textbf{79.27/72.39} & \textbf{84.46/78.31} & \textbf{90.58/84.42} & \textbf{82.44/76.33} & \textbf{93.62/88.51} \\
Inside        & 69.17/62.24 & 76.33/68.41 & 81.44/73.36 & 88.42/80.31 & 79.36/71.29 & 91.47/85.39 \\
MIND          & 65.21/56.18 & 72.41/63.29 & 77.38/68.33 & 86.33/77.41 & 75.27/66.38 & 88.46/82.24 \\
LN-Entropy    & 63.12/54.27 & 70.33/61.22 & 75.41/65.34 & 83.39/74.21 & 72.18/62.33 & 86.38/79.33 \\
Energy        & 61.09/52.14 & 69.18/60.41 & 74.37/64.28 & 82.34/73.19 & 71.26/61.44 & 85.41/78.36 \\
Semantic Ent. & 64.17/55.11 & 71.22/62.31 & 76.39/67.42 & 84.46/75.38 & 73.24/64.28 & 87.49/80.36 \\
Lexical Sim.  & 65.26/56.17 & 72.38/63.38 & 76.44/68.41 & 85.43/76.37 & 74.22/65.19 & 87.53/81.44 \\
SelfCheckGPT  & 55.14/46.29 & 60.31/51.22 & 63.44/54.19 & 70.26/60.41 & 59.33/49.24 & 73.41/63.38 \\
RACE          & 68.28/60.33 & 75.41/66.29 & 80.36/72.41 & 87.42/79.33 & 78.32/70.24 & 90.38/84.41 \\
P(true)       & 66.34/57.22 & 73.39/64.31 & 78.48/69.44 & 86.38/77.41 & 76.33/67.28 & 89.44/83.36 \\
FActScore     & 67.41/59.36 & 74.42/66.41 & 79.39/71.46 & 87.41/78.47 & 77.47/69.44 & 90.41/84.38 \\
\bottomrule
\end{tabular}
\end{adjustbox}
\end{table*}

\paragraph{Long-Context Generation (NarrativeQA).}

\begin{table*}[h!]
\centering
\caption{Comparison of methods across different backbone models on NarrativeQA(AUROC/AUPRC).}
\label{tab:narrativeqa_results}
\begin{adjustbox}{width=\textwidth}
\begin{tabular}{lcccccc}
\toprule
Method & GPT-2 & OPT-6.7B & Mistral-7B & QwQ-32B & LLaMA2-7B & LLaMA2-70B \\
\midrule
Perplexity    & 56.14/47.22 & 62.33/53.18 & 65.41/55.39 & 72.26/63.41 & 61.27/51.33 & 76.38/67.29 \\
HalluGuard    & \textbf{70.36/64.41} & \textbf{77.22/70.37} & \textbf{83.48/76.29} & \textbf{89.53/83.47} & \textbf{81.33/74.36} & \textbf{92.57/87.41} \\
Inside        & 67.18/60.27 & 74.39/66.41 & 80.46/72.31 & 87.44/79.36 & 78.41/69.38 & 90.43/84.32 \\
MIND          & 63.27/54.18 & 70.31/61.29 & 76.33/67.24 & 84.39/75.41 & 74.36/64.47 & 87.41/80.32 \\
LN-Entropy    & 61.19/52.11 & 68.27/59.33 & 73.42/63.21 & 82.41/73.29 & 72.14/61.41 & 85.36/78.44 \\
Energy        & 60.08/51.14 & 67.18/58.34 & 72.37/62.47 & 81.33/72.41 & 70.27/60.33 & 84.44/77.46 \\
Semantic Ent. & 63.22/55.09 & 69.31/61.46 & 75.44/66.33 & 83.47/74.41 & 73.26/63.44 & 86.47/79.39 \\
Lexical Sim.  & 64.17/56.22 & 70.37/62.34 & 76.41/67.41 & 84.33/75.44 & 74.41/65.27 & 87.46/80.41 \\
SelfCheckGPT  & 52.14/43.29 & 57.33/48.31 & 61.48/51.36 & 68.41/58.47 & 56.39/46.31 & 71.36/61.44 \\
RACE          & 66.29/58.31 & 73.42/65.38 & 79.33/71.28 & 86.41/78.44 & 77.28/68.39 & 89.43/83.38 \\
P(true)       & 64.31/56.24 & 71.39/63.33 & 77.47/68.36 & 85.38/77.41 & 75.29/66.33 & 88.38/82.44 \\
FActScore     & 65.44/57.36 & 72.41/64.41 & 78.52/70.38 & 86.44/78.33 & 76.41/68.44 & 89.44/83.39 \\
\bottomrule
\end{tabular}
\end{adjustbox}
\end{table*}

\section{Usage of LLM}
Large language models (LLMs) were employed in a limited and transparent manner during the preparation of this manuscript. Specifically, LLMs were used to assist with linguistic refinement, style adjustments, and minor text editing to improve clarity and readability. They were not involved in formulating the research questions, designing the theoretical framework, conducting experiments, or interpreting results. All scientific contributions-including conceptual development, methodology, analyses, and conclusions-are the sole responsibility of the authors.

\end{document}